\documentclass{article}

\usepackage{multirow}
\newcommand*\rot{\rotatebox{90}}

\usepackage[final]{neurips_2023}

\usepackage[utf8]{inputenc} %
\usepackage[T1]{fontenc}    %
\usepackage{url}            %
\usepackage{booktabs}       %
\usepackage{amsfonts}       %
\usepackage{nicefrac}       %
\usepackage{microtype}      %
\usepackage{xcolor}         %
\usepackage{amsmath}
\usepackage{amssymb}
\usepackage{mathtools}
\usepackage{slashed}
\usepackage{braket}
\usepackage{enumitem}
\usepackage[colorlinks,citecolor=darkgreen,linkcolor=firebrick,urlcolor=firebrick]{hyperref}

\usepackage{graphicx}
\usepackage{tikz}
\usepackage{tikz-cd}
\usepackage{times}
\usepackage{courier}
\usepackage{url}    
\usepackage{bm}
\usepackage{subfig}
\usepackage{physics}
\usepackage{xcolor}
\usepackage{natbib}
\usepackage{mdframed}
\usepackage{nicefrac}
\usepackage{booktabs}
\usepackage{lipsum}
\usepackage{titlesec}
\usepackage{wrapfig,lipsum,booktabs}
\usepackage{authblk}
\usepackage{blindtext}
\usepackage{algorithm}
\usepackage{algorithmicx}
\usepackage{algpseudocode}
\usepackage{graphicx}
\usepackage{siunitx}
\usepackage{multirow}
\usepackage{rotating}

\usepackage[capitalize,noabbrev]{cleveref}

\usepackage{CJK}

\newenvironment{claim}{  \begin{mdframed}[linecolor=black!0,backgroundcolor=black!10]\noindent%
		\ignorespaces
  }{\end{mdframed}}

\renewenvironment{figure}[1][]{
  \begin{originalfigure}[#1]
    \begin{mdframed}[linecolor=black!0,backgroundcolor=black!1]
}{
    \end{mdframed}
  \end{originalfigure}
}

\usepackage{array}
\usepackage{makecell}

\def\half{{\frac{1}{2}}}

\newcommand{\bea}{\begin{eqnarray}}
\newcommand{\eea}{\end{eqnarray}}

\usepackage{slashed}

\def\({\left(}
\def\){\right)}
\def\[{\left[}
\def\]{\right]}

\usepackage{dashbox}
\usepackage{xcolor}
\usepackage{colortbl}
\usepackage{arydshln}
\definecolor{lightyellow}{rgb}{1.0, 0.95, 0.7}
\definecolor{blue}{rgb}{0.0, 0.4, 1.0}
\definecolor{Blue}{rgb}{0,0,0}
\definecolor{darkgreen}{rgb}{0,0.40,0}
\definecolor{firebrick}{rgb}{0.698,0.133,0.133}

\newcommand*{\Blue}[1]{\textcolor{black}{#1}}
\newcommand*{\blue}[1]{\textcolor{black}{#1}}

\definecolor{colorA}{rgb}{1,0,0}
\definecolor{colorB}{rgb}{0,0.3,1}
\definecolor{colorC}{rgb}{0.9,0.8,0.2}
\definecolor{colorD}{rgb}{0,0.65,0}
\definecolor{lesslightgray}{rgb}{0.5,0.5,0.5}
\definecolor{light-gray}{gray}{0.95}

\let\tilde\widetilde

\newcommand{\calH}{\mathcal{H}}

\newcommand{\calO}{\mathcal{O}}

\newcommand{\calT}{\mathcal{T}}

\newcommand{\calX}{\mathcal{X}}

\newcommand{\bK}{\mathbf{K}}

\newcommand{\bQ}{\mathbf{Q}}
\newcommand{\bR}{\mathbf{R}}

\newcommand{\bV}{\mathbf{V}}
\newcommand{\bW}{\mathbf{W}}
\newcommand{\bX}{\mathbf{X}}
\newcommand{\bY}{\mathbf{Y}}
\newcommand{\bZ}{\bm{Z}}
\newcommand{\ba}{\mathbf{a}}
\newcommand{\bb}{\mathbf{b}}

\newcommand{\bp}{\mathbf{p}}

\newcommand{\bx}{\mathbf{x}}
\newcommand{\by}{\mathbf{y}}
\newcommand{\bz}{\mathbf{z}}
\newcommand{\bxi}{{\bm{\xi}}}

\newcommand{\Max}{\mathop{\rm Max}}
\newcommand{\Min}{\mathop{\rm Min}}
\newcommand{\argmin}{\mathop{\mathrm{ArgMin}}}  
\newcommand{\argmax}{\mathop{\mathrm{ArgMax}}}

\newcommand{\Softmax}{\mathop{\rm{Softmax}}}
\newcommand{\Sparsemax}{\mathop{\rm{Sparsemax}}}

\newcommand{\lse}{\mathop{\rm{lse}}}

\newcommand{\sT}{ \mathsf{T} }

\newcommand{\sumM}{\sum_{\mu=1}^M}

\def\R{\mathbb{R}}

\let\cite\citep 

\usepackage{amsthm}
\makeatletter
\def\th@remark{%
  \thm@headfont{\bfseries}%
  \normalfont %
  \thm@preskip\topsep \divide\thm@preskip\tw@
  \thm@postskip\thm@preskip
}
\makeatother
\usepackage[many]{tcolorbox}
\theoremstyle{definition}
\newtheorem{theorem}{Theorem}[section]
\tcolorboxenvironment{theorem}{
  colback=black!10,
  colframe=white,%
  width=\dimexpr\linewidth+10pt\relax,%
  enlarge left by=-5pt,%
  enlarge right by=-5pt,%
  boxsep=5pt,%
  boxrule=0pt,
  left=0pt,right=0pt,top=0pt,bottom=0pt,
  sharp corners,
  before skip=\topsep,
  after skip=\topsep
}
\newtheorem{lemma}{Lemma}[section]
\tcolorboxenvironment{lemma}{
  colback=black!10,
  colframe=white,%
  width=\dimexpr\linewidth+10pt\relax,%
  enlarge left by=-5pt,%
  enlarge right by=-5pt,%
  boxsep=5pt,%
  boxrule=0pt,
  left=0pt,right=0pt,top=0pt,bottom=0pt,
  sharp corners,
  before skip=\topsep,
  after skip=\topsep
}
\newtheorem{corollary}{Corollary}[theorem]
\tcolorboxenvironment{corollary}{
  colback=black!10,
  colframe=white,%
  width=\dimexpr\linewidth+10pt\relax,%
  enlarge left by=-5pt,%
  enlarge right by=-5pt,%
  boxsep=5pt,%
  boxrule=0pt,
  left=0pt,right=0pt,top=0pt,bottom=0pt,
  sharp corners,
  before skip=\topsep,
  after skip=\topsep
}

\theoremstyle{definition}
\newtheorem{definition}{Definition}[section]
\tcolorboxenvironment{definition}{
  colback=black!10,
  colframe=white,%
  width=\dimexpr\linewidth+10pt\relax,%
  enlarge left by=-5pt,%
  enlarge right by=-5pt,%
  boxsep=5pt,%
  boxrule=0pt,
  left=0pt,right=0pt,top=0pt,bottom=0pt,
  sharp corners,
  before skip=\topsep,
  after skip=\topsep
}
\theoremstyle{remark}
\newtheorem{remark}{Remark}[section]

\tcolorboxenvironment{assumption}{
  colback=black!10,
  colframe=white,%
  width=\dimexpr\linewidth+10pt\relax,%
  enlarge left by=-5pt,%
  enlarge right by=-5pt,%
  boxsep=5pt,%
  boxrule=0pt,
  left=0pt,right=0pt,top=0pt,bottom=0pt,
  sharp corners,
  before skip=\topsep,
  after skip=\topsep
}

\tcolorboxenvironment{problem}{
  colback=black!10,
  colframe=white,%
  width=\dimexpr\linewidth+10pt\relax,%
  enlarge left by=-5pt,%
  enlarge right by=-5pt,%
  boxsep=5pt,%
  boxrule=0pt,
  left=0pt,right=0pt,top=0pt,bottom=0pt,
  sharp corners,
  before skip=\topsep,
  after skip=\topsep
}

\newtheorem*{fact}{Sparsemax in Closed-Form}
\tcolorboxenvironment{fact}{
  colback=black!10,
  colframe=white,%
  width=\dimexpr\linewidth+10pt\relax,
  enlarge left by=-5pt,%
  enlarge right by=-5pt,%
  boxsep=5pt,%
  boxrule=0pt,
  left=0pt,right=0pt,top=0pt,bottom=0pt,
  sharp corners,
  before skip=\topsep,
  after skip=\topsep
}

\crefname{theorem}{Theorem}{Theorems}
\crefname{proposition}{Proposition}{Propositions}
\crefname{lemma}{Lemma}{Lemmas}
\crefname{corollary}{Corollary}{Corollaries}
\crefname{definition}{Definition}{Definitions}
\crefname{assumption}{Assumption}{Assumptions}
\crefname{remark}{Remark}{Remarks}
\crefname{problem}{Problem}{Problems}
\crefname{property}{Property}{property}
\crefname{fact}{Fact}{fact}

\numberwithin{equation}{section}
\numberwithin{theorem}{section}
\numberwithin{proposition}{section}
\numberwithin{definition}{section}
\numberwithin{lemma}{section}
\numberwithin{assumption}{section}
\numberwithin{remark}{section}

\usepackage{lipsum}

\newcommand\blfootnote[1]{%
  \begingroup
  \renewcommand\thefootnote{}\footnote{#1}%
  \addtocounter{footnote}{-1}%
  \endgroup
}
\newcommand*{\annot}[1]{\tag*{\footnotesize{\textcolor{black!50}{\big(#1\big)}}}}

\makeatletter
\let\save@mathaccent\mathaccent
\newcommand*\if@single[3]{%
    \setbox0\hbox{${\mathaccent"0362{#1}}^H$}%
    \setbox2\hbox{${\mathaccent"0362{\kern0pt#1}}^H$}%
    \ifdim\ht0=\ht2 #3\else #2\fi
}
\newcommand*\rel@kern[1]{\kern#1\dimexpr\macc@kerna}
\newcommand*\widebar[1]{\@ifnextchar^{{\wide@bar{#1}{0}}}{\wide@bar{#1}{1}}}
\newcommand*\wide@bar[2]{\if@single{#1}{\wide@bar@{#1}{#2}{1}}{\wide@bar@{#1}{#2}{2}}}
\newcommand*\wide@bar@[3]{%
    \begingroup
    \def\mathaccent##1##2{%
        \let\mathaccent\save@mathaccent
        \if#32 \let\macc@nucleus\first@char \fi
        \setbox\z@\hbox{$\macc@style{\macc@nucleus}_{}$}%
        \setbox\tw@\hbox{$\macc@style{\macc@nucleus}{}_{}$}%
        \dimen@\wd\tw@
        \advance\dimen@-\wd\z@
        \divide\dimen@ 3
        \@tempdima\wd\tw@
        \advance\@tempdima-\scriptspace
        \divide\@tempdima 10
        \advance\dimen@-\@tempdima
        \ifdim\dimen@>\z@ \dimen@0pt\fi
        \rel@kern{0.6}\kern-\dimen@
        \if#31
        \overline{\rel@kern{-0.6}\kern\dimen@\macc@nucleus\rel@kern{0.4}\kern\dimen@}%
        \advance\dimen@0.4\dimexpr\macc@kerna
        \let\final@kern#2%
        \ifdim\dimen@<\z@ \let\final@kern1\fi
        \if\final@kern1 \kern-\dimen@\fi
        \else
        \overline{\rel@kern{-0.6}\kern\dimen@#1}%
        \fi
    }%
    \macc@depth\@ne
    \let\math@bgroup\@empty \let\math@egroup\macc@set@skewchar
    \mathsurround\z@ \frozen@everymath{\mathgroup\macc@group\relax}%
    \macc@set@skewchar\relax
    \let\mathaccentV\macc@nested@a
    \if#31
    \macc@nested@a\relax111{#1}%
    \else
    \def\gobble@till@marker##1\endmarker{}%
    \futurelet\first@char\gobble@till@marker#1\endmarker
    \ifcat\noexpand\first@char A\else
    \def\first@char{}%
    \fi
    \macc@nested@a\relax111{\first@char}%
    \fi
    \endgroup
    }
\makeatother

\makeatletter
\newcommand*{\redefinesymbolwitharg}[1]{%
  \expandafter\let\csname ltx#1\expandafter\endcsname\csname #1\endcsname
  \@namedef{#1}{\@ifnextchar{^}{\@nameuse{#1@}}{\@nameuse{#1@}^{}}}%
  \expandafter\def\csname #1@\endcsname^##1##2{%
     \csname ltx#1\endcsname\ifx!##1!\else^{##1}\fi\mathopen{}\mathclose\bgroup\left(##2\aftergroup\egroup\right)
     }%
}
\makeatother
\redefinesymbolwitharg{sin}
\redefinesymbolwitharg{cos}

\usepackage{etoolbox}

\titlespacing\section{0pt}{4pt plus 4pt minus 2pt}{-2pt plus 2pt minus 2pt}
\titlespacing\subsection{0pt}{2pt plus 4pt minus 2pt}{-2pt plus 2pt minus 2pt}
\titlespacing\subsubsection{0pt}{2pt plus 4pt minus 2pt}{-2pt plus 2pt minus 2pt}

\usepackage{titletoc}

\title{On Sparse Modern Hopfield Model}

\author{
    {\bf
   Jerry Yao-Chieh Hu$^\dagger$\quad 
   Donglin Yang$^\dagger$ \quad
   Dennis Wu$^\dagger$}\\ \vspace{-0.5em}
   {\bf 
   Chenwei Xu$^\dagger$ \quad
   Bo-Yu Chen$^\ddag$ \quad
   Han Liu$^{\dagger\natural}$}
   \\ \vspace{0.5em}
{\small
$^\dagger$Department of Computer Science, Northwestern University, Evanston, IL 60208 USA\\
$^\ddag$Department of Physics, National Taiwan University, Taipei 10617, Taiwan\\
$^\natural$Department of Statistics and Data Science, Northwestern University, Evanston, IL 60208 USA
}
   \\
   \vspace{0.5em}
   {\footnotesize
   \texttt{\{\href{mailto:jhu@u.northwestern.edu}{jhu}, \href{mailto:dlyang@u.northwestern.edu}{dlyang}, \href{mailto:hibb@u.northwestern.edu}{hibb}, \href{mailto:cxu@u.northwestern.edu}{cxu}\}@u.northwestern.edu}\\
   \texttt{\href{mailto:b12202023@ntu.edu.tw}{b12202023@ntu.edu.tw}, \href{mailto:hanliu@northwestern.edu}{hanliu@northwestern.edu}}}
}

\begin{document}

\maketitle
\vspace{-2em}
\begin{abstract}
    We introduce the sparse modern Hopfield model as a sparse extension of the modern Hopfield model.
Like its dense counterpart, the sparse modern Hopfield model equips a memory-retrieval dynamics whose one-step approximation corresponds to the sparse attention mechanism. 
Theoretically, our key contribution is a principled derivation of a closed-form sparse Hopfield energy using the convex conjugate of the sparse entropic regularizer.
Building upon this, we derive the sparse memory retrieval dynamics from the sparse energy function and show its one-step approximation is equivalent to the sparse-structured attention.
Importantly, we provide a sparsity-dependent memory retrieval error bound which is provably tighter than its dense analog.
The conditions for the benefits of sparsity to arise are therefore identified and discussed.
In addition, we show that the sparse modern Hopfield model maintains the robust theoretical properties of its dense counterpart, including rapid fixed point convergence and exponential memory capacity.
Empirically, we use both synthetic and real-world datasets to demonstrate that the sparse Hopfield model outperforms its dense counterpart in many situations. 
\blfootnote{Code is available at \href{https://github.com/MAGICS-LAB/SparseModernHopfield}{GitHub}; future updates are on \href{https://arxiv.org/abs/2309.12673}{arXiv}. [Version: Novermber, 28, 2023]}
\end{abstract}

\section{Introduction}
\label{sec:intro}
We address the computational challenges of modern Hopfield models by introducing a sparse Hopfield model. 
Our sparse continuous Hopfield model equips a memory-retrieval dynamics that aligns with the sparse-structured attention mechanism. 
By establishing a connection to sparse attention, the proposed model not only offers a theoretically-grounded energy-based model for associative memory but also enables robust representation learning and seamless integration with deep learning architectures.
This approach serves as an initial attempt of pushing the correspondence\footnote{While this equivalence only holds when the retrieval dynamics is applied exactly once, as originally shown in \cite{ramsauer2020hopfield} and later emphasized in \cite{krotov2020large}, it allows us to view modern Hopfield models as generalized attentions with additional functionalities and hence opens new avenues for Hopfield-based  architecture designs. See \cref{sec:related_works} for more discussions.} between Hopfield models and attention mechanism \cite{ramsauer2020hopfield} toward sparse region, both theoretically and empirically, resulting in data-dependent sparsity for meaningful and robust pattern representations, and a focus on the most relevant information for each specific instance.

\blue{
Hopfield models are classic associative memory models for both biological and artificial neural networks \cite{hopfield1982neural,hopfield1984neurons}.
These models are designed to store and retrieve memory patterns.
They achieve these by embedding the memories in the energy landscape of a physical system (e.g., the Ising model in \cite{hopfield1982neural,peretto1986long}; see \cref{fig:energylandscape} for a visualization), where each memory corresponds to a local minimum. 
When a query is presented, the model initiates energy-minimizing retrieval dynamics at the query, which then navigate the energy landscape to find the nearest local minimum, effectively retrieving the memory most similar to the query.
}

In the same vein,  \citet{ramsauer2020hopfield} propose the modern Hopfield model and integrate it into deep learning architectures via a strong connection with transformer attention, offering enhanced performance, theoretically guaranteed exponential memory capacity, and the ability to handle continuous patterns.
In addition, the modern Hopfield models have found success in various applications, such as immunology \cite{widrich2020modern} and large language model \cite{furst2022cloob}.
Apart from the elegant connection to attention, theoretical advantages and empirically successes, the modern Hopfield models have been shown to be computationally heavy and vulnerable against noisy queries \cite{millidge2022universal}.
In particular, the dense output alignments of the retrieval dynamics in modern Hopfield models \cite{ramsauer2020hopfield} can be computationally inefficient, making models less interpretable and noise-sensitive by assigning probability mass to many implausible outputs (patterns/keys). 

To combat above, incorporating sparsity is an essential and common strategy. 
While there is a vast body of work on sparsifying attention mechanisms \cite{tay2022efficient,beltagy2020longformer, qiu2019blockwise,child2019generating,peters2019sparse, martins2016softmax}, similar developments for the Hopfield models remain less explored. 
To bridge this gap, we present a sparse Hopfield model that corresponds to the sparsemax attention mechanism \cite{martins2016softmax}.
In this paper, we study the sparsification of the modern Hopfield model.
The challenges are three-fold:
\begin{enumerate}[leftmargin=2em,label=(C\arabic*),before=\vspace{-0.5em}, after=\vspace{-0.5em}]
\setlength\itemsep{0.1em}
    \item 
    \textbf{Non-Trivial Sparsification --- Sparse Hopfield $\leftrightarrow$ Sparse Attention:} 
    To enable the use of sparse Hopfield models as computational devices (DNN learning models) akin to \cite{ramsauer2020hopfield}, it is essential to achieve \textit{non-trivial} sparsifications that exhibit equivalence to specific sparse attention models.
    In other words, 
    any meaningful sparsification should extend
    the established equivalence \cite{ramsauer2020hopfield}  between modern Hopfield models and attention to encompass the sparse domain.
    While generalizing such equivalence is potentially impactful as it may lay the groundwork for future Hopfield-based methodologies, architecture designs and bio-computing systems (as in \cite{kozachkov2022building}), the \textit{heuristic} design of the modern Hopfield model poses great difficulty to developing desired sparse models.

    \item
    \textbf{Introducing Sparsity into Hopfield Models:} 
    Unlike attention mechanisms where sparsification is typically achieved either on the attention matrix (e.g., structured-sparsity \cite{tay2020sparse,child2019generating}) or on the element-wise normalization map (e.g., sparsity-inducing maps
    \cite{correia2019adaptively,peters2019sparse,martins2016softmax}), the sparsification of Hopfield models is applied to \textit{both} the energy function and the memory-retrieval dynamics, where the latter monotonically decreases the Hopfield energy over time. 
    Since attention mechanisms (transformers) are typically not equipped with such a dynamical description, introducing sparsity into Hopfield models while retaining the connection to attention is a less straightforward process.

    \item 
    \textbf{Properties of the Sparse Hopfield Model:} 
    Further, it is unclear how the introduced sparsity may affect different aspects of the model, such as memory capacity, fixed point convergence, retrieval accuracy, and so on.
    Ideally, we are looking for sparsities that offer provable computational benefits, such as enhanced robustness and increased memory capacity, among others.

\end{enumerate}

Challenges (C1) and (C2) are inherent in Hopfield model, and certain requirements on the design of energy function and retrieval dynamics are inevitable to obtain non-trivial sparse models.
Hence, we suppose the sparsified models should satisfy some conditions and verify them accordingly. 
Concretely, a formulation for deriving desired sparse Hopfield energy via convex conjugation of entropic regularizers is proposed.
Furthermore, by applying Danskin's theorem and convex-concave procedure \cite{yuille2003concave,yuille2001concave} on the sparse Hopfield energy function, we obtain sparse retrieval dynamics linked to sparse attention.
For (C3), the convergence of energy stationary points and retrieval dynamics fixed points are connected via Zangwill's method \cite{zangwill1969nonlinear}. 
The sparse retrieval error bound is derived and used to determined the well-separation condition for successful memory storage and retrieval. 
Lastly, the fundamental limit of memory capacity is derived using the expected separation of random points on spheres \cite{cai2012phase,brauchart2018random,ramsauer2020hopfield}.

In summary, this work handles sparsification of modern Hopfield models while linking them to sparse attention by addressing the following question:
\begin{claim}
Is it possible to develop a theoretically-grounded (non-trivial) sparse Hopfield model capable of storing information or learned prototypes throughout various layers of DNN models?
\end{claim}

\textbf{Contributions.}
We propose the Sparse Modern Hopfield Model.
Our contributions are as follows:
\begin{itemize}[leftmargin=2em,before=\vspace{-0.5em}, after=\vspace{-0.5em}]
\setlength\itemsep{0em}
    \item 
    We propose a novel sparse Hopfield model whose retrieval dynamics corresponds to sparsemax attention mechanism.
    It leads to sparse patterns by design, inheriting both noise robustness and potential computational efficiency\footnote{\blue{Note that, the  proposed model's sparsity falls under the category of \textit{sparsity-inducing normalization maps}.
    Consequently, the forward pass still requires $\calO(n^2)$ space complexity.
    Here, ``\textit{potential} computational efficiency'' refers that the computational efficiency can be enhanced if one employs efficient implementations that leverage sparsity, such as sort operations or median-finding algorithms, to circumvent unnecessary computations, see \cref{sec:related_works} and  \cite[Section~2]{martins2016softmax} for more discussions.
    }
    } from \cite{martins2016softmax}, compared to its dense counterparts.
    This work extends the theoretical understanding of the correspondence between artificial and biological neural networks to sparse region.
    In addition, the sparse Hopfield layer, a new deep learning component, is introduced  with data-dependent sparsity.

    \item
    Theoretically, we establish provably advantages from sparsity and identify the conditions under which these benefits arise.
    We begin by deriving the closed-form sparse Hopfield energy from the convex conjugation of sparse entropic regularizer. 
    Next, we demonstrate the correspondence between sparse Hopfield retrieval dynamics and sparsemax attention. 
    In addition, we prove the {fast} convergence of the fixed points (also known as memory patterns, attractor states in literature) for the retrieval dynamics and establish the exponential (in pattern size) memory capacity  lower bound with \textit{tighter} retrieval error bound, \textit{compared} with modern Hopfield models.

    \item
    Empirically, we conduct synthetic and realistic experiments to verify our theoretical results and proposed methodology.
    Specifically, the sparse Hopfield model outperforms the dense Hopfield model and machine learning baselines in \textit{sparse} Multiple Instance Learning (MIL), time series prediction and neural machine translation problems. 
    This is observed with both \textit{sparse} synthetic and real-world datasets, where the baselines tend to fall short.
    Moreover, even in cases without data sparsity, our proposed model delivers performance on par with its dense counterpart.
\end{itemize}
To the best of our knowledge, we are the first to propose a sparse Hopfield model whose retrieval dynamics is equivalent to sparse attention mechanism with provably computational advantages. 
Methodologically, the proposed model complements existing Hopfield-based DNN architectures \cite{hoover2023energy,paischer2022history,seidl2022improving,furst2022cloob,ramsauer2020hopfield} by introducing a sparse Hopfield layer into deep learning models.

\textbf{Organization.} 
In \cref{sec:preliminary}, the sparse Hopfield model is introduced.
In \cref{sec:method}, 
the memory capacity is discussed.
In \cref{sec:exp}, experimental studies are conducted.
In \cref{sec:conclusion}, concluding discussions are provided.
Additionally, related works  and limitations are discussed in \cref{sec:related_works}.

\textbf{Notations.}
We write $\Braket{\ba,\bb}\coloneqq \ba^\sT \bb$ as the inner product for vectors $\ba,\bb\in \R^d$.
The index set $\{1,\cdots,I\}$ is denoted by $[I]$, where $I\in\mathbb{N}_+$. 
The spectral norm is denoted by $\norm{\cdot}$, which is equivalent to the $l_2$-norm when applied to a vector. 
Throughout this paper, we denote the memory patterns (keys) by $\bxi\in\R^d$ and the state/configuration/query pattern by $\bx\in\R^d$, and $\bm{\Xi}\coloneqq\(\bxi_1,\cdots,\bxi_M\)\in \R^{d\times M}$ as shorthand for stored memory (key) patterns $\{\bxi_\mu\}_{\mu\in[M]}$.
Moreover, \blue{we set  norm $n\coloneqq\norm{\bx}$ be the norm of the 
query pattern, and $m\coloneqq \Max_{\mu\in[M]}\norm{\bxi_\mu}$ be the largest norm of memory patterns.
We also provide a nomenclature table  (\cref{tab:nomenclature}) in the appendix.}

\section{Sparse Hopfield Model}
\label{sec:preliminary}
In this section, we introduce the sparse Hopfield energy from convex conjugate of entropic regularizer, and then the sparse retrieval dynamics.
In this paper we only consider the Gini entropic regularizer corresponding to the sparsemax distribution \cite{martins2016softmax}.

{\color{Blue}
Let $\bx\in\R^d$ represent the query pattern, and let $\bm{\Xi}\coloneqq\(\bxi_1,\cdots,\bxi_M\)\in \R^{d\times M}$ denote the memory patterns. 
The objective of the Hopfield models is to store the memory patterns $\bm{\Xi}$ and then retrieve a specific memory pattern $\bxi_\mu$ based on a given query $\bx$. Consequently, any Hopfield model  consist of two main components: an \textit{energy function} $\calH(\bx)$, encoding memories into its local minima, and a \textit{retrieval dynamics} $\calT(\bx)$, which retrieves a memory by iteratively minimizing $\calH(\bx)$ when initialized with a query.
We provide a visualization of this 
procedure in \cref{fig:energylandscape}. 
The construction of the energy function $\calH(\bx)$ is straightforward. As emphasized in \cite{krotov2016dense}, the memories can be easily encoded into $\calH(\bx)$ through the \textit{overlap-construction}: $\calH(\bx)=F(\bm{\Xi}^\sT\bx)$, where $F:\R^M\to \R$ is a smooth function. This ensures that the memories $\{\bxi_\mu\}_{\mu\in[M]}$ are located at the stationary points of $\calH(\bx)$, since $\grad_{\bx}F(\bm{\Xi}^\sT\bx)|_{\bxi_{\mu}}=0$ for all $\mu\in[M]$. Different choices of $F$ lead to different Hopfield models, as demonstrated in \cite{krotov2016dense,demircigil2017model,ramsauer2020hopfield, krotov2020large}.
However, finding a corresponding retrieval dynamics, $\calT$, for a given energy $\calH(\bx)$, is generally more challenging. This is because $\calT$ needs to satisfy two conditions to ensure successful memory retrieval: (i) To ensure consistent retrieval, an appropriate $\calT$ should monotonically minimize $\calH(\bx)$ when iteratively applied. (ii) To ensure accurate retrieval, an appropriate $\calT$ should align its fixed points (the points where iterative application terminates) with the stationary points of $\calH(\bx)$. 

To this end, we introduce the sparse Hopfield model, providing a principled construction for $\calH$ and $\calT$. 
This model not only fulfills the aforementioned desirable properties, but also enables more robust and faster memory retrieval compared to the modern Hopfield model \cite{ramsauer2020hopfield}.
}

\subsection{Sparse Hopfield Energy}

Let $\bx\in\R^d$ be the query pattern, and $\bm{\Xi}\coloneqq\(\bxi_1,\cdots,\bxi_M\)\in \R^{d\times M}$ be the memory patterns.
We introduce the sparse Hopfield energy as
\setlength{\abovedisplayskip}{0pt}
\setlength{\belowdisplayskip}{0pt}
\begin{align}
\label{eqn:H_sparsemax}
\calH(\bx)=-\Psi^\star\(\beta \bm{\bm{\Xi}}^\sT \bx\) +\half \Braket{\bx,\bx}
,
\end{align}
with
$
\Psi^\star(\bz)
    \coloneqq
    \half \norm{\bz}^2-\half\norm{\Sparsemax(\bz)-\bz}^2+\half
$,
where $\Sparsemax(\cdot)$ is defined as follows.
Let $\bz,\bp\in\R^M$, and $\Delta^{M}\coloneqq\{\bp\in\R^M_+ \mid \sum_\mu^M p_\mu=1\}$ be the $(M-1)$-dimensional unit simplex.

\begin{definition}[Sparsemax in Variational Form \cite{martins2016softmax}, also see \cref{remark:sparsemax}]
\label{def:variational_sparsemax}
\begin{align}
\label{eqn:opt_sparsemax}
\Sparsemax(\bz)\coloneqq\argmin_{\bp\in\Delta^M}\norm{\bp-\bz}^2
=
\argmax_{\bp\in\Delta^M}\[\bp^\sT\bz -\Psi(\bp)\],
\end{align}
where $\Psi(\bp)\coloneqq -\half \sum_\nu^M p_\nu(1-p_\nu)$ is the negative Gini entropy or Gini entropic regularizer.
\end{definition}
\vspace{0.5em}
{\color{Blue}
\begin{remark}
    Recall that, the variational form \eqref{eqn:opt_sparsemax} is in fact general, that applies to various entropic regularizers, as discussed in \cite{peters2019sparse, wainwright2008graphical}.  
    The choice of $\Psi$ determines the resulting sparse probability distribution. For instance, if we choose the Gibbs' entropic regularizer $\Psi_{\text{Gibbs}}=-\sum_{\nu}^M p_\nu\ln p_\nu$, \eqref{eqn:opt_sparsemax} reduces to the standard softmax distribution.
\end{remark}
}
{\color{Blue}
\textbf{Overview of Theoretical Results.}
At first glance, the energy function \eqref{eqn:H_sparsemax} may seem peculiar. However, it indeed represents a non-trivial sparse Hopfield model with appealing properties, including:
\vspace{-1em}
\begin{enumerate}[leftmargin=2em,label=(\roman*),before=\vspace{-0.5em}, after=\vspace{-0.5em}]
\setlength\itemsep{0.1em}
    \item  
    In response to challenge (C1) \& (C2),
    as we shall see in \cref{sec:sparse_retrieval_dynamics},
    the energy \eqref{eqn:H_sparsemax} leads to a sparse retrieval dynamics that not only retrieves memory by monotonically decreasing (\cref{thm:retrieval_dyn}) to its stationary points (\cref{coro:convergence_sparse}), but also associates with sparsemax attention through its single-step approximation (\cref{remark:hopfieldatt});
    \item  In response to challenge (C3), as we shall see in \cref{sec:method}, it indulges fast convergence of retrieval (\cref{coro:exp_suppressed_error}), exponential-in-$d$ memory capacity akin to modern Hopfield models (\cref{thm:memory_capacity}). Notably, it accomplishes these with a tighter retrieval error bound (\cref{thm:main_result}).
\end{enumerate}
We reveal each of these properties in the following sections.
}

\subsection{Sparse Retrieval Dynamics and Connection to Sparse Attention}
\label{sec:sparse_retrieval_dynamics}
The optimization problem $\argmax_{\bp\in\Delta^M}\[\bp^\sT\bz -\Psi(\bp)\]$ does not necessarily have a closed-form solution for arbitrary $\Psi$.
However, a family of $\Psi$ has been investigated in literature \cite{correia2019adaptively,martins2016softmax} with closed-form solutions derived, including the $\Sparsemax(\cdot)$.
\begin{fact}[Proposition~1 of \cite{martins2016softmax}]
\label{lemma:sparsemax_exact}
Let $\bz\in\R^M$.
Denote $[a]_+\coloneqq \Max\{0,a\}$, $z_{(\nu)}$ the $\nu$'th element in a sorted descending $z$-sequence $\bz_{\text{sorted}}\coloneqq z_{(1)}\ge z_{(2)}\ge\ldots\ge z_{(M)}$, and 
$
\kappa(\bz)\coloneqq\Max\big\{k\in [M]\;\big\vert \; 1+kz_{(k)}>\sum_{\nu\le k}z_{(\nu)}\big\}.
$
The optimization problem(s) \eqref{eqn:opt_sparsemax} has closed-form solution 
\bea
\label{eqn:sparsemax_exact}
\Sparsemax(\bz)=\[\bz-\tau(\bz)\mathbf{1}
_M\]_+,
\eea
where $\tau:\R^M \to \R$ is the threshold function 
$\tau(\bz)=\[\(\sum_{\nu\le \kappa(\bz)}z_{(\nu)}\)-1\]/\kappa(\bz)
$,
satisfying $\sumM \[z_\mu-\tau(\bz)\]_+=1$ for all $\bz$.
Notably, $\kappa(\bz)=\abs{S(\bz)}$ where $S(\bz)=\{\mu\in[M]\;|\;\Sparsemax_\mu(\bz)>0\}$ is the support set of $\Sparsemax(\bz)$.
\end{fact}
In this case, we present the following theorem to derive the convex conjugate of $\Psi$ in closed-form:
\begin{theorem}[Convex Conjugate of Negative Gini Entropy]
\label{thm:main_result}
Let $F(\bp)\coloneqq\Braket{\bp,\bz} -\Psi(\bp)$
with $\Psi$ being the negative Gini entropy, $\Psi(\bp)=\half \norm{\bp}^2 -\half$.
The convex conjugate of $\Psi(\bp)$ is
\begin{align}\label{eqn:PsiStar_sparse_hopfield}
    \Psi^\star(\bz)\coloneqq
    \Max_{\bp\in\Delta^M}F(\bp,\bz)
    =
    \half \norm{\bz}^2-\half\norm{\bp^\star-\bz}^2+\half,
\end{align}
where $\bp^\star=\Sparsemax(\bz)$ is given by \eqref{eqn:sparsemax_exact}.
\end{theorem}
\begin{corollary}
\label{corollary:Dankins}
    By Danskin’s Theorem,
    $
\grad \Psi^\star(\bz) = \argmax_{\bp\in\Delta^M}F(\bp,\bz)
=\Sparsemax(\bz).
$
\end{corollary}
\begin{proof}
\vspace{-1em}
A detailed proof is shown in \cref{proof:main_result}.
\end{proof}
\vspace{-1em}
\cref{thm:main_result} and \cref{corollary:Dankins} not only provide the intuition behind the sparse Hopfield energy \eqref{eqn:H_sparsemax} --- the memory patterns are stored in local minima aligned with the overlap-function constructions (i.e.  $\norm{\bm{\bm{\Xi}}^\sT \bx}^2=\sumM \Braket{\bxi_\mu, \bx}^2$) in \cite{ramsauer2020hopfield,demircigil2017model,krotov2016dense} --- but also prepare us for the following corresponding sparse retrieval dynamics.

\begin{lemma}[Sparse Retrieval Dynamics]
\label{thm:retrieval_dyn}
Let $t$ be the iteration number.
The energy \eqref{eqn:H_sparsemax} can be monotonically decreased by the following sparse retrieval dynamics over $t$:
\vspace{3pt}
\bea
\label{def:sparse_retrieval_dyn}
\calT\(\bx_t\)\coloneqq\grad_\bx\Psi\(\beta\bm{\bm{\Xi}}^\sT \bx\)\big\vert_{\bx_{t}}=\bm{\bm{\Xi}} \Sparsemax\(\beta\bm{\bm{\Xi}}^\sT \bx_t\)=\bx_{t+1}.
\eea
\end{lemma}
\begin{proof}[Proof Sketch]
\vspace{-0.5em}
To show monotonic decreasing property, we first derive the sparse retrieval dynamics by utilizing \cref{thm:main_result}, \cref{corollary:Dankins}, along with the convex-concave procedure \cite{yuille2003concave,yuille2001concave}.
Then, we show the monotonicity of $\calH$ by constructing a iterative upper bound of $\calH$ which is convex in $\bx_{t+1}$ and thus, can be lowered iteratively by the convex-concave procedure.
A detailed proof is shown in the \cref{sec:thm_retreival}.
\end{proof}
\vspace{-0.5em}
\begin{remark}
\label{remark:hopfieldatt}
    Similar to \cite{ramsauer2020hopfield}, \eqref{def:sparse_retrieval_dyn} is equivalent to sparsemax attention \cite{martins2016softmax} when the $\calT$ is applied only once, see \cref{sec:T_is_attn} for more details.
    \blue{Importantly, $\beta$ acts as a scaling factor for the energy function, often referred to as the ``inverse temperature''. It influences the sharpness of energy landscape \cref{eqn:H_sparsemax}, thereby controlling the dynamics. 
    High $\beta$ values, corresponding to low temperatures, encourage that the basins of attraction for individual memory patterns remain distinct, leading to easier retrieval.}
\end{remark}
Notably, since $\norm{\bm{\bm{\Xi}}^\sT \bx}^2=\sumM \Braket{\bxi_\mu, \bx}^2$, \eqref{def:sparse_retrieval_dyn} implies that the local optimum of $\calH$ are located near the patterns ${\bxi_\mu}$.
Different from previous studies on binary Hopfield models \cite{demircigil2017model,krotov2016dense},
for continuous patterns, we adopt the relaxed definition from \cite{ramsauer2020hopfield}\footnote{
Recall that a fixed point of $\calT$ with respect to $\calH$ is a point where $\bx = \calT(\bx)$, and a generalized fixed point is a point where $\bx\in\calT(\bx)$. For more details, refer to \cite{sriperumbudur2009convergence}.} to rigorously analyze the memory retrieval, and the subsequent lemma.
\begin{definition} [Stored and Retrieved]
\label{def:stored_and_retrieved}
Assuming that every pattern $\bxi_\mu$ surrounded by a sphere $S_\mu$ with finite radius $R\coloneqq \half \Min_{\mu\neq\nu;\mu,\nu\in[M]}\norm{\bxi_\mu-\bxi_\nu}$, we say $\bxi_\mu$ is \textit{stored} if there exists a generalized fixed point of $\calT$, $\bx^\star_\mu \in S_\mu$, to which all limit points $\bx \in S_\mu$ converge to, and $S_\mu \cap S_\nu=\emptyset$ for $\mu \neq \nu$. 
We say $\bxi_\mu$ is $\epsilon$-\textit{retrieved} by $\calT$ with $\bx$ for an error\footnote{The retrieval error has a naive bound $\epsilon\coloneqq\Max\big\{\norm{\bx-\bxi_\mu},\norm{\bxi_\mu-\bx^\star_\mu}\big\}$ by interpolating from $\bx$ to $\bxi_\mu$.} $\epsilon$, if $\norm{\calT(\bx)-\bxi_\mu}\le \epsilon$.
\end{definition}
\blue{
\cref{def:stored_and_retrieved} sets the threshold for a memory pattern $\bxi_\mu$ to be considered \textit{stored} at a fixed point of $\calT$, $\bx^\star_\mu$.
However, this definition does not imply that the fixed points of $\calT$ are also stationary points of the energy function $\calH$. 
In fact, monotonicity of \eqref{def:sparse_retrieval_dyn} does not assure the existence of stationary points of energy $\calH$ \cite{sriperumbudur2009convergence}. 
To establish a well-defined Hopfield model, we need two types of convergence guarantees.
The first is the convergence between $\bx^\star_\mu$ and $\bxi_\mu$, which ensures that the retrieved memory is close to the stored memory. The second is the convergence of $\calH$ to its stationary points through the dynamics of $\calT$, which ensures that the system reaches a state of minimal energy.
The following lemma provides the convergence results for both.
}
\begin{lemma}[Convergence of Retrieval Dynamics $\calT$]
\label{coro:convergence_sparse}
Suppose $\calH$ is given by \eqref{eqn:H_sparsemax} and $\calT(\bx)$ is given by \eqref{def:sparse_retrieval_dyn}.
For any sequence $\{\mathbf{x}_t\}_{t=0}^{\infty}$ defined by $\mathbf{x}_{t'+1} = \mathcal{T}(\mathbf{x}_{t'})$, all limit points of this sequence are stationary points if they are obtained by iteratively applying $\mathcal{T}$ to $\mathcal{H}$.

\end{lemma}
\begin{proof}[Proof Sketch]
\vspace{-1em}
We verify and utilize Zangwill's global convergence theory \cite{zangwill1969nonlinear} for iterative algorithms $\calT$, to first show that all the limit points of $\{\bx_t\}^\infty_{t=0}$ are generalized fixed points and $\lim_{t \to \infty} \calH\(\bx_t\)=\calH\(\bx^\star\)$, where $\bx^\star$ are some generalized fixed points of $\calT$. 
Subsequently, by \cite[Lemma~5]{sriperumbudur2009convergence}, we show that $\{\bx^\star\}$ are also stationary points of $\Min_\bx \[\calH\]$, and hence 
$\calH$ converges to local optimum.
A detailed proof is shown in \cref{sec:convergence}.
\end{proof}
\vspace{-1em}
Intuitively, \cref{coro:convergence_sparse} indicates that the energy function converges to local optimum, i.e. $
    \lim_{t\to \infty}\calH\(\bx_t\)\to \calH\(\bx^\star\),
    $
    where $\bx^\star$ are stationary points of $\calH$.
Consequently, it offers formal justifications for the retrieval dynamics \eqref{def:sparse_retrieval_dyn} to retrieve stored memory patterns $\{\bxi_\mu\}_{\mu\in[M]}$:
for any query (initial point) $\bx$, $\calT$ monotonically and iteratively approaches stationary points of $\calH$, where the memory patterns $\{\bxi_\mu\}_{\mu\in[M]}$ are stored.
As for the retrieval error,
we provide the following theorem stating that $\calT$ achieves a lower retrieval error compared to its dense counterpart.
\begin{theorem}[Retrieval Error]
\label{coro:eps_sparse_dense}
Let $\calT_{\text{Dense}}$ be the retrieval dynamics of the dense modern Hopfield model \cite{ramsauer2020hopfield}.
It holds
$\norm{\calT(\bx)-\bxi_\mu} 
\leq
\norm{\calT_{\text{Dense}}(\bx)-\bxi_\mu}$ for all $\bx\in S_\mu$.
Moreover, 
\bea
\label{eqn:sparse_error_bound}
\norm{\calT(\bx)-\bxi_\mu}\le
m+d^{\nicefrac{1}{2}}m\beta \[\kappa \(\Max_{\nu\in[M]}\Braket{\bxi_\nu,\bx}-\[ \bm{\Xi}^\sT \bx\]_{(\kappa)}\)+\frac{1}{\beta}\],
\eea
where $\[ \bm{\Xi}^\sT \bx\]_{(\kappa)}$ is the $\kappa$th-largest element of $\bm{\Xi}^\sT \bx \in \R^M$ following the sparsemax definition \eqref{eqn:sparsemax_exact}.
\end{theorem}
\begin{proof}
\vspace{-1em}
A detailed proof is shown in \cref{sec:pf_eps_sparse_dense}.
\end{proof}
\vspace{-1em}
Interestingly, \eqref{eqn:sparse_error_bound} is a sparsity dependent bound\footnote{
Notably, $\norm{\calT(\bx)-\bxi_\mu}$ is also upper-bounded by a sparsity-independent but $M,\beta$-dependent bound 
\begin{align}
\label{eqn:error_bound_exp}
\norm{\calT(\bx)-\bxi_\mu} 
\leq 
\norm{\calT_{\text{Dense}}(\bx)-\bxi_\mu}\le  2m(M-1) \exp{-\beta \(\Braket{\bxi_\mu,\bx}-\Max_{\nu\in[M],\nu\neq \mu}\Braket{\bxi_\mu,\bxi_\nu}\)}.
\end{align}
}.
By denoting $n\coloneqq \norm{\bx}$, the second term on the RHS of \eqref{eqn:sparse_error_bound} is dominated by the sparsity dimension $\kappa$ as it can be expressed as $ \kappa \(1-\nicefrac{\[ \bm{\Xi}^\sT \bx\]_{(\kappa)}}{(nm)}\) \propto \alpha\kappa$ with a constant $0\le\alpha\le 2$.
When $\bm{\Xi}^\sT\bx$ is sparse (i.e. $\kappa$ is small), the bound is tighter, vice versa.
{\color{Blue}
\begin{remark}[Faster Convergence]
    Computationally, \cref{coro:eps_sparse_dense} implies that $\calT$ requires fewer iterations to reach fixed points with the same amount of error tolerance compared to $\calT_{\text{dense}}$.
    Namely, $\calT$ retrieves stored memory patterns faster and therefore more efficiently, as evidenced in \cref{fig:convg}.
\end{remark}
\begin{remark}[Noise-Robustness]
    Moreover, in cases of contaminated patterns with noise $\bm{\eta}$, i.e. $\Tilde{\bx}=\bx+\bm{\eta}$ (noise in query) or $\Tilde{\bxi}_\mu=\bxi_\mu+\bm{\eta}$ (noise in memory), the impact of noise $\bm{\eta}$ on the sparse retrieval error \eqref{eqn:sparse_error_bound} is linear, while its effect on the dense retrieval error \eqref{eqn:error_bound_exp} is exponential.
This suggests the robustness advantage of the sparse Hopfield model, as evidenced in \cref{fig:capacity_robustness}.
\end{remark}
}

\subsection{Sparse Hopfield Layers for Deep Learning}
The sparse Hopfield model can serve as a versatile component for deep learning frameworks, given its continuity and differentiability with respect to parameters.
Corresponding to three types of Hopfield Layers proposed in \cite{ramsauer2020hopfield}, we introduce their sparse analogs: 
\textbf{(1)} \texttt{SparseHopfield},  
\textbf{(2)} \texttt{SparseHopfieldPooling}, 
\textbf{(3)} \texttt{SparseHopfieldLayer}. 
Layer \texttt{SparseHopfield}  has  memory (stored or key) patterns $\bm{\Xi}$ and query (state) pattern $\bx$ as inputs, and associates these two sets of patterns via the sparse retrieval dynamics \eqref{def:sparse_retrieval_dyn}. 
This layer regards the transformer attention layer as its one-step approximation, while utilizing the sparsemax 
\cite{martins2016softmax} on attention matrix. 
Layer \texttt{SparseHopfieldPooling} and Layer \texttt{SparseHopfieldLayer} are two variants of \texttt{SparseHopfield}, whose input patterns are memory patterns and query patterns from previous layers or external plugin, respectively. \texttt{SparseHopfieldPooling}, whose query patterns are learnable parameters, can be interpreted as performing a pooling operation over input memory patterns. 
{\color{Blue}\texttt{SparseHopfieldLayer}, by contrast, has learnable memory patterns that maps query patterns to hidden states with sparsemax activation. Thus it can substitute a fully connected layer within deep learning architectures. }
See \eqref{eqn:sparse_hopfield_attention} and the implementation  \blue{\cref{alg:multistep} in  \cref{sec:hopfield}}, and \cite[Section~3]{ramsauer2020hopfield} for more details of these associations.
In \cref{sec:exp}, we apply these layers and compare them with their dense counterparts in \cite{ramsauer2020hopfield} and other baseline machine learning methods.

\section{Fundamental Limits of Memory Capacity of Sparse Hopfield Models}
\label{sec:method}
How many patterns can be stored and reliably retrievable in the proposed model? We address this by decomposing it into to two sub-questions and answering them separately:
\begin{itemize}[leftmargin=2em,before=\vspace{-0.5em}, after=\vspace{-0.5em}]
\setlength\itemsep{0em}
    \item[(A)]
    What is the condition for a pattern $\bxi_\mu$ considered well stored in $\calH$, and correctly retrieved?
    \item  [(B)] 
    What is the number, in expectation, of the the patterns satisfying such condition?
\end{itemize}

For (A), we first introduce the notion of separation of patterns following \cite{ramsauer2020hopfield}, 
\begin{definition}[Separation of Patterns]
\label{def:separation_of_patterns}
The separation of a memory pattern $\bxi_\mu$ from all other memory patterns $\bm{\Xi}$ is defined as its minimal inner product difference to any other patterns:
\vspace{3pt}
\bea
\label{eqn:sep_delta_mu}
\Delta_\mu\coloneqq
\Min_{\nu,\nu\neq \mu}\[\Braket{\bxi_\mu,\bxi_\mu}-\Braket{\bxi_\mu,\bxi_\nu}\]=\Braket{\bxi_\mu,\bxi_\mu}-\Max_{\nu,\nu\neq \mu }\[\Braket{\bxi_\mu,\bxi_\nu}\].
\eea
Similarly, the separation of $\bxi_\mu$ at a given $\bx$ from all memory patterns $\bm{\Xi}$ is given by
\bea
\label{eqn:tilde_sep_delta_mu}
\Tilde{\Delta}_\mu\coloneqq
\Min_{\nu,\nu\neq \mu}\[\Braket{\bx,\bxi_\mu}-\Braket{\bx,\bxi_\nu}\].
\eea    
\end{definition}
and then the well-separation condition for a pattern being well-stored and retrieved.
\begin{theorem}[Well-Separation Condition]
\label{thm:well_separation}
{\color{Blue}
Given the definition of stored and retrieved memories in \cref{def:stored_and_retrieved}, suppose the memory patterns $\{\bxi_\mu\}_{\mu\in[M]}$ are located within the sphere
$
 S_\mu\coloneqq \big\{\bx \; \big\vert\; \norm{\bx-\bxi_\mu}\le R\big\},
$
where the radius $R$ is finite and defined in \cref{def:stored_and_retrieved}
for all $\mu$. 
Then, the retrieval dynamics $\calT$ maps the sphere $S_\mu$ onto itself under the following conditions:
\begin{enumerate}[leftmargin=*,before=\vspace{-0.2em}, after=\vspace{-0.2em}]
\setlength\itemsep{0em}
\item The initial query $\bx$ is located within the sphere $S_\mu$, i.e., $\bx\in S_\mu$.
\item The \textit{well-separation} condition is satisfied, which is given by:
$$\Delta_\mu \ge 
mn+2mR-
\[ \bm{\Xi}^\sT \bx\]_{(\kappa)}-
\frac{1}{\kappa}
\(\frac{R-m-md^{\nicefrac{1}{2}}}{m\beta d^{\nicefrac{1}{2}}}\).$$
\end{enumerate}
}
\end{theorem}
\begin{corollary}
\label{coro:general_well_separation_cond}
Let $\delta\coloneqq \norm{\calT_{\text{Dense}}-\bxi_\mu}-\norm{\calT-\bxi_\mu}$.
The {well-separation} condition can be expressed as
$
\Delta_\mu \ge 
\frac{1}{\beta}\ln(\frac{2(M-1)m}{R+\delta})+2mR
$, which reduces to that of the dense Hopfield model when $\delta=0$.
\end{corollary}
\begin{proof}[Proof Sketch]
\vspace{-0.5em}
The proofs proceed by connecting $\Delta_\mu$ with $\norm{\calT(\bx)-\bxi_\mu}$.
To do so, we utilize \cref{coro:eps_sparse_dense} to incorporate the $\Delta_\mu$-dependent bound on the retrieval error of both sparse and dense Hopfield models \cite{ramsauer2020hopfield}.
A detailed proof is shown in \cref{sec:pf_well_separation}.
\end{proof}
\vspace{-0.5em}
Together with \cref{coro:convergence_sparse}, the well-separated condition serves as the necessary condition for pattern $\bxi_\mu$ to be well-stored at the stationary points of $\calH$, and can be retrieved with at most $\epsilon=R$ by $\calT$, as per \cref{def:stored_and_retrieved}.
We make the following three observations about the blessings from sparsity.
\begin{enumerate}[leftmargin=*,before=\vspace{-0.2em}, after=\vspace{-0.2em}]
\setlength\itemsep{0.5em}
    \item
     In general, to appreciate the blessings of sparsity, we
     rearrange the well-separation condition as
     \bea
     \Delta_\mu \ge 
    2mR+
    \underbrace{\(mn-
    \[ \bm{\Xi}^\sT \bx\]_{(\kappa)}\)}_{\coloneqq\alpha nm\;\;\text{with }0\le\alpha\le2}-
    \frac{1}{\kappa}
    \(\frac{R-m-md^{\nicefrac{1}{2}}}{m\beta d^{\nicefrac{1}{2}}}\),
     \eea
     and 
     observe the two competing terms, $\alpha nm$ and $
     \nicefrac{(R-m-md^{\nicefrac{1}{2}})}{(\kappa m\beta d^{\nicefrac{1}{2}})}$.
     Sparsity proves advantageous when the latter term surpasses the former, i.e. the sparse well-separation condition is consistently lower than its dense counterpart.
     The condition under which sparsity benefits are more likely to emerge (i.e., when the well-separation condition is more readily satisfied) is thereby:
    \bea
    \half\Min_{\mu,\nu\in[M]}\norm{\bxi_\mu-\bxi_\nu}
    \ge 
    md^{\nicefrac{1}{2}}\(1+\alpha \beta nm\kappa\)+m,\quad\text{with }0\le\alpha\le 2.
    \eea
    Intuitively,  the sparser $\bm{\Xi}^\sT\bx$ is, the easier it is for the above condition to be fulfilled.
    \item 
    \textbf{Large $M$ limit:}
    For large $M$,
    the dense well-separation condition (\cref{coro:general_well_separation_cond}) explodes while the sparse one (\cref{thm:well_separation}) saturates to the first three $M$-independent terms.
    This suggests that the hardness of distinguishing patterns can be tamed by the sparsity, preventing an increase of $\Delta_\mu$ with $M$ as observed in the dense Hopfield model.
    We numerically confirm this in \cref{fig:capacity_robustness}.
    \item 
    \textbf{$\beta \to \infty$ Limit:}
    In the region of low temperature, where $\beta \to \infty$ and hence all patterns can be \textit{error-free} retrieved as per \eqref{eqn:error_bound_exp},
    we have $\Delta_\mu \ge 2mR +\alpha nm $ with $0\le\alpha\le 2$. 
    Here, the second term on the RHS represents the sparsity level of $\bm{\Xi}^\sT \bx$, i.e. a smaller $\alpha$ indicates a higher degree of sparsity in $\bm{\Xi}^\sT \bx$.
    Hence, the higher the sparsity, the easier it is to separate patterns.

\end{enumerate}

For (B), equipped with \cref{thm:well_separation} and \cref{coro:general_well_separation_cond}, we provide a lower bound for the number of patterns being well-stored and can be \textit{at least} $R$-retrieved in the next lemma\footnote{
Following the convention in memory capacity literature \cite{ramsauer2020hopfield,demircigil2017model,krotov2016dense}, we assume that all memory patterns $\{\bxi_\mu\}$ are sampled from a $d$-sphere of radius $m$.}:
\begin{lemma}[Memory Capacity Lower Bound]
\label{thm:memory_capacity}
Let $1-p$ be the probability of successfully storing and retrieving a pattern.
The number of patterns randomly sampled from a sphere of radius $m$ that the sparse Hopfield model can store and retrieve is lower-bounded by
\bea
M \ge \sqrt{p}C^\frac{d-1}{4},
\eea
where $C$ is the solution to $C=\nicefrac{b}{W_0({\exp\{a+\ln{b}\}})}$ with $W_0(\cdot)$ being the principal branch of Lambert $W$ function, $a\coloneqq (\nicefrac{4}{d-1})\(\ln\[\nicefrac{2m(\sqrt{p}-1)}{(R+\delta)}\]+1\)$ and $b\coloneqq \nicefrac{4m^2\beta}{5(d-1)}$.
For sufficiently large $\beta$, the sparse Hopfield model exhibits a larger lower bound on the exponential memory capacity compared to its dense counterpart \cite{ramsauer2020hopfield}:
$
M\ge M_{\text{Dense}}.
$
\end{lemma}
\begin{proof}[Proof Sketch]
\vspace{-1em}
Our proof is built on \cite{ramsauer2020hopfield}.
The high-level idea is to \blue{utilize} the separation of random patterns sampled from spheres \cite{cai2012phase,brauchart2018random} and the asymptotic expansion of the Lambert $W$ function \cite{corless1996lambert}. Firstly, we link the well-separation condition to cosine similarity distance, creating an inequality for the probability of a pattern being well-stored and retrieved.
Next, we identify and prove conditions for the exponential memory capacity $M=\sqrt{p}C^{\nicefrac{(d-1)}{4}}$ to hold. 
Finally, we analyze the scaling behaviors of $C$ using its asymptotic expansion and show that $M\ge M_{\text{Dense}}$.
A detailed proof is shown in \cref{sec:pf_thm_main}.
\end{proof}
\vspace{-0.8em}
Intuitively, the benefits of sparsity arises from the increased energy landscape separation provided by the sparse Hopfield energy function, which enables the separation of closely correlated patterns, resulting in a tighter well-separation condition for distinguishing such patterns and hence a larger lower bound on the memory capacity. 
Moreover, the sparse Hopfield model also enjoys the properties of fast convergence and exponentially suppressed retrieval error provided by the following corollary.
\begin{corollary}[Fast Convergence and Exponentially Suppressed Retrieval Error] 
\label{coro:exp_suppressed_error}
For any query $\bx$, $\calT$ approximately retrieves a memory \blue{pattern} $\bxi_\mu$ with retrieval error $\epsilon$ exponentially suppressed by $\Delta_\mu$:
$
\norm{\calT(\bx)-\bxi_\mu}
\le 
2 m(M-1)\exp{-\beta\(\Delta_\mu-2m\Max\[\norm{\bx-\bxi_\mu},\norm{\bx-\bx^\star_\mu}\]\)}.
\nonumber
$
\end{corollary}
\begin{proof}
\vspace{-0.5em}
This results from  \cref{coro:eps_sparse_dense}, \cref{coro:convergence_sparse}, and \cite[Theorem~4]{ramsauer2020hopfield}.
\end{proof}
\vspace{-0.8em}
\cref{coro:exp_suppressed_error} suggests that, with a sufficient $\Delta_\mu$, $\calT$ can approximately retrieve patterns after a single \textit{activation}, allowing the integration of sparse Hopfield models into deep learning architectures similarly to \cite{hoover2023energy,seidl2022improving,furst2022cloob,ramsauer2020hopfield}.

\section{Proof of Concept Experimental Studies}
\label{sec:exp}
We demonstrate the validity of our theoretical results and method by testing them on various experimental settings with both synthetic and real-world datasets.

\subsection{Experimental Validation of Theoretical Results}
\label{sec:exp_theory}
We conduct experiments to verify our theoretical findings, and report the results in \cref{fig:capacity_robustness}.
For the memory capacity (the top row of \cref{fig:capacity_robustness}), we test the proposed sparse model on retrieving half-masked patterns comparing with the Dense (Softmax) and 10th order polynomial Hopfield models \cite{millidge2022universal,krotov2016dense} on MNIST (high sparsity), Cifar10 (low sparsity) and ImageNet (low sparsity) datasets. 
For all Hopfield models, we set $\beta=1$.\footnote{However, as pointed out in \cite{millidge2022universal}, this is in fact \textit{not} fair to compare modern Hopfield with $\beta=1$ with higher order polynomial Hopfield models.
}
A query is regarded as correctly retrieved if its cosine similarity error is below a set threshold.
In addition, for the robustness against noisy queries (the bottom row of \cref{fig:capacity_robustness}), we inject Gaussian noises with varying variances ($\sigma$) into the images.
Plotted are the means and standard deviations of 10 runs. 
The results show that the proposed sparse Hopfield model excels when memory
patterns exhibit a high degree of sparsity and the signal-to-noise ratio in patterns is low, aligning with our theoretical results.
\begin{figure*}[t]
    \centering
    {    \includegraphics[width=0.32\textwidth]{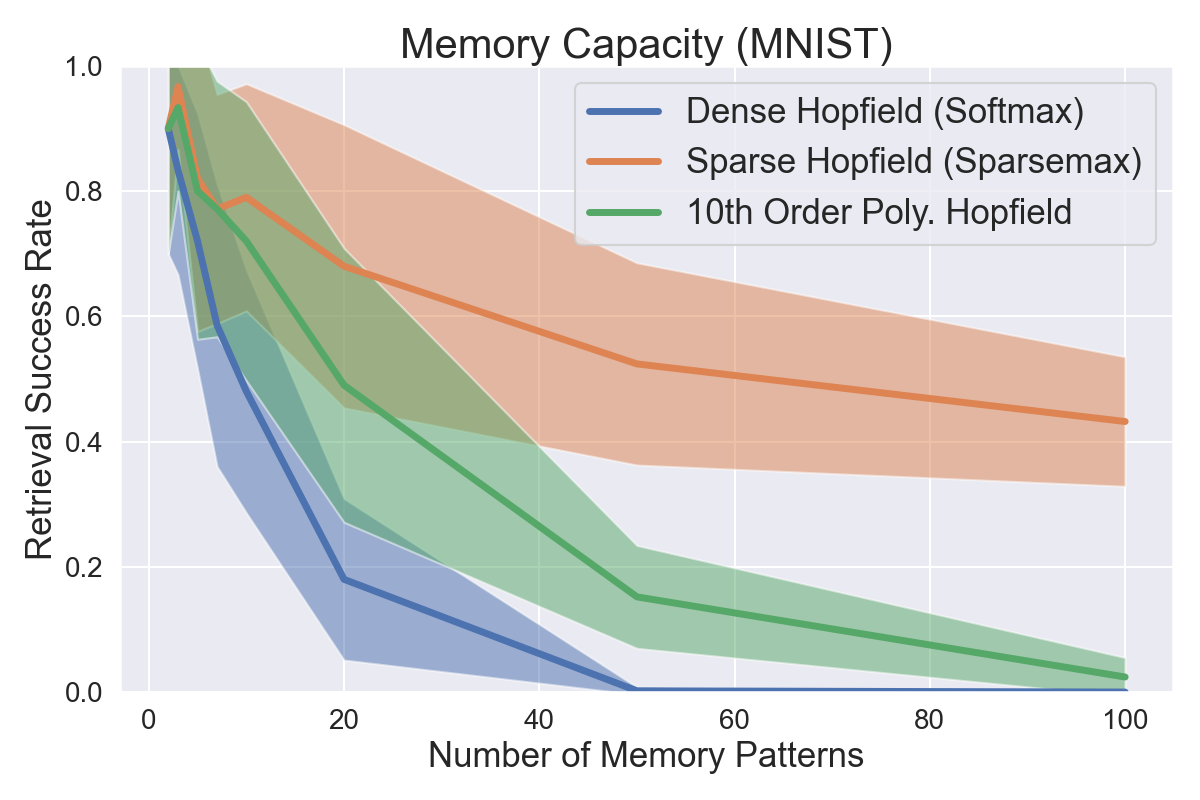}
    }
    \hfill
    {    \includegraphics[width=0.32\textwidth]{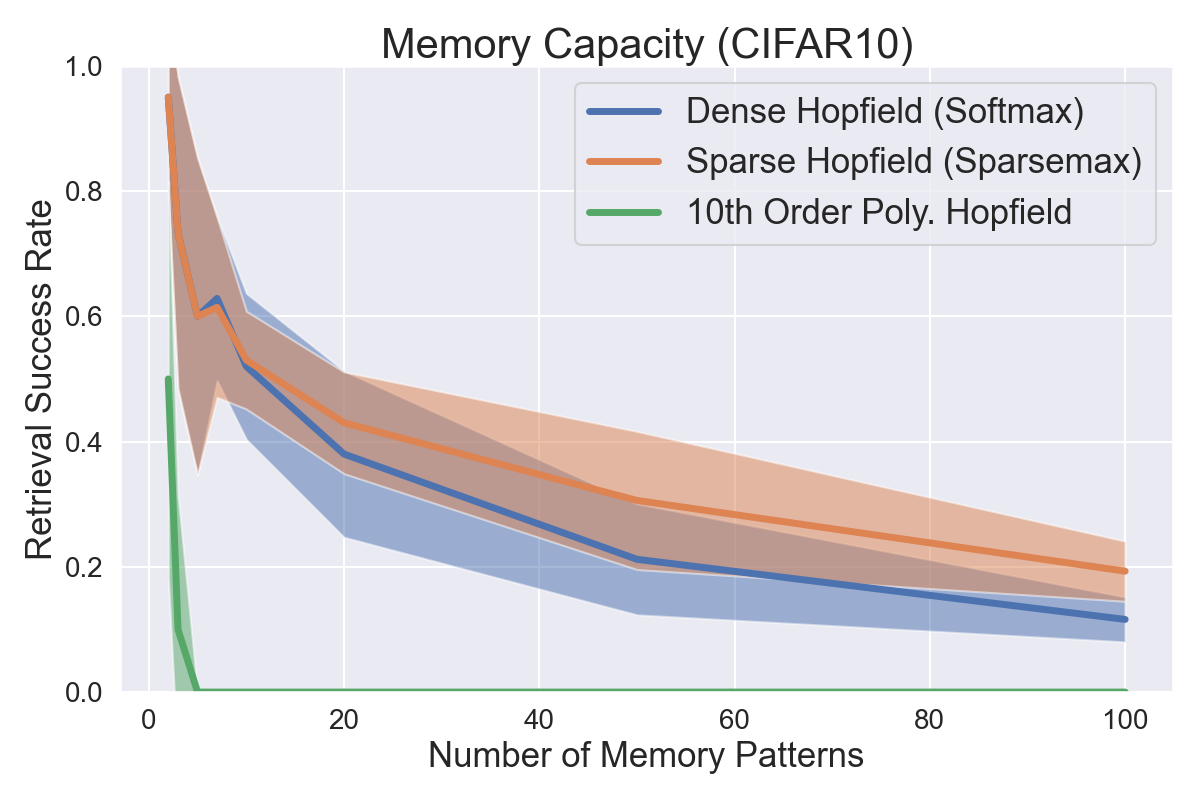}
    }
    \hfill
    {    \includegraphics[width=0.32\textwidth]{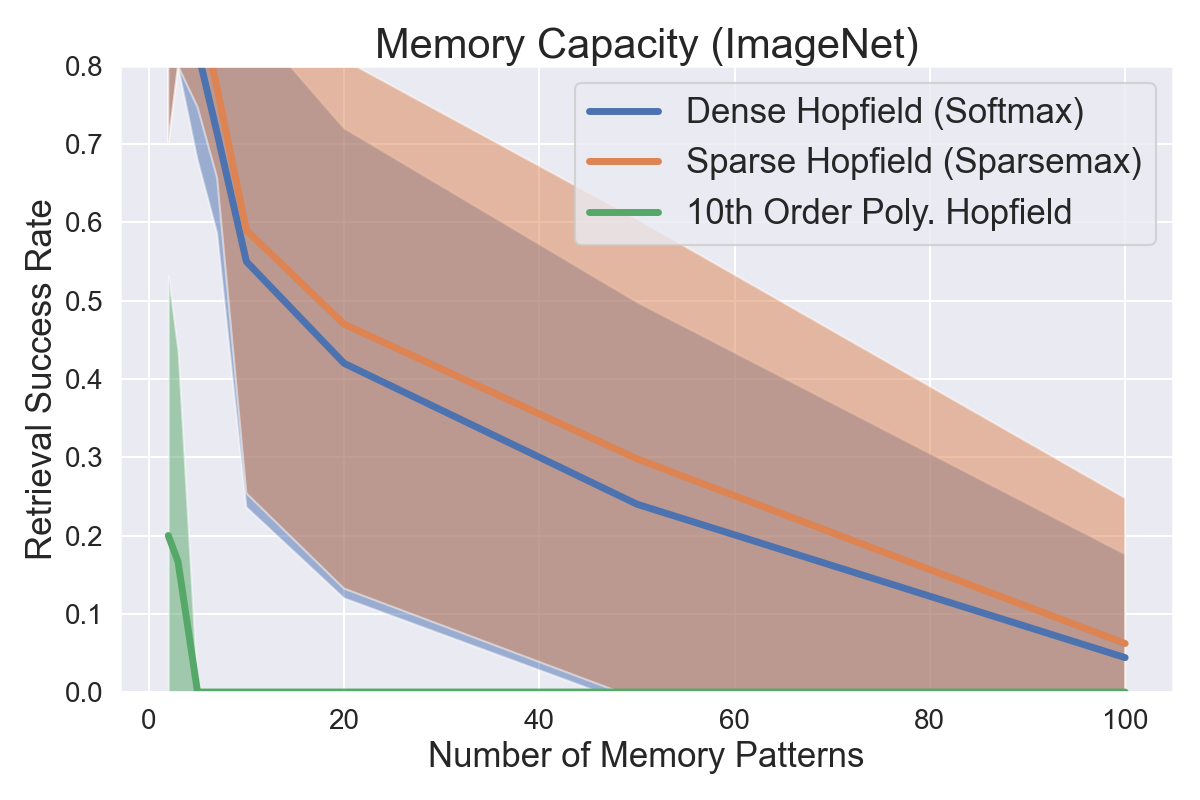}
    }
    \\
    \vspace{-0.3em}
    {    \includegraphics[width=0.32\textwidth]{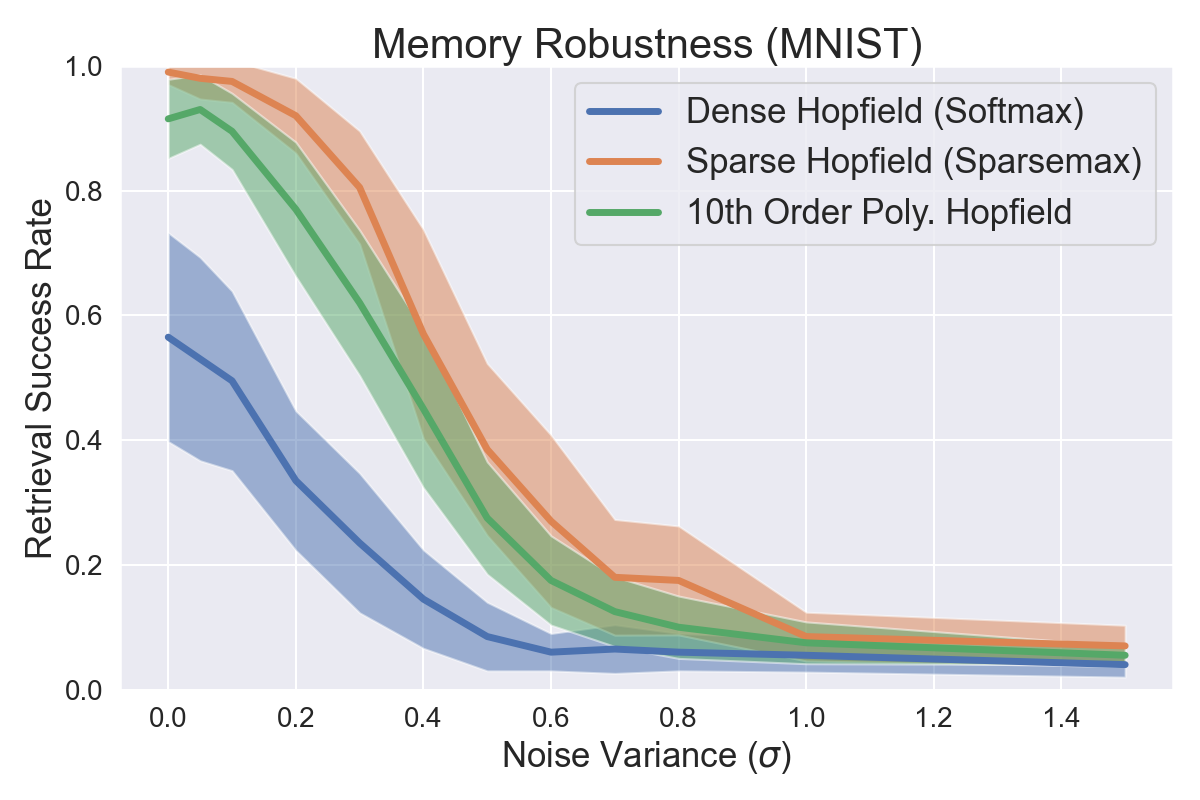}
    }
    \hfill
    {    \includegraphics[width=0.32\textwidth]{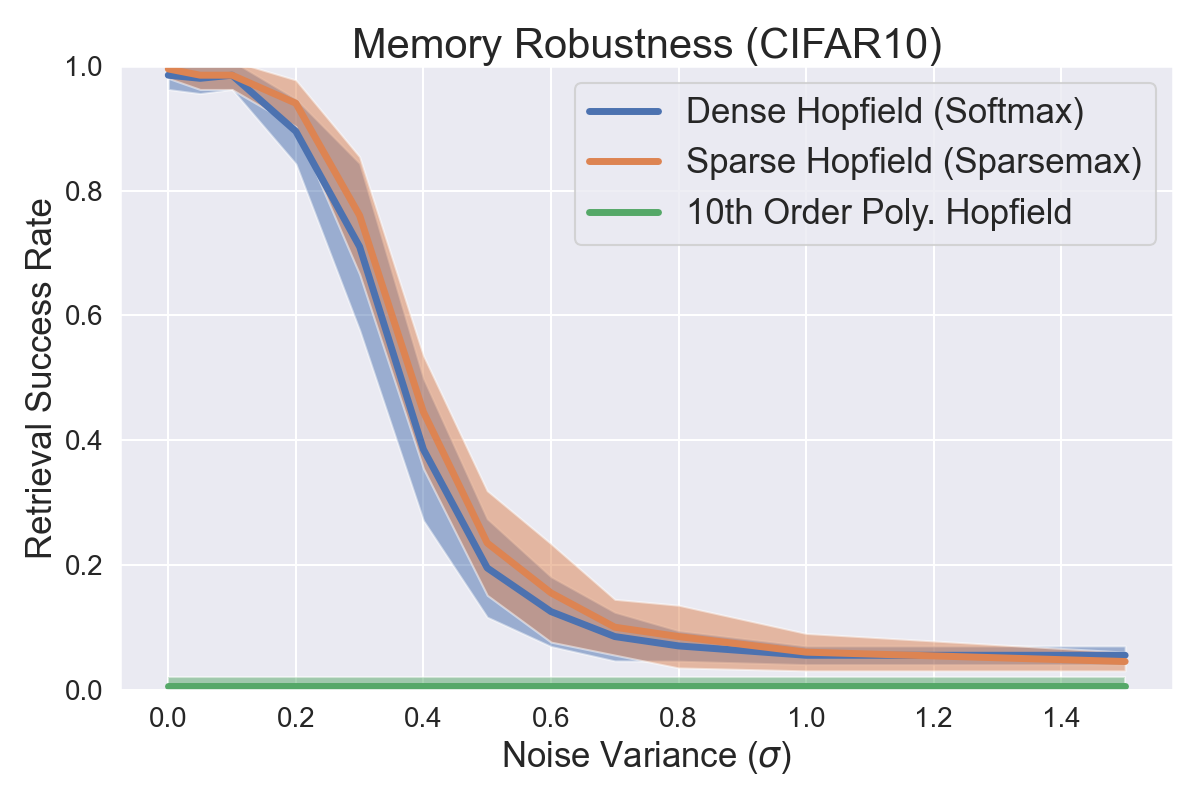}
    }
    \hfill
    {    \includegraphics[width=0.32\textwidth]{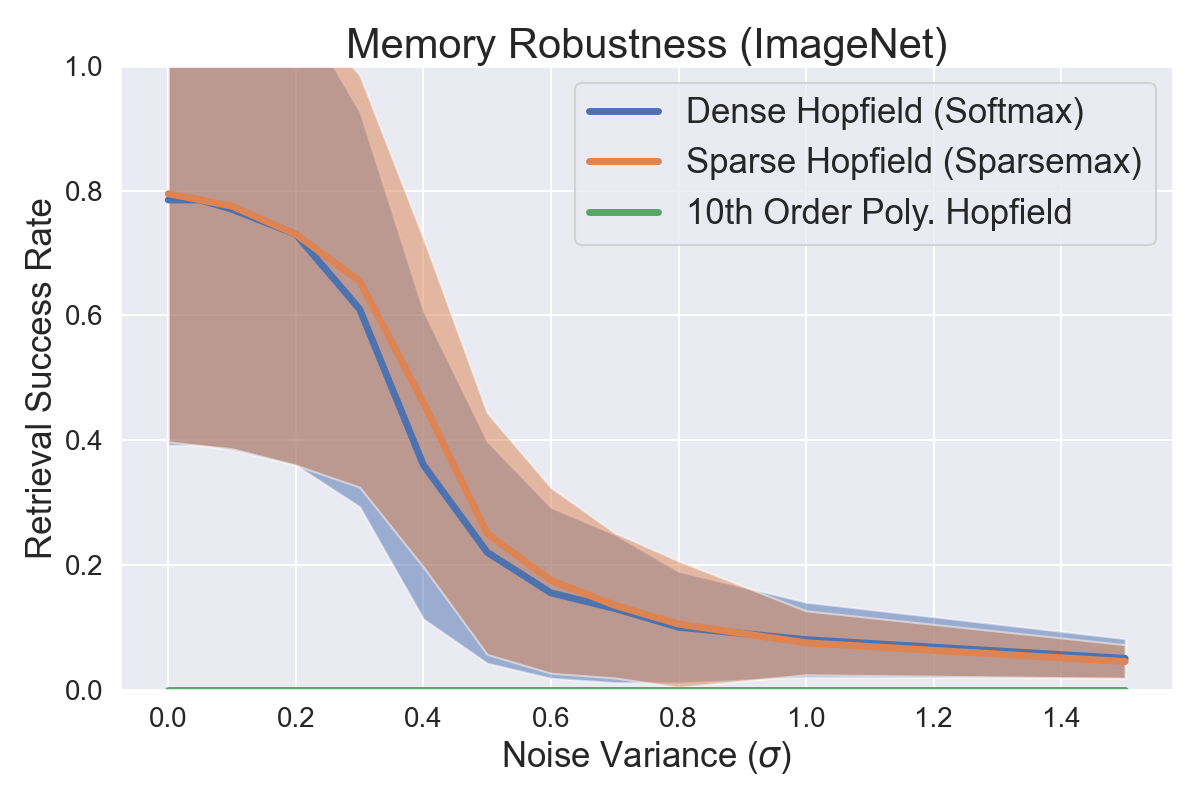}
    }
    \vspace{-2em}
    \caption{\textbf{Top:} Memory Capacity measured by successful half-masked retrieval rates. \textbf{Bottom:} Memory Robustness measured by retrieving patterns with varying levels of Gaussian noise.
    For all Hopfield models, we set $\beta=.01/0.1/0.1$ (for MNIST/CIFAR10/ImageNet) for better visualization.
    A query pattern is deemed correctly retrieved if its squared Euclidean distance is below a set threshold.
    For MNIST/CIFAR10/ImageNet datasets, we set the error thresholds to be 10/20/20 to cope with different sparse levels in data.
    Plotted are the means and standard deviations of 10 runs. 
    The results suggest that the sparse Hopfield model excels when memory patterns exhibit a high degree of sparsity and the signal-to-noise ratio in patterns is low.
    }
    \label{fig:capacity_robustness}
    \vspace{-1.5em}
\end{figure*}

\vspace{-1em}
\subsection{Multiple Instance Learning Tasks}
\citet{ramsauer2020hopfield} point out that the memory-enhanced Hopfield layers present a promising approach for Multiple Instance Learning (MIL) tasks. 
{\color{Blue}
Multiple Instance Learning (MIL) \cite{ilse2018attention,carbonneau2018multiple} is a variation of supervised learning where the training set consists of labeled bags, each containing multiple instances.
The goal of MIL is to predict the bag labels based on the instances they contain, which makes it particularly useful in scenarios where labeling individual instances is difficult or impractical, but bag-level labels are available.
Examples of such scenarios include medical imaging (where a bag could be an image, instances could be patches of the image, and the label could indicate the presence or absence of disease) and document classification (where a bag could be a document, instances could be the words or sentences in the document, and the label could indicate the topic or sentiment of the document).
}
In this subsection, we implement our sparse Hopfield layers and applied them to MIL tasks on one synthetic and four real-world settings.

\vspace{-0.5em}
\subsubsection{Synthetic Experiments} 
\label{sec:syn_exp}
We use a synthetic MIL dataset, the bit pattern dataset, to demonstrate the effectiveness of the sparse Hopfield model. 
Each bag in this synthetic dataset contains a set of binary bit strings. 
The positive bag includes at least one of the positive bit patterns. 
We compare the performance of the \texttt{SparseHopfield} and \texttt{SparseHopfieldPooling}  to their dense counterparts and vanilla attention \cite{vaswani2017attention}. 
We report the mean test accuracy of 10 runs. 
To demonstrate the effectiveness of sparse Hopfield model, we vary two hyperparameters of the bit pattern dataset corresponding to two perspectives: bag sparsity (sparsity in data) and bag size (number of memory patterns, $M$). 
For \textbf{bag sparsity}, we fix the bag size as 200, and inject from 2 to 80 positive patterns in a positive bag, results in 1 to 40 percent of positive patterns in each positive bag.
For \textbf{bag size}, we fix the number of positive pattern in a bag to be 1, and vary bag size from 20 to 300.
We report results of \texttt{SparseHopfieldPooling} in \cref{table:bags}, and implementation details in \cref{appendix:syn}.
A more complete version of \cref{table:bags}, including the results of \texttt{Hopfield} and attention, is in \cref{sec:syn_additional}.
The sparse Hopfield model demonstrates a better performance across all sparsity and all bag sizes.

\begin{table}[h]
\vspace{-1.0em}
    \centering
    \caption{\textbf{Top (Bag Size):} Accuracy comparison on bit pattern dataset for sparse and dense Hopfield model. 
    We report the average accuracy over 10 runs.
    The results suggest that the sparse Hopfield model demonstrates a better performance when facing a bag size increase.
    \textbf{Bottom (Bag Sparsity):}
    Performance comparison on bit pattern dataset for sparse and dense Hopfield model with varying bag sparsity. 
    We report the average accuracy over 10 runs. 
    The results suggest that the sparse Hopfield model demonstrates a better performance across all sparsity.}
    \vspace*{0.05truein} 

\resizebox{ \textwidth}{!}{    
    \begin{tabular}{cccccccc}
    \toprule
    \toprule
    \makecell{\rotatebox{0}{\small Bag Size}} 
         & \multicolumn{1}{c}{20} 
         & \multicolumn{1}{c}{50}
         & \multicolumn{1}{c}{100} 
         & \multicolumn{1}{c}{150} 
         & \multicolumn{1}{c}{200}
         & \multicolumn{1}{c}{300}
        \\
        \midrule
        \makecell{\rotatebox{0}{\small Dense Hopfield Pooling}} 
        & 100.0 $\pm$ 0.00
        & 100.0 $\pm$ 0.00
        & 100.0 $\pm$ 0.00
        & 76.44 $\pm$ 0.23
        & 49.13 $\pm$ 0.01
        & 52.88 $\pm$ 0.01
        \\
        \makecell{\rotatebox{0}{\small Sparse Hopfield Pooling}} 
        & 100.0 $\pm$ 0.00
        & 100.0 $\pm$ 0.00
        & 100.0 $\pm$ 0.00
        & \textbf{99.76 $\pm$ 0.00}
        & \textbf{99.76 $\pm$ 0.00}
        & \textbf{99.76 $\pm$ 0.00}
        \\

    \end{tabular}
    \vspace*{-0.5truein}
}
\resizebox{ \textwidth}{!}{    
    \begin{tabular}{ccccccccc}
    \toprule
    \toprule
    \makecell{\rotatebox{0}{\small Bag Sparsity}} 
         & \multicolumn{1}{c}{1\%} 
         & \multicolumn{1}{c}{5\%} 
         & \multicolumn{1}{c}{10\%} 
         &  \multicolumn{1}{c}{20\%}
         &  \multicolumn{1}{c}{40\%}
        \\
        \midrule
        \makecell{\rotatebox{0}{\small Dense Hopfield Pooling}} 
        & 49.20 $\pm$ 0.00
        & 85.58 $\pm$ 0.10
        & 100.0 $\pm$ 0.00        
        & 100.0 $\pm$ 0.00
        & 99.68 $\pm$ 0.00
        \\
        \makecell{\rotatebox{0}{\small Sparse Hopfield Pooling}} 
        & \textbf{73.40 $\pm$ 0.06}
        & \textbf{99.68 $\pm$ 0.00}
        & 100.0 $\pm$ 0.00
        & 100.0 $\pm$ 0.00
        & \textbf{100.0 $\pm$ 0.00}
        \\ 
        \bottomrule\bottomrule
    \end{tabular}
    \vspace*{-0.5truein}
}
\label{table:bags}
\vspace{-1em}
\end{table}

\textbf{Convergence Analysis.}
In \cref{fig:convg}, we numerically examine the convergence of the sparse and dense Hopfield models, plotting their loss and accuracy for the \textbf{bag size} tasks in above on the bit pattern dataset.
We include multiple bag sizes to assess the effect of increasing memory patterns (i.e. $M$) on the loss curve.
The plotted are the loss and accuracy curves of \texttt{SparseHopfieldPooling}.
\blue{We refer results of \texttt{Hopfield} and more details to 
\cref{sec:syn_abliation_convegence}.}
The results (\cref{fig:convg}) show that, sparse Hopfield model surpasses its dense counterpart in all bag sizes. 
Moreover, for the same bag size, the sparse Hopfield model always reaches the minimum validation loss faster than dense Hopfield model, validating our \cref{eqn:sparse_error_bound}.

\textbf{Sparsity Generalization.}
We also evaluate the models' generalization performance with shifting information sparsity, by training dense and sparse Hopfield models with a specific bag sparsity and testing them on the other.
We report the results in \cref{table:sparsity_generalization} and refer more details to \cref{sec:syn_abliation_convegence}.
\begin{figure*}[h]
\vspace{-1em}
    \centering
    \hfill
    {    \includegraphics[width=1\textwidth]{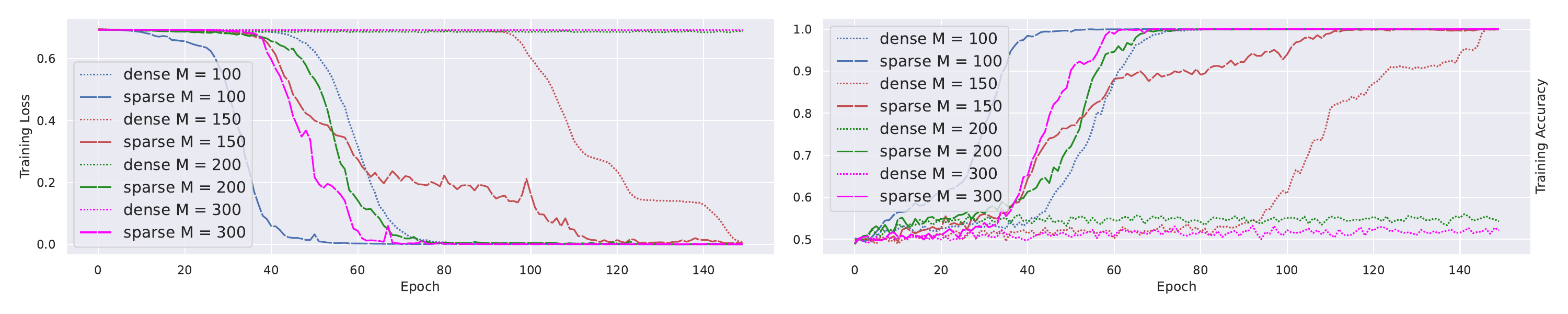}
    }
    \\
    \vspace{-2em}
    \hfill
    {    \includegraphics[width=1\textwidth]{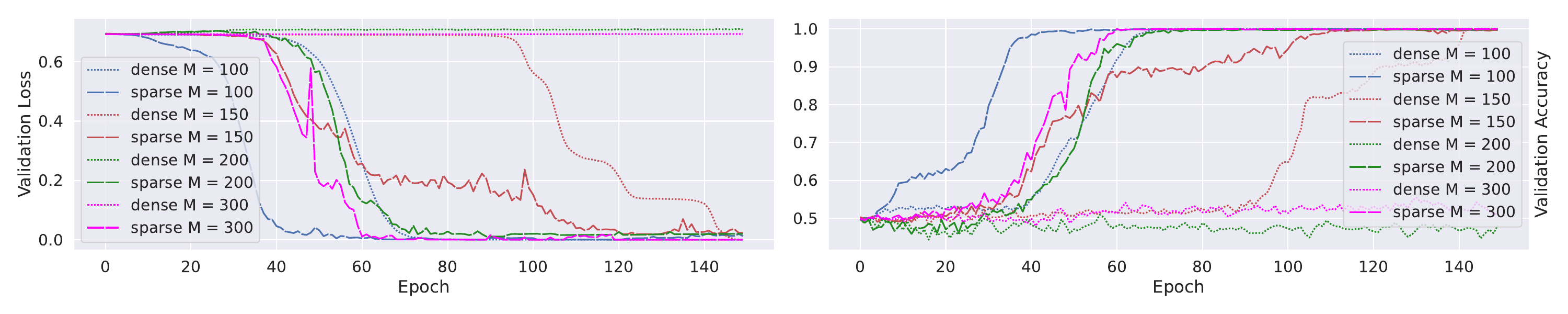}
    }
    \vspace{-2em}
    \caption{\textbf{Top:} The training loss and accuracy curve of dense and sparse Hopfield models with different bag sizes. 
    \textbf{Bottom:} The validation loss and accuracy curve of dense and sparse Hopfield models with different bag sizes. 
    The plotted are the mean of 10 runs.
    The results indicate that the sparse Hopfield model converges faster than the dense model and also yields superior 
    accuracy.
    }
    \label{fig:convg}
    \vspace{-1em}
\end{figure*}

\vspace{-0.5em}
\subsubsection{Real-World MIL Tasks} 
Next, we demonstrate that the proposed method achieves near-optimal performance on four realistic (\textit{non-sparse}) MIL benchmark datasets: Elephant, Fox and Tiger for image annotation \cite{ilse2018attention}, UCSB breast cancer classification \cite{kandemir2014empowering}. 
We use \texttt{Hopfield} and \texttt{SparseHopfield} to construct a similar model architecture proposed in \cite{ramsauer2020hopfield} and \blue{a detailed description of this experiment as well as its training and evaluating process can be found in Appendix \ref{appendix:mil_benchmark}.}
As shown in Table \ref{table:mil_benchmark}, both Sparse and Dense Hopfield achieve near-best results on Tiger, Elephant and UCSB datasets, despite the low sparsity in data. 
The sparse Hopfield model outperforms the dense Hopfield model by a small margin on three out of four datasets.
\begin{wraptable}{r}{0.6\textwidth}
\vspace{-0em}
    \centering
    \caption{Results for MIL benchmark datasets in terms of AUC score. The baselines are Path encoding \cite{kuccukacsci2018bag}, MInD \cite{cheplygina2015dissimilarity}, MILES \cite{chen2006miles}, APR \cite{dietterich1997solving}, Citation-KNN \cite{wang2000solving} and DD \cite{maron1997framework}. Results for baselines are taken from \cite{ramsauer2020hopfield}.
    The results suggest the proposed model achieves near-optimal performance even when the data is not sparse.  }
    \resizebox{0.6\textwidth}{!}{
    \begin{tabular}{l*{4}{c}}
    \toprule
    \toprule
            Method & \multicolumn{1}{c}{Tiger}
             & \multicolumn{1}{c}{Fox}
             & \multicolumn{1}{c}{Elephant}
             & \multicolumn{1}{c}{UCSB}\\ 
            \midrule
            Dense Hopfield
            & $0.878\pm0.028$
            & $0.600\pm0.011$
            & $0.907\pm0.022$
            & $ 0.880 \pm0.013$ \\
            Sparse Hopfield
            & $0.892\pm0.021$
            & $0.611\pm0.010$
            & $0.912\pm0.016$
            & $0.877\pm0.009$\\
            \midrule
            Path encoding 
            & $0.910\pm0.010$
            & $0.712\pm0.014$
            & $0.944\pm0.007$
            & $0.880\pm0.022$\\
            MInD 
            & $0.853\pm0.011$
            & $0.704\pm0.016$
            & $0.936\pm0.009$
            & $0.831\pm0.027$
            \\
            MILES 
            & $0.872\pm0.017$
            & $0.738\pm0.016$
            & $0.927\pm0.007$
            & $0.833\pm0.026$
            \\
            APR 
            & $0.778\pm0.007$
            & $0.541\pm0.009$
            & $0.550\pm0.010$
            & $\overline{\ \ \ }$
            \\
            Citation-kNN 
            & $0.855\pm0.009$
            & $0.635\pm0.015$
            & $0.896\pm0.009$
            & $0.706\pm0.032$
            \\
            DD 
            & $0.841$
            & $0.631$
            & $0.907$
            & $\overline{\ \ \ }$\\ 
            \bottomrule
            \bottomrule
    \end{tabular}
    }
    \label{table:mil_benchmark}
    \vspace{-1.2em}
\end{wraptable}

\vspace{-1.2em}
\section{Conclusion}
\label{sec:conclusion}
We present a sparse Hopfield model with a memory-retrieval dynamics that corresponds to the sparse-structured attention mechanism.
This model is capable of merging into deep learning architectures with data-dependent sparsity.
Theoretically, we introduce a principled construction for modern Hopfield models,
based on the convex conjugate of different entropy regularizers.
It allows us to easily recover the dense modern Hopfield model \cite{ramsauer2020hopfield} using Gibbs entropy.
Moreover, we introduce the sparse Hopfield model using the Gini entropic regularizer, and explore its theoretical advantages, delineating conditions that favor its use.
Empirically, we demonstrate our theoretical results and methodology to be effective on various synthetic and realistic settings.
This work extends the correspondence between artificial and biological neural networks to sparse domain, potentially paving the way for future Hopfield-based methodologies and bio-inspired computing systems.
\paragraph{Post-Acceptance Note [November 28, 2023].}
After the completion of this work, the authors learn of two upcoming works ---  \cite{anonymous2023stanhop} at ICLR'24 and \cite{martins2023sparse} in the Associative Memory \& Hopfield Networks Workshop at NeurIPS'23 ---  both addressing similar topics.
Both of these works explore  theoretical generalizations of this work.
In addition, 
\cite{anonymous2023stanhop} further presents a (sparse) Hopfield-based deep learning model for multivariate time series prediction.
We thank the authors of \cite{martins2023sparse} for enlightening discussions and for sharing their preliminary manuscript.

\section*{Acknowledgments}
JH would like to thank Tim Tsz-Kit Lau, Robin Luo and Andrew Chen for enlightening discussions on related topics, and Jiayi Wang for invaluable support in facilitating experimental deployments.
The authors would like to thank the anonymous reviewers and program chairs for constructive comments.

JH is partially supported by the Walter P. Murphy Fellowship.
HL is partially supported by NIH R01LM1372201, NSF CAREER1841569, DOE DE-AC02-07CH11359, DOE LAB 20-2261 and a NSF TRIPODS1740735.
This research was supported in part through the computational resources and staff contributions provided for the Quest high performance computing facility at Northwestern University which is jointly supported by the Office of the Provost, the Office for Research, and Northwestern University Information Technology.
The content is solely the responsibility of the authors and does not necessarily represent the official
views of the funding agencies.

\bibliographystyle{plainnat}
\bibliography{refs}

\begin{thebibliography}{68}
\providecommand{\natexlab}[1]{#1}
\providecommand{\url}[1]{\texttt{#1}}
\expandafter\ifx\csname urlstyle\endcsname\relax
  \providecommand{\doi}[1]{doi: #1}\else
  \providecommand{\doi}{doi: \begingroup \urlstyle{rm}\Url}\fi

\bibitem[Anonymous(2023)]{anonymous2023stanhop}
Anonymous.
\newblock {ST}anhop: Sparse tandem hopfield model for memory-enhanced time
  series prediction.
\newblock In \emph{Submitted to The Twelfth International Conference on
  Learning Representations}, 2023.
\newblock URL \url{https://openreview.net/forum?id=6iwg437CZs}.
\newblock under review.

\bibitem[Bai et~al.(2019)Bai, Kolter, and Koltun]{bai2019deep}
Shaojie Bai, J~Zico Kolter, and Vladlen Koltun.
\newblock Deep equilibrium models.
\newblock \emph{Advances in Neural Information Processing Systems}, 32, 2019.
\newblock URL \url{https://arxiv.org/abs/1909.01377}.

\bibitem[Beltagy et~al.(2020)Beltagy, Peters, and Cohan]{beltagy2020longformer}
Iz~Beltagy, Matthew~E Peters, and Arman Cohan.
\newblock Longformer: The long-document transformer.
\newblock \emph{arXiv preprint arXiv:2004.05150}, 2020.
\newblock URL \url{https://arxiv.org/abs/2004.05150}.

\bibitem[Bojar et~al.(2017)Bojar, Chatterjee, Federmann, Graham, Haddow, Huang,
  Huck, Koehn, Liu, Logacheva, et~al.]{wmt17}
Ond{\v{r}}ej Bojar, Rajen Chatterjee, Christian Federmann, Yvette Graham, Barry
  Haddow, Shujian Huang, Matthias Huck, Philipp Koehn, Qun Liu, Varvara
  Logacheva, et~al.
\newblock Findings of the 2017 conference on machine translation (wmt17).
\newblock Association for Computational Linguistics, 2017.

\bibitem[Brandstetter(2021)]{hopfeildblog2021}
Johannes Brandstetter.
\newblock Blog post: Hopfield networks is all you need, 2021.
\newblock URL \url{https://ml-jku.github.io/hopfield-layers/}.
\newblock Accessed: April 4, 2023.

\bibitem[Brauchart et~al.(2018)Brauchart, Reznikov, Saff, Sloan, Wang, and
  Womersley]{brauchart2018random}
Johann~S Brauchart, Alexander~B Reznikov, Edward~B Saff, Ian~H Sloan, Yu~Guang
  Wang, and Robert~S Womersley.
\newblock Random point sets on the sphere—hole radii, covering, and
  separation.
\newblock \emph{Experimental Mathematics}, 27\penalty0 (1):\penalty0 62--81,
  2018.
\newblock URL \url{https://arxiv.org/abs/1512.07470}.

\bibitem[Brown et~al.(2020)Brown, Mann, Ryder, Subbiah, Kaplan, Dhariwal,
  Neelakantan, Shyam, Sastry, Askell, et~al.]{brown2020language}
Tom Brown, Benjamin Mann, Nick Ryder, Melanie Subbiah, Jared~D Kaplan, Prafulla
  Dhariwal, Arvind Neelakantan, Pranav Shyam, Girish Sastry, Amanda Askell,
  et~al.
\newblock Language models are few-shot learners.
\newblock \emph{Advances in neural information processing systems},
  33:\penalty0 1877--1901, 2020.
\newblock URL \url{https://arxiv.org/abs/1512.07470}.

\bibitem[Cai and Jiang(2012)]{cai2012phase}
T~Tony Cai and Tiefeng Jiang.
\newblock Phase transition in limiting distributions of coherence of
  high-dimensional random matrices.
\newblock \emph{Journal of Multivariate Analysis}, 107:\penalty0 24--39, 2012.
\newblock URL \url{https://arxiv.org/abs/1102.2926}.

\bibitem[Carbonneau et~al.(2018)Carbonneau, Cheplygina, Granger, and
  Gagnon]{carbonneau2018multiple}
Marc-Andr{\'e} Carbonneau, Veronika Cheplygina, Eric Granger, and Ghyslain
  Gagnon.
\newblock Multiple instance learning: A survey of problem characteristics and
  applications.
\newblock \emph{Pattern Recognition}, 77:\penalty0 329--353, 2018.
\newblock URL \url{https://arxiv.org/abs/1612.03365}.

\bibitem[Chen et~al.(2021)Chen, Peng, Fu, and Ling]{chen2021autoformer}
Minghao Chen, Houwen Peng, Jianlong Fu, and Haibin Ling.
\newblock Autoformer: Searching transformers for visual recognition.
\newblock In \emph{Proceedings of the IEEE/CVF international conference on
  computer vision}, pages 12270--12280, 2021.
\newblock URL \url{https://arxiv.org/abs/2107.00651}.

\bibitem[Chen et~al.(2006)Chen, Bi, and Wang]{chen2006miles}
Yixin Chen, Jinbo Bi, and James~Ze Wang.
\newblock Miles: Multiple-instance learning via embedded instance selection.
\newblock \emph{IEEE transactions on pattern analysis and machine
  intelligence}, 28\penalty0 (12):\penalty0 1931--1947, 2006.

\bibitem[Cheplygina et~al.(2015)Cheplygina, Tax, and
  Loog]{cheplygina2015dissimilarity}
Veronika Cheplygina, David~MJ Tax, and Marco Loog.
\newblock Dissimilarity-based ensembles for multiple instance learning.
\newblock \emph{IEEE transactions on neural networks and learning systems},
  27\penalty0 (6):\penalty0 1379--1391, 2015.
\newblock URL \url{https://arxiv.org/abs/1402.1349}.

\bibitem[Child et~al.(2019)Child, Gray, Radford, and
  Sutskever]{child2019generating}
Rewon Child, Scott Gray, Alec Radford, and Ilya Sutskever.
\newblock Generating long sequences with sparse transformers.
\newblock \emph{arXiv preprint arXiv:1904.10509}, 2019.
\newblock URL \url{https://arxiv.org/abs/1904.10509}.

\bibitem[Chowdhery et~al.(2022)Chowdhery, Narang, Devlin, Bosma, Mishra,
  Roberts, Barham, Chung, Sutton, Gehrmann, et~al.]{chowdhery2022palm}
Aakanksha Chowdhery, Sharan Narang, Jacob Devlin, Maarten Bosma, Gaurav Mishra,
  Adam Roberts, Paul Barham, Hyung~Won Chung, Charles Sutton, Sebastian
  Gehrmann, et~al.
\newblock Palm: Scaling language modeling with pathways.
\newblock \emph{arXiv preprint arXiv:2204.02311}, 2022.
\newblock URL \url{https://arxiv.org/abs/2204.02311}.

\bibitem[Corless et~al.(1996)Corless, Gonnet, Hare, Jeffrey, and
  Knuth]{corless1996lambert}
Robert~M Corless, Gaston~H Gonnet, David~EG Hare, David~J Jeffrey, and Donald~E
  Knuth.
\newblock On the lambert w function.
\newblock \emph{Advances in Computational mathematics}, 5:\penalty0 329--359,
  1996.

\bibitem[Correia et~al.(2019)Correia, Niculae, and
  Martins]{correia2019adaptively}
Gon{\c{c}}alo~M Correia, Vlad Niculae, and Andr{\'e}~FT Martins.
\newblock Adaptively sparse transformers.
\newblock \emph{arXiv preprint arXiv:1909.00015}, 2019.
\newblock URL \url{https://arxiv.org/abs/1909.00015}.

\bibitem[David and Nagaraja(2004)]{david2004order}
Herbert~A David and Haikady~N Nagaraja.
\newblock \emph{Order statistics}.
\newblock John Wiley \& Sons, 2004.

\bibitem[Demircigil et~al.(2017)Demircigil, Heusel, L{\"o}we, Upgang, and
  Vermet]{demircigil2017model}
Mete Demircigil, Judith Heusel, Matthias L{\"o}we, Sven Upgang, and Franck
  Vermet.
\newblock On a model of associative memory with huge storage capacity.
\newblock \emph{Journal of Statistical Physics}, 168:\penalty0 288--299, 2017.
\newblock URL \url{https://arxiv.org/abs/1702.01929}.

\bibitem[Dietterich et~al.(1997)Dietterich, Lathrop, and
  Lozano-P{\'e}rez]{dietterich1997solving}
Thomas~G Dietterich, Richard~H Lathrop, and Tom{\'a}s Lozano-P{\'e}rez.
\newblock Solving the multiple instance problem with axis-parallel rectangles.
\newblock \emph{Artificial intelligence}, 89\penalty0 (1-2):\penalty0 31--71,
  1997.

\bibitem[Elad(2010)]{elad2010sparse}
Michael Elad.
\newblock \emph{Sparse and redundant representations: from theory to
  applications in signal and image processing}, volume~2.
\newblock Springer, 2010.

\bibitem[F{\"o}ldiak(1990)]{foldiak1990forming}
Peter F{\"o}ldiak.
\newblock Forming sparse representations by local anti-hebbian learning.
\newblock \emph{Biological cybernetics}, 64\penalty0 (2):\penalty0 165--170,
  1990.

\bibitem[F{\"u}rst et~al.(2022)F{\"u}rst, Rumetshofer, Lehner, Tran, Tang,
  Ramsauer, Kreil, Kopp, Klambauer, Bitto, et~al.]{furst2022cloob}
Andreas F{\"u}rst, Elisabeth Rumetshofer, Johannes Lehner, Viet~T Tran, Fei
  Tang, Hubert Ramsauer, David Kreil, Michael Kopp, G{\"u}nter Klambauer,
  Angela Bitto, et~al.
\newblock Cloob: Modern hopfield networks with infoloob outperform clip.
\newblock \emph{Advances in neural information processing systems},
  35:\penalty0 20450--20468, 2022.
\newblock URL \url{https://arxiv.org/abs/2110.11316}.

\bibitem[Gao and Pavel(2017)]{gao2017properties}
Bolin Gao and Lacra Pavel.
\newblock On the properties of the softmax function with application in game
  theory and reinforcement learning.
\newblock \emph{arXiv preprint arXiv:1704.00805}, 2017.
\newblock URL \url{https://arxiv.org/abs/1704.00805}.

\bibitem[Gunawardana et~al.(2005)Gunawardana, Byrne, and
  Jordan]{gunawardana2005convergence}
Asela Gunawardana, William Byrne, and Michael~I Jordan.
\newblock Convergence theorems for generalized alternating minimization
  procedures.
\newblock \emph{Journal of machine learning research}, 6\penalty0 (12), 2005.
\newblock URL
  \url{https://www.jmlr.org/papers/volume6/gunawardana05a/gunawardana05a.pdf}.

\bibitem[Hoover et~al.(2023)Hoover, Liang, Pham, Panda, Strobelt, Chau, Zaki,
  and Krotov]{hoover2023energy}
Benjamin Hoover, Yuchen Liang, Bao Pham, Rameswar Panda, Hendrik Strobelt,
  Duen~Horng Chau, Mohammed~J Zaki, and Dmitry Krotov.
\newblock Energy transformer.
\newblock \emph{arXiv preprint arXiv:2302.07253}, 2023.
\newblock URL \url{https://arxiv.org/abs/2302.07253}.

\bibitem[Hopfield(1982)]{hopfield1982neural}
John~J Hopfield.
\newblock Neural networks and physical systems with emergent collective
  computational abilities.
\newblock \emph{Proceedings of the national academy of sciences}, 79\penalty0
  (8):\penalty0 2554--2558, 1982.

\bibitem[Hopfield(1984)]{hopfield1984neurons}
John~J Hopfield.
\newblock Neurons with graded response have collective computational properties
  like those of two-state neurons.
\newblock \emph{Proceedings of the national academy of sciences}, 81\penalty0
  (10):\penalty0 3088--3092, 1984.

\bibitem[Ilse et~al.(2018)Ilse, Tomczak, and Welling]{ilse2018attention}
Maximilian Ilse, Jakub Tomczak, and Max Welling.
\newblock Attention-based deep multiple instance learning.
\newblock In \emph{International conference on machine learning}, pages
  2127--2136. PMLR, 2018.
\newblock URL \url{https://arxiv.org/abs/1802.04712}.

\bibitem[Ji et~al.(2021)Ji, Zhou, Liu, and Davuluri]{ji2021dnabert}
Yanrong Ji, Zhihan Zhou, Han Liu, and Ramana~V Davuluri.
\newblock Dnabert: pre-trained bidirectional encoder representations from
  transformers model for dna-language in genome.
\newblock \emph{Bioinformatics}, 37\penalty0 (15):\penalty0 2112--2120, 2021.
\newblock URL
  \url{https://www.biorxiv.org/content/10.1101/2020.09.17.301879v1}.

\bibitem[Kandemir et~al.(2014)Kandemir, Zhang, and
  Hamprecht]{kandemir2014empowering}
Melih Kandemir, Chong Zhang, and Fred~A Hamprecht.
\newblock Empowering multiple instance histopathology cancer diagnosis by cell
  graphs.
\newblock In \emph{Medical Image Computing and Computer-Assisted
  Intervention--MICCAI 2014: 17th International Conference, Boston, MA, USA,
  September 14-18, 2014, Proceedings, Part II 17}, pages 228--235. Springer,
  2014.

\bibitem[Kozachkov et~al.(2023)Kozachkov, Kastanenka, and
  Krotov]{kozachkov2022building}
Leo Kozachkov, Ksenia~V Kastanenka, and Dmitry Krotov.
\newblock Building transformers from neurons and astrocytes.
\newblock \emph{Proceedings of the National Academy of Sciences}, 120\penalty0
  (34):\penalty0 e2219150120, 2023.
\newblock URL
  \url{https://www.biorxiv.org/content/10.1101/2022.10.12.511910v1}.

\bibitem[Krotov and Hopfield(2016)]{krotov2016dense}
Dmitry Krotov and John~J Hopfield.
\newblock Dense associative memory for pattern recognition.
\newblock \emph{Advances in neural information processing systems}, 29, 2016.
\newblock URL \url{https://arxiv.org/abs/1606.01164}.

\bibitem[Krotov and Hopfield(2021)]{krotov2020large}
Dmitry Krotov and John~J. Hopfield.
\newblock Large associative memory problem in neurobiology and machine
  learning.
\newblock In \emph{International Conference on Learning Representations}, 2021.
\newblock URL \url{https://arxiv.org/abs/2008.06996}.

\bibitem[K{\"u}{\c{c}}{\"u}ka{\c{s}}c{\i} and
  Baydo{\u{g}}an(2018)]{kuccukacsci2018bag}
Emel~{\c{S}}eyma K{\"u}{\c{c}}{\"u}ka{\c{s}}c{\i} and Mustafa~G{\"o}k{\c{c}}e
  Baydo{\u{g}}an.
\newblock Bag encoding strategies in multiple instance learning problems.
\newblock \emph{Information Sciences}, 467:\penalty0 559--578, 2018.

\bibitem[Lee et~al.(1986)Lee, Doolen, Chen, Sun, Maxwell, and
  Lee]{lee1986machine}
YC~Lee, Gary Doolen, HH~Chen, GZ~Sun, Tom Maxwell, and HY~Lee.
\newblock Machine learning using a higher order correlation network.
\newblock Technical report, Los Alamos National Lab.(LANL), Los Alamos, NM
  (United States); Univ. of Maryland, College Park, MD (United States), 1986.

\bibitem[Liu et~al.(2021)Liu, Lin, Cao, Hu, Wei, Zhang, Lin, and
  Guo]{liu2021swin}
Ze~Liu, Yutong Lin, Yue Cao, Han Hu, Yixuan Wei, Zheng Zhang, Stephen Lin, and
  Baining Guo.
\newblock Swin transformer: Hierarchical vision transformer using shifted
  windows.
\newblock In \emph{Proceedings of the IEEE/CVF international conference on
  computer vision}, pages 10012--10022, 2021.
\newblock URL \url{https://arxiv.org/abs/2103.14030}.

\bibitem[Loshchilov and Hutter(2017)]{loshchilov2017decoupled}
Ilya Loshchilov and Frank Hutter.
\newblock Decoupled weight decay regularization.
\newblock \emph{arXiv preprint arXiv:1711.05101}, 2017.
\newblock URL \url{https://arxiv.org/abs/1711.05101}.

\bibitem[Mairal et~al.(2010)Mairal, Bach, Ponce, and Sapiro]{mairal2010online}
Julien Mairal, Francis Bach, Jean Ponce, and Guillermo Sapiro.
\newblock Online learning for matrix factorization and sparse coding.
\newblock \emph{Journal of Machine Learning Research}, 11\penalty0 (1), 2010.
\newblock URL \url{https://arxiv.org/abs/0908.0050}.

\bibitem[Makhzani and Frey(2015)]{makhzani2015winner}
Alireza Makhzani and Brendan~J Frey.
\newblock Winner-take-all autoencoders.
\newblock \emph{Advances in neural information processing systems}, 28, 2015.
\newblock URL \url{https://arxiv.org/abs/1409.2752}.

\bibitem[Maron and Lozano-P{\'e}rez(1997)]{maron1997framework}
Oded Maron and Tom{\'a}s Lozano-P{\'e}rez.
\newblock A framework for multiple-instance learning.
\newblock \emph{Advances in neural information processing systems}, 10, 1997.

\bibitem[Martins and Astudillo(2016)]{martins2016softmax}
Andre Martins and Ramon Astudillo.
\newblock From softmax to sparsemax: A sparse model of attention and
  multi-label classification.
\newblock In \emph{International conference on machine learning}, pages
  1614--1623. PMLR, 2016.
\newblock URL \url{https://arxiv.org/abs/1602.02068}.

\bibitem[Martins et~al.(2023)Martins, Niculae, and McNamee]{martins2023sparse}
Andre F.~T. Martins, Vlad Niculae, and Daniel McNamee.
\newblock Sparse modern hopfield networks.
\newblock \emph{Associative Memory \& Hopfield Networks in 2023. NeurIPS 2023
  workshop.}, 2023.
\newblock URL \url{https://openreview.net/pdf?id=zwqlV7HoaT}.

\bibitem[Millidge et~al.(2022)Millidge, Salvatori, Song, Lukasiewicz, and
  Bogacz]{millidge2022universal}
Beren Millidge, Tommaso Salvatori, Yuhang Song, Thomas Lukasiewicz, and Rafal
  Bogacz.
\newblock Universal hopfield networks: A general framework for single-shot
  associative memory models.
\newblock In \emph{International Conference on Machine Learning}, pages
  15561--15583. PMLR, 2022.
\newblock URL \url{https://arxiv.org/abs/2202.04557}.

\bibitem[Newman(1988)]{newman1988memory}
Charles~M Newman.
\newblock Memory capacity in neural network models: Rigorous lower bounds.
\newblock \emph{Neural Networks}, 1\penalty0 (3):\penalty0 223--238, 1988.

\bibitem[Olshausen and Field(1997)]{olshausen1997sparse}
Bruno~A Olshausen and David~J Field.
\newblock Sparse coding with an overcomplete basis set: A strategy employed by
  v1?
\newblock \emph{Vision research}, 37\penalty0 (23):\penalty0 3311--3325, 1997.

\bibitem[Olver et~al.(2010)Olver, Lozier, Boisvert, and Clark]{olver2010nist}
Frank~WJ Olver, Daniel~W Lozier, Ronald~F Boisvert, and Charles~W Clark.
\newblock \emph{NIST handbook of mathematical functions hardback and CD-ROM}.
\newblock Cambridge university press, 2010.

\bibitem[Paischer et~al.(2022)Paischer, Adler, Patil, Bitto-Nemling,
  Holzleitner, Lehner, Eghbal-Zadeh, and Hochreiter]{paischer2022history}
Fabian Paischer, Thomas Adler, Vihang Patil, Angela Bitto-Nemling, Markus
  Holzleitner, Sebastian Lehner, Hamid Eghbal-Zadeh, and Sepp Hochreiter.
\newblock History compression via language models in reinforcement learning.
\newblock In \emph{International Conference on Machine Learning}, pages
  17156--17185. PMLR, 2022.
\newblock URL \url{https://arxiv.org/abs/2205.12258}.

\bibitem[Palm(2013)]{palm2013neural}
G{\"u}nther Palm.
\newblock Neural associative memories and sparse coding.
\newblock \emph{Neural Networks}, 37:\penalty0 165--171, 2013.

\bibitem[Peretto and Niez(1986)]{peretto1986long}
Pierre Peretto and Jean-Jacques Niez.
\newblock Long term memory storage capacity of multiconnected neural networks.
\newblock \emph{Biological Cybernetics}, 54\penalty0 (1):\penalty0 53--63,
  1986.

\bibitem[Peters et~al.(2019)Peters, Niculae, and Martins]{peters2019sparse}
Ben Peters, Vlad Niculae, and Andr{\'e}~FT Martins.
\newblock Sparse sequence-to-sequence models.
\newblock \emph{arXiv preprint arXiv:1905.05702}, 2019.
\newblock URL \url{https://arxiv.org/abs/1905.05702}.

\bibitem[Qiu et~al.(2019)Qiu, Ma, Levy, Yih, Wang, and Tang]{qiu2019blockwise}
Jiezhong Qiu, Hao Ma, Omer Levy, Scott Wen-tau Yih, Sinong Wang, and Jie Tang.
\newblock Blockwise self-attention for long document understanding.
\newblock \emph{arXiv preprint arXiv:1911.02972}, 2019.
\newblock URL \url{https://arxiv.org/abs/1911.02972}.

\bibitem[Ramsauer et~al.(2021)Ramsauer, Sch{\"a}fl, Lehner, Seidl, Widrich,
  Gruber, Holzleitner, Adler, Kreil, Kopp, Klambauer, Brandstetter, and
  Hochreiter]{ramsauer2020hopfield}
Hubert Ramsauer, Bernhard Sch{\"a}fl, Johannes Lehner, Philipp Seidl, Michael
  Widrich, Lukas Gruber, Markus Holzleitner, Thomas Adler, David Kreil,
  Michael~K Kopp, G{\"u}nter Klambauer, Johannes Brandstetter, and Sepp
  Hochreiter.
\newblock Hopfield networks is all you need.
\newblock In \emph{International Conference on Learning Representations}, 2021.
\newblock URL \url{https://arxiv.org/abs/2008.02217}.

\bibitem[Rubinstein et~al.(2010)Rubinstein, Bruckstein, and
  Elad]{rubinstein2010dictionaries}
Ron Rubinstein, Alfred~M Bruckstein, and Michael Elad.
\newblock Dictionaries for sparse representation modeling.
\newblock \emph{Proceedings of the IEEE}, 98\penalty0 (6):\penalty0 1045--1057,
  2010.

\bibitem[Seidl et~al.(2022)Seidl, Renz, Dyubankova, Neves, Verhoeven, Wegner,
  Segler, Hochreiter, and Klambauer]{seidl2022improving}
Philipp Seidl, Philipp Renz, Natalia Dyubankova, Paulo Neves, Jonas Verhoeven,
  Jorg~K Wegner, Marwin Segler, Sepp Hochreiter, and Gunter Klambauer.
\newblock Improving few-and zero-shot reaction template prediction using modern
  hopfield networks.
\newblock \emph{Journal of chemical information and modeling}, 62\penalty0
  (9):\penalty0 2111--2120, 2022.

\bibitem[Sriperumbudur and Lanckriet(2009)]{sriperumbudur2009convergence}
Bharath~K Sriperumbudur and Gert~RG Lanckriet.
\newblock On the convergence of the concave-convex procedure.
\newblock In \emph{Advances in neural information processing systems},
  volume~9, pages 1759--1767, 2009.
\newblock URL
  \url{https://papers.nips.cc/paper_files/paper/2009/file/8b5040a8a5baf3e0e67386c2e3a9b903-Paper.pdf}.

\bibitem[Tay et~al.(2020)Tay, Bahri, Yang, Metzler, and Juan]{tay2020sparse}
Yi~Tay, Dara Bahri, Liu Yang, Donald Metzler, and Da-Cheng Juan.
\newblock Sparse sinkhorn attention.
\newblock In \emph{International Conference on Machine Learning}, pages
  9438--9447. PMLR, 2020.

\bibitem[Tay et~al.(2022)Tay, Dehghani, Bahri, and Metzler]{tay2022efficient}
Yi~Tay, Mostafa Dehghani, Dara Bahri, and Donald Metzler.
\newblock Efficient transformers: A survey.
\newblock \emph{ACM Computing Surveys}, 55\penalty0 (6):\penalty0 1--28, 2022.
\newblock URL \url{https://arxiv.org/abs/2009.06732}.

\bibitem[Vaswani et~al.(2017)Vaswani, Shazeer, Parmar, Uszkoreit, Jones, Gomez,
  Kaiser, and Polosukhin]{vaswani2017attention}
Ashish Vaswani, Noam Shazeer, Niki Parmar, Jakob Uszkoreit, Llion Jones,
  Aidan~N Gomez, {\L}ukasz Kaiser, and Illia Polosukhin.
\newblock Attention is all you need.
\newblock \emph{Advances in neural information processing systems}, 30, 2017.
\newblock URL \url{https://arxiv.org/abs/1706.03762}.

\bibitem[Wainwright et~al.(2008)Wainwright, Jordan,
  et~al.]{wainwright2008graphical}
Martin~J Wainwright, Michael~I Jordan, et~al.
\newblock Graphical models, exponential families, and variational inference.
\newblock \emph{Foundations and Trends{\textregistered} in Machine Learning},
  1\penalty0 (1--2):\penalty0 1--305, 2008.

\bibitem[Wang and Zucker(2000)]{wang2000solving}
Jun Wang and Jean-Daniel Zucker.
\newblock Solving multiple-instance problem: A lazy learning approach.
\newblock 2000.

\bibitem[Widrich et~al.(2020)Widrich, Sch{\"a}fl, Pavlovi{\'c}, Ramsauer,
  Gruber, Holzleitner, Brandstetter, Sandve, Greiff, Hochreiter,
  et~al.]{widrich2020modern}
Michael Widrich, Bernhard Sch{\"a}fl, Milena Pavlovi{\'c}, Hubert Ramsauer,
  Lukas Gruber, Markus Holzleitner, Johannes Brandstetter, Geir~Kjetil Sandve,
  Victor Greiff, Sepp Hochreiter, et~al.
\newblock Modern hopfield networks and attention for immune repertoire
  classification.
\newblock \emph{Advances in Neural Information Processing Systems},
  33:\penalty0 18832--18845, 2020.
\newblock URL \url{https://arxiv.org/abs/2007.13505}.

\bibitem[Yang et~al.(2022)Yang, Huang, and Wipf]{yang2022transformers}
Yongyi Yang, Zengfeng Huang, and David Wipf.
\newblock Transformers from an optimization perspective.
\newblock In Alice~H. Oh, Alekh Agarwal, Danielle Belgrave, and Kyunghyun Cho,
  editors, \emph{Advances in Neural Information Processing Systems}, 2022.
\newblock URL \url{https://openreview.net/forum?id=VT0Y4PlV2m0}.

\bibitem[Yuille and Rangarajan(2001)]{yuille2001concave}
Alan~L Yuille and Anand Rangarajan.
\newblock The concave-convex procedure (cccp).
\newblock \emph{Advances in neural information processing systems}, 14, 2001.
\newblock URL
  \url{https://proceedings.neurips.cc/paper_files/paper/2001/file/a012869311d64a44b5a0d567cd20de04-Paper.pdf}.

\bibitem[Yuille and Rangarajan(2003)]{yuille2003concave}
Alan~L Yuille and Anand Rangarajan.
\newblock The concave-convex procedure.
\newblock \emph{Neural computation}, 15\penalty0 (4):\penalty0 915--936, 2003.

\bibitem[Zangwill(1969)]{zangwill1969nonlinear}
Willard~I Zangwill.
\newblock \emph{Nonlinear programming: a unified approach}, volume~52.
\newblock Prentice-Hall Englewood Cliffs, NJ, 1969.

\bibitem[Zhang and Yan(2022)]{zhang2022crossformer}
Yunhao Zhang and Junchi Yan.
\newblock Crossformer: Transformer utilizing cross-dimension dependency for
  multivariate time series forecasting.
\newblock In \emph{The Eleventh International Conference on Learning
  Representations}, 2022.
\newblock URL \url{https://openreview.net/forum?id=vSVLM2j9eie}.

\bibitem[Zhou et~al.(2021)Zhou, Zhang, Peng, Zhang, Li, Xiong, and
  Zhang]{zhou2021informer}
Haoyi Zhou, Shanghang Zhang, Jieqi Peng, Shuai Zhang, Jianxin Li, Hui Xiong,
  and Wancai Zhang.
\newblock Informer: Beyond efficient transformer for long sequence time-series
  forecasting.
\newblock In \emph{Proceedings of the AAAI conference on artificial
  intelligence}, volume~35, pages 11106--11115, 2021.
\newblock URL \url{https://arxiv.org/abs/2012.07436}.

\bibitem[Zhou et~al.(2022)Zhou, Ma, Wen, Sun, Yao, Yin, Jin,
  et~al.]{zhou2022film}
Tian Zhou, Ziqing Ma, Qingsong Wen, Liang Sun, Tao Yao, Wotao Yin, Rong Jin,
  et~al.
\newblock Film: Frequency improved legendre memory model for long-term time
  series forecasting.
\newblock \emph{Advances in Neural Information Processing Systems},
  35:\penalty0 12677--12690, 2022.
\newblock URL \url{https://arxiv.org/abs/2205.08897}.

\end{thebibliography}

\newpage  %

\titlespacing*{\section}{0pt}{*1}{*1}
\titlespacing*{\subsection}{0pt}{*1.25}{*1.25}
\titlespacing*{\subsubsection}{0pt}{*1.5}{*1.5}

\setlength{\abovedisplayskip}{10pt}
\setlength{\abovedisplayshortskip}{10pt}
\setlength{\belowdisplayskip}{10pt}
\setlength{\belowdisplayshortskip}{10pt}

\normalsize
\appendix	
\label{sec:append}
\part*{Appendix}

{
\setlength{\parskip}{-0em}
\startcontents[sections]
\printcontents[sections]{ }{1}{}
}

\clearpage
\section{Nomenclature Table}
We summarize our notations in the following table for easy reference.
   
\begin{table}[h]
    \caption{Mathematical Notations and Symbols}
    \centering
    \resizebox{ \textwidth}{!}{ 
    \begin{tabular}{c l}
    \toprule
        Symbol & Description \\
    \midrule
        $\Braket{\ba,\bb}$ & Inner product for vectors $\ba,\bb\in \R^d$ \\
        $[I]$ & Index set $\{1,\cdots,I\}$, where $I\in\mathbb{N}^+$ \\
        $\norm{\cdot}_2$ & Spectral norm, equivalent to the $l_2$-norm when applied to a vector \\
        \midrule
        $d$ & Dimension of patterns
        \\
        $M$ & Number of stored memory patterns
        \\
        $\beta$ &
        A scaling factor of the energy function that controls the learning dynamics
        \\
        \midrule
        $\bx$ & State/configuration/query pattern in $\R^d$ \\
        $\bxi$ & Memory patterns (keys) in $\R^d$ \\
        $\bm{\Xi}$ & Shorthand for $M$ stored memory (key) patterns $\{\bxi_\mu\}_{\mu\in[M]}$ in $\R^{d\times M}$ \\
        $\bm{\Xi}^\sT \bx$ & $M$-dimensional overlap vector  $\(\Braket{\bxi_1,\bx},\cdots,\Braket{\bxi_\mu,\bx},\cdots,\Braket{\bxi_M,\bx}\)$ in $\R^{M}$ \\
        $\[\bm{\Xi}^\sT \bx\]_\kappa$ & The $\kappa$-th element of  $\bm{\Xi}^\sT \bx$ \\
        \midrule
        $n$ & Norm of $\bx$, denoted as $n\coloneqq\norm{\bx}$ \\
        $m$ & Largest norm of memory patterns, denoted as $m\coloneqq \Max_{\mu\in[M]}\norm{\bxi_\mu}$ 
        \\
        $\kappa$ & The number of non-zero element of $\Sparsemax$,
        defined in \eqref{eqn:sparsemax_exact}
        \\
        \midrule
        $R$ &
        The minimal Euclidean distance across all possible pairs of memory patterns, $R\coloneqq \half \Min_{\mu,\nu\in[M]}\norm{\bxi_\mu-\bxi_\nu}$
        \\
        $S_\mu$ &
        The sphere centered at the memory pattern $\bxi_\mu$ with finite radius $R$
        \\
        $\bx^\star_\mu$ & The fixed point of $\calT$ covered by $S_\mu$, i.e. $\bx_\mu^\star \in S_\mu$
        \\
        $\Delta_\mu$ & The separation of a memory pattern $\bxi_\mu$ from all other memory patterns $\bm{\Xi}$, defined in \eqref{eqn:sep_delta_mu}
        \\
        $\Tilde{\Delta}_\mu$&
        The separation of $\bxi_\mu$ at a given $\bx$ from all memory patterns $\bm{\Xi}$, defined in \eqref{eqn:tilde_sep_delta_mu}
        \\
    \bottomrule
    \end{tabular}
    }
    \label{tab:nomenclature}
\end{table}

\section{Broader Impacts and Future Directions: Brain Science and Foundation Models}
The primary theme of our research is to perceive any data representation (set of patterns) as analogous to the neural responses of a global brain reacting to a vast range of external stimuli (queries). This perspective presents exciting opportunities to study large generative foundational models, such as large language models, within a rigorous scientific framework inspired by contemporary brain science research. 

We believe this work could be impactful in several respects, even though it is  foundational research and not tied to specific applications:
\textbf{(Cognition.)} This research could contribute to our understanding of a memory-enhanced model's predictive capacity when given either in-context input (like historical data) or external stimuli (such as real-time events).
\textbf{(Memory.)} It may also shed light on the inherent limits of artificial neural networks' memorization capabilities and how to augment them with external memory modules for rapid responses to potential external stimuli.
\textbf{(Network.)} This research could enable models to better assess the intricate network of cross-sectional brain activity among different variables and infer its dynamic structural alterations to identify possible systematic properties.

\section{Related Works and Limitations}
\label{sec:related_works}
\paragraph{Sparse Hopfield Models.}
Our work is closely related to and motivated by \cite{foldiak1990forming}, which proposes a local anti-Hebbian learning rule for sparse representations in associative memory networks. This rule enhances storage capacity and retrieval capabilities but has limitations: (i) fixed sparsity based on local similarity of receptive fields, (ii) difficulty in scaling up and integration with modern DNNs \cite{makhzani2015winner}, (iii) lack of a solid theoretical foundation for convergence and stability, and (iv) inherently unsupervised retrieval dynamics, limiting its applicability for supervised learning or other paradigms like reinforcement learning or semi-supervised learning.
On the other hand, another line of related work, not specifically focusing on sparsifying Hopfield models, centers on sparse coding \cite{palm2013neural,olshausen1997sparse}, introducing sparsity to associative memory models through thresholding memory patterns. 
These studies offer insights into the relationship between the sparseness of the stored memory patterns and the robustness of the network but sufferers from the issues related to scalability, sparsity level bias, and noise vulnerability \cite{mairal2010online,rubinstein2010dictionaries,elad2010sparse,olshausen1997sparse}.
In contrast, our approach is theoretically grounded and has data-dependent sparsity leading to better scalability, more meaningful and robust representations of patterns and allows the model to focus on the most relevant information for each specific instance.

\paragraph{Hopfield Models and Connection to Attention.}
Hopfield Models \cite{hopfield1984neurons,hopfield1982neural,krotov2016dense} have seen renewed interest in the machine learning community due to advances in memory storage capacity understanding \cite{krotov2016dense,demircigil2017model}, architectural innovations \cite{hoover2023energy,seidl2022improving,furst2022cloob,ramsauer2020hopfield},
and biological plausibility \cite{kozachkov2022building,krotov2020large}.
Notably, Modern Hopfield Networks \cite{ramsauer2020hopfield}\footnote{Also see the well-written blog post \cite{hopfeildblog2021}.}, a new subclass, highlight the equivalence\footnote{While this equivalence only holds when the retrieval dynamics is applied exactly once, as originally shown in \cite{ramsauer2020hopfield} and later emphasized in \cite{krotov2020large}, it allows us to view modern Hopfield models as generalized attentions with additional functionalities and hence opens new avenues for Hopfield-based  architecture designs.} between their memory retrieval dynamics and attention mechanisms in transformers.
With this hindsight, it becomes clear that
transformers and modern Hopfield models share some high-level similarities, as well as differences.
Both architectures are designed for denoising input, with transformers typically pre-trained on masked-token tasks and the modern Hopfield model aimed at completing incomplete or contaminated patterns. 
However, the modern Hopfield models are recurrent networks with a global energy function that ensures convergence to a fixed-point attractor, while transformers are generally viewed as feed-forward networks without such dynamics.
It is natural to ask whether such equivalence is fundamental.
Although, apart from Hopfield-side investigations \cite{hoover2023energy,krotov2020large,ramsauer2020hopfield}, there have been studies viewing transformers as dynamical systems, including the deep equilibrium models \cite{bai2019deep}, and unfolded optimization \cite{yang2022transformers}, none exhibit similar converge-to-memory dynamics as in Hopfield models (hence missing the connection between dynamical memory retrieval and transformers), nor do they address sparsity.
Building on the established equivalence in \cite{ramsauer2020hopfield}, our work serves as an initial attempt to push such equivalence toward sparse models, both theoretically and empirically.
It lays the groundwork for future Hopfield-based methodologies, architecture designs and biological computers (as in \cite{kozachkov2022building}).

\paragraph{Sparse Attention.}
Attention-based seq2seq models excel in various applications like large language models \cite{chowdhery2022palm,brown2020language}, time series prediction \cite{zhou2022film,zhou2021informer}, and biomedical science \cite{ji2021dnabert}, primarily due to their versatility in framing tasks as source-to-target sequence transformations with potentially varying lengths.
However, the original transformer architecture utilizes a dense, quadratic attention score matrix, which can be computationally demanding (with $\calO(n^2)$ complexity for input sequence length $n$), memory-intensive, and challenging to interpret for long sequences. 
To combat these issues, there is a large amount of literature works leverages various
sparsifying methods for attention and transformers to enhance computational efficiency while preserving the models' expressiveness, see \cite{tay2022efficient} for an overview.
Here, we classify sparse Transformers into two distinct categories based on the different kinds of sparsities.
The first category focuses on structured-sparsity \cite{beltagy2020longformer, qiu2019blockwise, child2019generating}, which involves creating a sparse attention score matrix in a pre-determined manner. 
In these approaches, each token in the sequence attends to a fixed subset of other tokens, rather than the entire sequence. 
The second category obtains sparsity through the sparsity-inducing normalization maps \cite{peters2019sparse,correia2019adaptively,krotov2016dense} that encourage the models to focus on a subset of relevant input elements, thereby fostering sparsity, scalability and interpretability.
Compared to the first category, while these approaches still have $\calO(n^2)$ space complexity, they offer the advantage of producing sparsity patterns that are more adaptive to the data.
Our work is closely related to the second and
utilizes \textit{sparsity-inducing alternatives} to the softmax function in modern Hopfield models.

\subsection{Limitations}

Since our model aligns with sparsemax attention, it also grapples with $\calO(d^2)$ complexity, a characteristic typical of the sparsity-inducing normalization map category of sparse attention.
In addition, we opt not to impose any assumptions on the data (patterns) to maintain the general applicability of our model. 
This decision, however, prevents us from providing a rigorous characterization of how data-dependent sparsity explicitly impacts retrieval error, the well-separation condition, and memory capacity. 
Specifically, a detailed analysis of $\[\bm{\Xi}^\sT\bx\]_{(\kappa)}$ is a problem of order statistics \cite{david2004order} that hinges on the distribution of patterns.
Instead, we offer qualitative discussions in Section \ref{sec:method} to provide insights into the behavior of the sparse model under various conditions, aiding in a better understanding and application of the model.

\section{Modern Hopfield Model and Its Connection to Attention Mechanism}
\label{sec:hopfield}

\label{sec:T_is_attn}
\blue{
  \citet{ramsauer2020hopfield} generalize the  exponential-interaction-based energy function proposed in \cite{demircigil2017model} to continuous patterns and states with a strong link the attention mechanism.  
 In this section, we provide an overview for both of them and then draw the connection of the modern Hopfield model to attention mechanism.
 }

\begin{figure*}[h]
    \centering
    \includegraphics[scale=0.2]{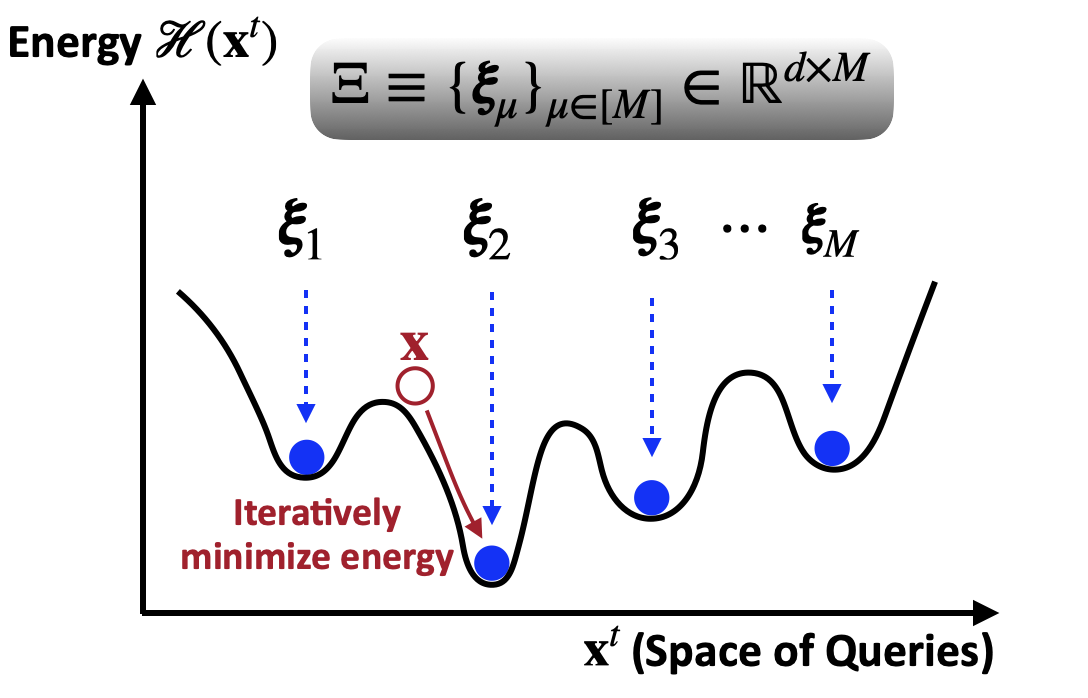}
    \caption{\textbf{Visualizing Hopfield Models.} Let $\bx\in\R^d$ represent the query pattern, and let $\bm{\Xi}\coloneqq\(\bxi_1,\cdots,\bxi_M\)\in \R^{d\times M}$ denote the memory patterns. 
The objective of the Hopfield models is to store the memory patterns $\bm{\Xi}$ and then retrieve a specific memory pattern $\bxi_\mu$ based on a given query $\bx$.
They achieve these by embedding the memories $\bm{\Xi}$ in the energy landscape $\calH(\bx)$ of a physical system (e.g., the Ising model in \cite{hopfield1982neural} and its higher-order generalizations \cite{lee1986machine,peretto1986long,newman1988memory}), where each memory $\bxi_\mu$ corresponds to a local minimum. 
When a query $\bx$ is presented, the model initiates energy-minimizing retrieval dynamics $\calT$ at the query $\bx$, which then navigate the energy landscape to find the nearest local minimum, effectively retrieving the memory most similar to the query.
}
    \label{fig:energylandscape}
\end{figure*}
\subsection{Modern Hopfield Model}
We first introduce the log-sum-exponential ($\lse$) function for any given vector $\bz=(z_1,\cdots,z_M)$ and $\beta>0$:
\bea
\lse\(\beta,\bz\)\coloneqq \frac{1}{\beta}\log\(\sumM \exp{\beta z_\mu}\),
\eea
which is an important representation of the softmax function which can be derived by considering the ``argmax function'' under entropy regularization, see \cite{gao2017properties} and references therein.

\paragraph{Exponential Binary Hopfield Model.}
With $\lse(\cdot)$, the exponential-Hopfield model for binary patterns $\bxi,\bx\in\{\pm 1\}^d$ proposed in \cite{demircigil2017model} can be written as, denoting $\bm{\Xi}\coloneqq\(\bxi^1,\cdots,\bxi^M\)$,
\bea
\calH(\bx)= -\sumM\exp{\Braket{\bxi^\mu,\bm{\sigma}}}=-\exp{\lse\(1,\bm{\Xi}^\sT \bm{\sigma}\)},
\eea
which leads to the super-linear memory capacity of $M \propto 2^{d/2}$.

\paragraph{Modern Hopfield Model.}
For continuous patterns
 $\bx,\{\bxi^\mu\}\in \R^d$, \citet{ramsauer2020hopfield} propose the continuous\footnote{\blue{Note that, there are also many continuous Hopfield models prior than \cite{ramsauer2020hopfield}, including \cite{krotov2016dense,hopfield1984neurons}.}} modern Hopfield model 
\begin{align}
\label{eqn:modern_hopfield}
    \calH(\bx)\coloneqq-\lse\(1,\bm{\Xi}^\sT \bx\)
    +\half \Braket{\bx,\bx}+ \frac{1}{\beta}\log M + \half m^2,
\end{align}
with retrieval dynamics
\bea
\label{eqn:modern_retrieval_dyn}
\bx^{\text{new}}= \calT_{\text{Dense}}(\bx)=\bm{\Xi}\cdot \Softmax\(\beta \bm{\Xi}^\sT \bx\),
\eea
where $\half \Braket{\bx,\bx}$ is a regularizer introduced for ensuring configuration vector $\bx$ being finite, and $m\coloneqq \max_\mu\norm{\bxi^\mu}$ is the largest norm of memory patterns.
Moreover, they show that 
(i) the modern Hopfield \eqref{eqn:modern_hopfield} has an exponential memory capacity in $d$, 
(ii) the retrieval dynamics \eqref{eqn:modern_retrieval_dyn} can consistently retrieve patterns with high accuracy with only one step, and 
(iii) surprisingly, the retrieval dynamics \eqref{eqn:modern_retrieval_dyn} is connected to the attention mechanism in transformer giving rise to a new methodology --- the Hopfield DNN layer.

\subsection{Memory Retrieval Dynamics $\calT_{\text{Dense}}\leftrightarrow$ Self-Attention Mechanism}

Following \cite{ramsauer2020hopfield,hopfeildblog2021}, we say $\bX$ and $\bm{\Xi}$ are in the associative space (embedded space), as they are mapped from the \textit{raw} query ${\bR}$ and $\bY$ memory patterns, respectively, via
\begin{align}
\bX^\sT&={\bR}\bW_Q\coloneqq \bQ,\\
\bm{\Xi}^\sT&=\bY \bW_K \coloneqq \bK,
\end{align}
with some $\bW_Q$ and $\bW_K$.
Therefore, we can express $\calT_{\text{Dense}}$ as
\begin{align}
\(\bQ^{\text{new}}\)^\sT = \bK^\sT \Softmax\(\beta \bK \bQ^\sT\).
\end{align}
Taking transpose to above, we have
\bea
\bQ^{\text{new}}= \Softmax\(\beta  \bQ\bK^\sT\)\bK.
\eea
Projecting $\bK$ to $\bV$ with $\bW_V$, we have
\begin{align}
    \bZ\coloneqq \bQ^{\text{new}} \bW_V&=\Softmax\(\beta  \bQ\bK^\sT\)\bK \bW_V
    \\
    &=\Softmax\(\beta  \bQ\bK^\sT\)\bV,
\end{align}
which leads to the self-attention mechanism.

Plugging back the raw patterns $\bR$ and $\bY$, 
we arrive the foundation of the Hopfield layer,
\bea
\bZ=\Softmax\(\beta  {\bR}\bW_Q \bW_K^\sT\bY^\sT\)\bY\bW_K \bW_V.
\eea

The same construction applies to the sparse retrieval dynamics \eqref{def:sparse_retrieval_dyn}, 
\bea
\bZ'=\Sparsemax\(\beta  {\bR}\bW_Q' \bW_K'^\sT\bY^\sT\)\bY\bW_K' \bW_V'.
\label{eqn:sparse_hopfield_attention}
\eea

resulting in a sparse Hopfield layer that can be seamlessly integrated into deep learning architectures.

{
\color{Blue}
\subsection{Algorithm of Multi-Step $\mathtt{SparseHopfield}$ Layer}
Here, we present an algorithm for implementing the $\mathtt{SparseHopfield}$ layer with multi-step updates (i.e. multiple iterative retrievals).  
The algorithm, summarized in \cref{alg:multistep} below, outlines the process for $U$ update steps.
Similar to \cite{ramsauer2020hopfield}, the $\mathtt{SparseHopfield}$  takes as input the matrices $\bR, \bY$, and the weight matrices $\bW_Q', \bW_K', \bW_V'$.

\begin{algorithm}
\caption{Multi-Step $\mathtt{SparseHopfield}$ Layer}\label{alg:multistep}
\begin{algorithmic}
\Require $U \in \mathbb{R} \geq 1, \bR, \bY.$ 
\State $\bQ \gets \bR \bW_Q'$
\For{$i \rightarrow 1 \text{ to } U $}
\State $\bQ^{\text{new}} \gets \text{Sparsemax}\( \beta \bQ \bW_K'^T \bY^T \) \bY \bW_V' \bW_K' $ \hfill{\textit{Hopfield Update} as \ref{eqn:sparse_hopfield_attention}}
\State $\bQ \gets \bQ^{\text{new}}$
\EndFor \\
\Return $\bQ$
\end{algorithmic}
\end{algorithm}
}

{
\color{Blue}
Here we explain the usage of the above algorithm w.r.t. different settings.
\begin{enumerate}
    \item \textbf{Memory Retrieval.}
The memory retrieval is a learning-free setting.
Thus, we can exclude the use of weight matrices $\bW_K, \bW_Q, \bW_V$ (by setting them to identity matrices).
And let the input (corrupted image) to be our $\bR$, stored patterns as $\bY$ for retrieval.

\item \textbf{$\mathtt{SparseHopfield}$.}
The $\mathtt{SparseHopfield}$ has two inputs, $\bR, \bY$.
Since the $\mathtt{SparseHopfield}$ can be used to replace attention mechanism in models, we make the weight matrices $\bW_K', \bW_Q', \bW_V'$ learnable, and $\bR, \bY, \bY$ be the source of query, key, value, respectively.
Note that the self-attention-liked mechanism can be realized by setting $\bR = \bY$.

\item \textbf{$\mathtt{SparseHopfieldPooling}$.}
The $\mathtt{SparseHopfieldPooling}$ layer has one input, $\bY$, where $\bQ$ is the learnable \textbf{prototype pattern} and fixed during inference, and $\bY$ is the stored patterns we want to perform pooling over.
Note that the $\bQ$ here is independent from the input and can be seen as part of the learnable parameter of the $\mathtt{SparseHopfieldPooling}$ layer.
Here since we replace the query pattern ($\bR \bW_Q'$) with a static \textbf{prototype pattern} $\bQ$, the learnable weight matrices here will only be $\bW_K', \bW_V'$.

\item 
\textbf{$\mathtt{SparseHopfieldLayer}$.}
The $\mathtt{SparseHopfieldLayer}$ layer has one input, $\bR$.
Where $\bR$ is the query pattern.
And we have learnable weight matrices $\bW_K', \bW_V'$ served as our stored patterns and pattern projections, leading our key and value independent to the input.
In other words, following the notation in \cref{alg:multistep}, $\bY$ can be seen as an identity matrix.

\end{enumerate}
}
\clearpage
\section{Proofs of Main Text}
\label{sec:proofs}

\subsection{\cref{thm:main_result}}
\label{proof:main_result}
\begin{proof}[Proof of Theorem~\ref{thm:main_result}]
\begin{align}
    \Max_{\bp\in\Delta^d}\[\Braket{\bp,\bz} -\half\norm{\bp}^2+\half\]
    &=
    \Max_{\bp\in\Delta^d}\[\half\norm{\bz}^2+\Braket{\bp,\bz} -\half\norm{\bp}^2-\half\norm{\bz}^2+\half\]\\
    &=
    \Max_{\bp\in\Delta^d}\[\half\norm{\bz}^2+\half
    -\half\norm{\bp-\bz}^2
    \]\\
    &= 
    \half\norm{\bz}^2+\half
    -\Min_{\bp\in\Delta^d}\[\half\norm{\bp-\bz}^2
    \]\\
    &=
    \half\norm{\bz}^2
    -\half\norm{\bp^\star-\bz}^2+\half=\Psi^\star(\bz),
\end{align}
with $\bp^\star$ given by \eqref{eqn:sparsemax_exact}.
\end{proof}

\subsection{\cref{thm:retrieval_dyn}}
\label{sec:thm_retreival}
\begin{proof}[Proof of \cref{thm:retrieval_dyn}]

To show monotonic decreasing property of the energy \eqref{eqn:H_sparsemax}, we first derive the sparse retrieval dynamics by utilizing the aforementioned \cref{thm:main_result}, \cref{corollary:Dankins}, along with the convex-concave procedure \cite{yuille2003concave,yuille2001concave}.
Then, we show the monotonicity of $\calH$ by constructing an iterative upper bound of $\calH$ which is convex in $\bx_{t+1}$ and thus, lowered iteratively by the CCCP method.

By convex conjugate, $\Psi^*$, the conjugate convex of $\Psi$, is always convex, and hence $-\Psi^*$ is a concave function. 
Therefore, the energy function $\calH$ is by construction the sum of the convex function $\calH_1(\bx)\coloneqq\half\Braket{\bx,\bx}$ and the concave function $\calH_2(\bx)\coloneqq-\Psi^\star\(\bm{\Xi}^\sT\bx\)$.
In addition, $\calH$ is differentiable by definition.

Applying the convex-concave procedure to $\calH$ gives
\begin{align}
\grad_\bx\calH_1\(\bx_{t+1}\)=-\grad_{\bx}\calH_2\(\bx_t\),
\end{align}
which leads to
\bea
\bx_{t+1}=\grad_{\bx}\Psi\(\bm{\Xi}\bx_t\)= \bm{\Xi}\Sparsemax\(\bm{\Xi}^\sT \bx_t\),
\eea
by \cref{thm:main_result} and \cref{corollary:Dankins}.

Following  \cite{yuille2003concave,yuille2001concave}, we show the monotonic decreasing of \eqref{eqn:H_sparsemax} over $t$ with
by considering the problem of energy minimization:
\bea
\label{eqn:energy_minimizaiton}
\Min_\bx \[\calH(\bx)\]&=&\Min_\bx \[\calH_1(\bx)+\calH_2(\bx)\],
\eea
which, in the convex-concave procedure, is solved by iteratively computing
{
\bea
\label{eqn:iterative_argmin}
\bx_{t+1} &\in& \argmin_{\bx} \[\calH_1(\bx)+
\Braket{\bx,\grad_\bx \calH_2\(\bx_t\)}
\],
\eea
for all $t$.
}
{
The intuition behind this is 
to linearize the concave $\calH_2$ around the current iteration's solution $\bx_t$, making $\calH_1(\bx_{t+1})+\Braket{\bx_{t+1},\grad_\bx \calH_2(\bx_t)}$ convex in $\bx_{t+1}$.

By convexity and concavity of $\calH_1$ and $\calH_2$, we have
\begin{align}
\calH_1(\bx)&\ge \calH_1(\by)+\Braket{\(\bx-\by\),\grad_{\bx}\calH_1(\by)},\\
\calH_2(\bx)&\le \calH_2(\by)+\Braket{\(\bx-\by\),\grad_{\bx}\calH_2(\by)},
\end{align}
for all $\bx,\by$.
}
Therefore, it holds
\begin{align}
\calH(\bx)&=\calH_1(\bx)+\calH_2(\bx)\\
&\le \calH_1(\bx)+\calH_2(\by)  +\Braket{(\bx-\by),\grad_\bx \calH_2(\by)}
\coloneqq \calH_U\(\bx,\by\),
\label{eqn:H_U}
\end{align}
where $\calH_U$ is the upper bound of $\calH$.
Then, for each iteration $t$, we have
\bea
\bx_{t+1} \in \argmin_{\bx} \[\calH_U(\bx,\bx_t)\]=\argmin_{\bx}\[\calH_1(\bx)+\Braket{\bx,\grad_\bx \calH_2(\bx_t)}\],
\eea
which lowers the upper bound $\calH_U$ iteratively and hence decreases the value of $\calH$ monotonically, i.e.
\begin{align}
\calH(\bx_{t+1}) &\leq \calH_U(\bx_{t+1},\bx_t) 
\annot{By \eqref{eqn:H_U}}
\\
&\leq \calH_U(\bx_t,\bx_t)
\annot{Set $\bx=\by$ in \eqref{eqn:H_U}}
\\
&=\calH(\bx_t),
\end{align}
for all $t$.
This completes the proof that $\calH$ can be monotonically decreased by $\calT(\bx)$ given by 
\eqref{def:sparse_retrieval_dyn}.
\end{proof}

\subsection{\cref{coro:eps_sparse_dense}}
\label{sec:pf_eps_sparse_dense}
\begin{proof}[Proof of \cref{coro:eps_sparse_dense}]
Let $\calT_{\text{Dense}}$ be the retrieval dynamics of the dense modern Hopfield model \cite{ramsauer2020hopfield},
and $\norm{\calT(\bx)-\bxi_\mu}$ and $\norm{\calT_{\text{Dense}}(\bx)-\bxi_\mu}$ be the retrieval error of sparse and dense Hopfield model, respectively.

We observe
\begin{align}
&\norm{\calT(\bx)-\bxi_\mu}-\norm{\calT_{\text{Dense}}(\bx)-\bxi_\mu}
\nonumber
\\
&=\norm{\sum_{\nu=1}^{\kappa} \bxi_\nu \[\Sparsemax\(\beta \bm{\Xi}^\sT \bx\) \]_\nu-\bxi_\mu}-\norm{\sum_{\nu=1}^{\kappa} \bxi_\nu \[\Softmax\(\beta \bm{\Xi}^\sT \bx\) \]_\nu-\bxi_\mu}\\
&\leq
\norm{\sum^{\kappa}_{\nu=1}\[\Sparsemax(\beta\bm{\Xi}^\sT \bx)\]_\nu\bxi_\nu }-
\norm{\sum^{\kappa}_{\nu=1} \[\Softmax\(\beta \bm{\Xi}^\sT \bx\)\]_\nu \bxi_\nu}\\
&\leq 0,
\end{align}
which gives
\bea
\label{eqn:eps_sparse_dense}
\norm{\calT(\bx)-\bxi_\mu} 
&\leq& 
\norm{\calT_{\text{Dense}}(\bx)-\bxi_\mu}.
\eea

Next, we provide an upper bound of the sparse retrieval error for a query $\bx\in S_\mu$ given memory patterns $\{\bxi_\nu\}_{\nu\in[M]}$.

According to the \eqref{eqn:sparsemax_exact}, it holds
\begin{align}
[ \Sparsemax \(\beta \bm{\Xi}^\sT \bx\) ]_\mu 
&\le 
\[\beta \bm{\Xi}^\sT \bx \]_{\mu}-\[\beta \bm{\Xi}^\sT \bx\]_{(\kappa)}+\frac{1}{\kappa},
\label{eqn:sparsemax_upper_identity}
\end{align}
for all $\mu\in[M]$.
Then, the sparse retrieval error is 
\begin{align}
\norm{\calT\(\bx\)-\bxi^\mu}&=
\norm{\bm{\bm{\Xi}}\Sparsemax\(\beta \bm{\bm{\Xi}}^\sT \bx\)-\bxi^\mu}
=\norm{\sum_{\nu=1}^{\kappa} \bxi_{(\nu)} \[\Sparsemax\(\beta \bm{\bm{\Xi}}^\sT \bx\) \]_{(\nu)}-\bxi^\mu}\nonumber\\
&\le 
m+m\beta \norm{\sum^{\kappa}_{\nu=1} \(\[ \bm{\Xi}^\sT \bx \]_{(\nu)}-\[ \bm{\Xi}^\sT \bx\]_{(\kappa)}+\frac{1}{\beta\kappa}\)\frac{\bxi_{(\nu)}}{m}}
\annot{By \eqref{eqn:sparsemax_upper_identity}}\\
&= 
m+d^{\nicefrac{1}{2}}m\beta \[\sum^{\kappa}_{\nu=1}\( \[ \bm{\Xi}^\sT \bx \]_{(\nu)}-\[ \bm{\Xi}^\sT \bx\]_{(\kappa)}+\frac{1}{\beta\kappa}\)\]
\\
&\le 
m+d^{\nicefrac{1}{2}}m\beta \[\kappa \(\Max_{\nu\in[M]}\Braket{\bxi_\nu,\bx}-\[ \bm{\Xi}^\sT \bx\]_{(\kappa)}\)+\frac{1}{\beta}\].
\end{align}
\end{proof}

\subsection{\cref{coro:convergence_sparse}}
\label{sec:convergence}
In order to prove \cref{coro:convergence_sparse},  we need the following two auxiliary lemmas.
    \begin{lemma}[\cite{gunawardana2005convergence}, Proposition 7]
    \label{Second condition of generalizaed fixed points}
    Let $\bx_t\in\calX_t$ and $\bx_{t+1}\in\calX_{t+1}$.
    Given a real-valued continuous function $\calH_U$ on $\calX_{t} \times \calX_{t+1}$, define the point-to-set map $\calT: \calX_t \to \calX_{t+1}$ by
    \begin{align}
    \calT(\bx_{t}) &\coloneqq \argmin_{\bx'_{t+1} \in \calX_{t+1}} \calH_U(\bx_t,\bx'_{t+1})\\
    &= \{\bx_{t+1} \;\vert\; \calH_U(\bx_t,\bx_{t+1}) \leq \calH_U(\bx_t,\bx'_{t+1}), \forall \bx'_{t+1}\in \calX_{t+1} \}.
    \end{align}
    Then $\calT$ is a closed map at $\bx_t$ if $\calT(\bx)$ is non-empty.
    \end{lemma}
    \begin{lemma}[\cite{sriperumbudur2009convergence}, Lemma 5]
    \label{KKT for stationary points}
    Recall a fixed point of $\calT$ w.r.t. $\calH$ is a point for which $\bx = \calT(\bx)$, and a generalized fixed point is a point for which $\bx\in\calT(\bx)$.
    Suppose $\bx^\star$ is a generalized fixed point of $\calT$, then, $\bx^\star$ is a stationary point of the minimization problem \eqref{eqn:energy_minimizaiton}.
    \end{lemma}

\begin{proof}[Proof of \cref{coro:convergence_sparse}]
    From Zangwill global convergence theory for iterative algorithms \cite{zangwill1969nonlinear}, all limit points of $\{\bx_t\}^\infty_{t=0}$ are generalized fixed points\footnote{Recall that, a generalized fixed point of $\calT$ is defined as $\bx^\star \coloneqq \{\bx\;|\; \bx\in\calT(\bx)\}$.}, if the energy function $\calH$ and the retrieval dynamics $\calT$ satisfy the following three conditions.
    \begin{enumerate}
        \item [(i)] 
        For any sequence $\{\mathbf{x}_t\}_{t=0}^{\infty}$ with starting point $\bx_0\in S_\mu$, 
        all points \blue{in this sequence} are in the same compact set $S_\mu$.
        \item [(ii)] $\calH$ is monotonically decreased by $\calT(\bx)$, i.e. $\calH(\bx_{t+1}) \leq \calH(\bx_t), \forall \bx_{t+1}=\calT(\bx_t)$.
        \item [(iii)] For all $t$, if $\calH(\bx_{t+1})<\calH(\bx_t)$, $\calT$ is closed at $\bx_t$.
    \end{enumerate}

    From \cref{def:stored_and_retrieved}, 
    since $S_\mu$ with finite radius $R$ is bounded and closed, every $S_\mu$ is a compact set.
    Namely, for any sequence $\{\mathbf{x}_t\}_{t=0}^{\infty}$, all points are embedded in the compact set $S_\mu$. 
    Therefore, condition (i) is automatically satisfied. Then condition (ii), the monotonic descent property of $\{\bx_t\}^\infty_{t=0}$, has been analyzed in the original paper of CCCP \cite{yuille2003concave}. 
    By our definition on $\calH_1$ and $\calH_2$, we have $\calH_U\(\bx,\by\) \coloneqq \calH_1(\bx)+\calH_2(\by)  +\Braket{(\bx-\by),\grad_\bx \calH_2(\by)}$ is continuous in $\bx$ and $\by$. 
    Consequently, by Lemma \ref{Second condition of generalizaed fixed points}, condition (iii) holds due to the non-empty assumption on the point-to-set map $\calT$. 
    Thus, by Zangwill global convergence theory, all the limit points of $\{\bx_t\}^\infty_{t=0}$ are fixed points \Blue{of $\calT$}. 
    Subsequently, by the results of Lemma \ref{KKT for stationary points}, these fixed points are also the stationary points of the minimization problem \eqref{eqn:energy_minimizaiton}. Therefore, the energy function is ensured to converge to local optimum.
\end{proof}

\clearpage
\subsection{\cref{thm:well_separation} and \cref{coro:general_well_separation_cond}}
\label{sec:pf_well_separation}
\begin{proof}[Proof of \cref{thm:well_separation}]
Recall $n\coloneqq \norm{\bx}$.
By \cref{def:separation_of_patterns}, we have
\bea
\Max_{\mu\in [M]}\Braket{\bxi_\mu,\bx}=\Braket{\bxi_\nu,\bx}-\Tilde{\Delta}_\nu,
\eea
thereby obtaining
\bea
\label{eqn:bound_sparse_general}
\norm{\calT(\bx)-\bxi_\mu} &\le&
m+d^{\nicefrac{1}{2}}m\beta \[\kappa \(\Max_{\nu\in[M]}\Braket{\bxi_\nu,\bx}-\[ \bm{\Xi}^\sT \bx\]_{(\kappa)}\)+\frac{1}{\beta}\]\\
&=& m+d^{\nicefrac{1}{2}}m\beta \[\kappa \(\Braket{\bxi_\mu,\bx}-\Tilde{\Delta}_\mu-\[ \bm{\Xi}^\sT \bx\]_{(\kappa)}\)+\frac{1}{\beta}\]
\eea
Since $n\coloneqq \norm{\bx}$ and $m \coloneqq \max_\mu \norm{\bxi^\mu}$, we have
\bea
m+d^{\nicefrac{1}{2}}m\beta \[\kappa \(\Braket{\bxi_\mu,\bx}-\Tilde{\Delta}_\mu-\[ \bm{\Xi}^\sT \bx\]_{(\kappa)}\)+\frac{1}{\beta}\] \\
\leq
m+d^{\nicefrac{1}{2}}m\beta \[\kappa \(mn-\Tilde{\Delta}_\mu-\[ \bm{\Xi}^\sT \bx\]_{(\kappa)}\)+\frac{1}{\beta}\]
\eea
Then, by the Cauchy-Schwartz inequality
\bea
\abs{\Braket{\bxi_\mu,\bxi_\mu}-\Braket{\bx,\bxi_\mu}}
\leq
\norm{\bxi_\mu-\bx} \cdot \norm{\bxi_\mu}
\leq
\norm{\bxi_\mu-\bx}m,\quad\forall \mu\in [M],
\eea
we observe that $\tilde{\Delta}_\mu$ can be expressed in terms of $\Delta_\mu$: 
\begin{align}
\Tilde{\Delta}_\mu 
&=
\Min_{\nu,\nu\neq \mu}
\[
\Braket{\bx,\bxi_\mu}-\Braket{\bx,\bxi_\nu}
+
\(\Braket{\bxi_\mu,\bxi_\mu}-\Braket{\bxi_\mu,\bxi_\nu}\)
-
\(\Braket{\bxi_\mu,\bxi_\mu}-\Braket{\bxi_\mu,\bxi_\nu}\)
\]
\\
&
\geq \Min_{\nu,\nu\neq \mu}
\[
\Braket{\bxi_\mu,\bxi_\mu}-\Braket{\bxi_\mu,\bxi_\nu}
+
\(\Braket{\bxi_\mu,\bxi_\nu}-
\Braket{\bx,\bxi_\nu}\)
-
\(\Braket{\bxi_\mu,\bxi_\mu}-
\Braket{\bx,\bxi_\mu}\)
\]
\annot{By Cauchy-Schwarz}
\\
&=
\Delta_\mu - 2 \norm{\bxi_\mu-\bx}m=\Delta_\mu-2mR,
\annot{By $\bx\in  S_\mu$}
\end{align}
where $R$ is radius of the sphere $S_{\mu}$.
Inserting the bound on $\Tilde{\Delta}_\mu$
, we obtain
\bea
\norm{\calT(\bx)-\bxi_\mu}
\leq
m+d^{\nicefrac{1}{2}}m\beta \[\kappa \(mn-\Delta_\mu+2mR-\[ \bm{\Xi}^\sT \bx\]_{(\kappa)}\)+\frac{1}{\beta}\].
\eea
For $\calT$ to be a mapping from $S_\mu$ to $S_\mu$, we obtain the inequality:
\bea
m+d^{\nicefrac{1}{2}}m\beta \[\kappa \(mn-\Delta_\mu+2mR-\[ \bm{\Xi}^\sT \bx\]_{(\kappa)}\)+\frac{1}{\beta}\]
\leq
R,
\eea
which gives
\bea
\Delta_\mu
\geq
mn+2mR-
\[ \bm{\Xi}^\sT \bx\]_{(\kappa)}-
\frac{1}{\kappa}
\(\frac{R-m-md^{\nicefrac{1}{2}}}{m\beta d^{\nicefrac{1}{2}}}\).
\eea
Therefore, as long as $\Delta_\mu$ satisfies this inequality, $\calT$ is a mapping from $S_\mu$ onto itself.
\end{proof}

\begin{proof}[Proof of \cref{coro:general_well_separation_cond}]
Let $\calT_{\text{Dense}}$ be the retrieval dynamics of the dense modern Hopfield model \cite{ramsauer2020hopfield},
and $\epsilon_{\Sparsemax}\coloneqq \norm{\calT(\bx)-\bxi_\mu}$ and $\epsilon_{\text{Dense}}\coloneqq\norm{\calT_{\text{Dense}}(\bx)-\bxi_\mu}$ be the retrieval error of sparse and dense Hopfield model, respectively.

First, let's recall \cref{coro:eps_sparse_dense}, which states that 
\bea
\label{eqn:eps_sparse_dense}
\epsilon_{\Sparsemax}=\norm{\calT(\bx)-\bxi_\mu} 
&\leq& \epsilon_{\text{Dense}}=
\norm{\calT_{\text{Dense}}(\bx)-\bxi_\mu}.
\eea

Next we want to find the lower bound of separation $\Delta_\mu$ such that $\calT$ is a mapping from $S_\mu$ onto $ S_\mu$.

To link $\Delta_\mu$ to $\calT$, we first bound $\epsilon_{\text{Dense}}$ via \cite[Lemma~A.4]{ramsauer2020hopfield}:
\begin{align}
\epsilon_{\text{Dense}}&=\norm{\calT_{\text{Dense}}(\bx)-\bxi_\mu}\\
&=
\norm{\bxi_\mu-\sum^M_{\nu=1} [\Softmax(\beta \Xi^\sT \bx)]_\nu \bxi_\nu}\\
&=
\norm{\(1-\[\Softmax(\beta \Xi^\sT \bx)\]_\mu\)\bxi_\mu-\sum^M_{\nu=1, \nu \neq \mu}\[\Softmax(\beta \Xi^\sT \bx)\]_\nu\bxi_\nu}\\
&\leq
\tilde{\epsilon} \norm{\bxi_\mu}+\frac{\tilde{\epsilon}}{M-1}\sum^M_{\nu=1, \nu \neq \mu} \norm{\bxi_\nu}\\
&\le 
\tilde{\epsilon} \( m+\frac{1}{M-1}\sum^M_{\nu=1, \nu \neq \mu} m\)
\\
\label{eqn:bound_eps}
&\leq
2\tilde{\epsilon} m,
\end{align}
where $\tilde{\epsilon}\coloneqq (M-1)\exp{-\beta \tilde{\Delta}_\mu}
=(M-1)\exp{-\beta \(\Braket{\bxi_\mu,\bx}-\Max_{\nu\in[M],\nu\neq \mu }\Braket{\bxi_\mu,\bxi_\nu}\)}
$ and the inequality 
\bea
\[\Softmax(\beta \Xi^\sT \bx)\]_\nu
=\frac{\exp{\beta\(\Braket{\bx,\bxi_\nu}-\Braket{\bx,\bxi_\mu}\)}}{1+\sum_{\nu'\neq \mu} \exp{\beta\(\Braket{\bx,\bxi_{\nu'}}-\Braket{\bx,\bxi_\mu}\)}}
\le 
\exp{-\beta \Tilde{\Delta}_\mu},
\eea is used in the fourth line.

Then, by the Cauchy-Schwartz inequality
\bea
\abs{\Braket{\bxi_\mu,\bxi_\mu}-\Braket{\bx,\bxi_\mu}}
\leq
\norm{\bxi_\mu-\bx} \cdot \norm{\bxi_\mu}
\leq
\norm{\bxi_\mu-\bx}m,\quad\forall \mu\in [M],
\eea
we observe that $\tilde{\Delta}_\mu$ can be expressed in terms of $\Delta_\mu$: 
\begin{align}
\Tilde{\Delta}_\mu 
&=
\Min_{\nu,\nu\neq \mu}
\[
\Braket{\bx,\bxi_\mu}-\Braket{\bx,\bxi_\nu}
+
\(\Braket{\bxi_\mu,\bxi_\mu}-\Braket{\bxi_\mu,\bxi_\nu}\)
-
\(\Braket{\bxi_\mu,\bxi_\mu}-\Braket{\bxi_\mu,\bxi_\nu}\)
\]
\\
&
\geq \Min_{\nu,\nu\neq \mu}
\[
\Braket{\bxi_\mu,\bxi_\mu}-\Braket{\bxi_\mu,\bxi_\nu}
+
\(\Braket{\bxi_\mu,\bxi_\nu}-
\Braket{\bx,\bxi_\nu}\)
-
\(\Braket{\bxi_\mu,\bxi_\mu}-
\Braket{\bx,\bxi_\mu}\)
\]
\annot{By Cauchy-Schwarz}
\\
&=
\Delta_\mu - 2 \norm{\bxi_\mu-\bx}m=\Delta_\mu-2mR,
\annot{By $\bx\in  S_\mu$}
\end{align}
where $R$ is radius of the sphere $S_{\mu}$.

Hence, combining the bound from \eqref{eqn:bound_eps} with \eqref{eqn:eps_sparse_dense} results in
\bea
\norm{\calT(\bx)-\bxi_\mu}
&\leq&
\norm{\calT_{\text{Dense}}(\bx)-\bxi_\mu} 
\leq 
2 \Tilde{\epsilon} m\\
&=&
2(M-1)\exp{-\beta \Tilde{\Delta}_\mu}m\\
&\leq&
2(M-1) \exp{-\beta\(\Delta_\mu-2mR\)}m.
\eea
Therefore, given $\delta\coloneqq \norm{\calT_{\text{Dense}}(\bx)-\bxi_\mu}-\norm{\calT(\bx)-\bxi_\mu}\le 0 $, we have 
\bea
\norm{\calT(\bx)-\bxi_\mu} \le 2(M-1) \exp{-\beta\(\Delta_\mu-2mR +\delta\)}m - \delta
\le \norm{\calT_{\text{Dense}}(\bx)-\bxi_\mu} .
\eea

For $\calT$ to be a mapping from $S_\mu$ onto $S_\mu$, it is sufficient to have
\bea
2(M-1) \exp{-\beta(\Delta_\mu-2mR)}m  - \delta
\leq
R,
\eea
which leads to
\bea
\Delta_\mu
\geq
\frac{1}{\beta}\ln(\frac{2(M-1)m}{R+\delta})+2mR.
\eea
\end{proof}

\subsection{\cref{thm:memory_capacity}}
\label{sec:pf_thm_main}
We begin with a helper lemma.
\begin{lemma}[\cite{ramsauer2020hopfield}]
\label{lemma:W}
Given real numbers $a, b \in \mathbb{R}$. 
If the equation
\begin{equation}
ac + c\ln{c} - b = 0,
\end{equation}
holds, then the solution is
\begin{equation}
c = \frac{b}{W_0(\exp(a+\ln{b}))}.
\end{equation}
\end{lemma}
\begin{proof}
Starting from the given equation, we can rearrange and solve for $c$ as follows:
\begin{align*}
ac + c\ln{c} - b &= 0, \\
a + \ln{c} &= \frac{b}{c}, \\
\frac{b}{c} + \ln\left(\frac{b}{c}\right) &= a + \ln{b}, \\
\frac{b}{c}\exp\left(\frac{b}{c}\right) &= \exp(a + \ln{b}), \\
\frac{b}{c} &= W_0(\exp(a + \ln{b})), \\
c &= \frac{b}{W_0(\exp(a + \ln{b}))}.
\end{align*}
This completes the proof.
\end{proof}

Then we present the proof.
\begin{proof}[Proof of \cref{thm:memory_capacity}]
Equipped with 
$\Delta_\mu \ge 
    \frac{1}{\beta}\ln(\frac{2(M-1)m}{R+\delta})+2mR
    $ from \cref{coro:general_well_separation_cond}, we first write down the probability of success storage and retrieval, i.e. minimal separation $\Delta_{\min}$ satisfies well-separation condition.

Let $\Delta_{\min}=\Min_{\mu\in[M]}\Delta_\mu$, and $\theta_{\mu\nu}$ be the angle between two patterns $\bxi^\mu$ and $\bxi^\nu$.
Intuitively, $\theta_{\mu\nu}\in[0,\pi]$ represent the pairwise correlation of two patterns the two patterns.

We have
\bea
\Delta_{\min}\ge 
\frac{1}{\beta}\ln(\frac{2(M-1)m}{R+\delta})+2mR,
\eea
and
\bea
\Delta_{\min}=\Min_{1\le\mu\le\nu\le M}\[m^2\(1-\cos{\theta_{\mu\nu}}\)\]
=m^2 \[1-\cos{\theta_{\min}}\],
\eea
where $\theta_{\min}\coloneqq \Min_{1\le\mu\le\nu\le M} \theta_{\mu\nu}\in[0,\pi]$.
Then, it holds
\bea
\label{eqn:cos}
m^2 \[1-\cos{\theta_{\min}}\] 
\ge 
\frac{1}{\beta}\ln(\frac{2(M-1)m}{R+\delta})+2mR.
\eea

With \cref{coro:general_well_separation_cond}
, we write down the probability of success storage and retrieval as
\bea
P\(\Delta_\mu\ge \frac{1}{\beta}\ln(\frac{2(M-1)m}{R+\delta})+2mR \)= 1-p.
\eea
By \eqref{eqn:cos}, we have
\bea
\label{success}
P\( m^2 \[1-\cos{\theta_{\min}}\] \ge  \frac{1}{\beta}\ln(\frac{2(M-1)m}{R+\delta})+2mR \) = 1-p.
\eea
By \cite[(4.22.2)]{olver2010nist}, for $0 \le \cos{\theta_{\min}} \le 1$, $\cos{\theta_{\min}}$ can be upper bounded by:
\bea
\cos{\theta_{\min}} \le 1- \frac{\theta_{\min}^2}{5}.
\eea
It holds
\bea
P\( \frac{m^2\theta_{\min}^2}{5} \ge \frac{1}{\beta}\ln(\frac{2(M-1)m}{R+\delta})+2mR\)= 1-p,
\eea
which can be rewritten as
\bea
P\(M^{\frac{2}{d-1}}\theta_{\min} \ge \frac{\sqrt{5}M^{\frac{2}{d-1}}}{m}\[\frac{1}{\beta}\ln(\frac{2(M-1)m}{R+\delta})+2mR\]^\half\) = 1-p.
\label{eqn:ineq1}
\eea
Here, $M^{\nicefrac{2}{d-1}}$ is introduced for later convenience.

Let 
\bea
\omega_d
\coloneqq 
\frac{2\pi^{\nicefrac{d+1}{2}}}{\Gamma\(\frac{d+1}{2}\)},
\eea
be the surface area of a $d$-dimensional unit sphere is, where $\Gamma(\cdot)$ represents the gamma function.

By \cite[Lemma~3.5]{brauchart2018random}, we obtain
\begin{align}
P&\(M^{\frac{2}{d-1}}\theta_{\min} \ge \frac{\sqrt{5}M^{\frac{2}{d-1}}}{m}\[\frac{1}{\beta}\ln(\frac{2(M-1)m}{R+\delta})+2mR\]^{\half}\)
= 1-p
\nonumber\\
&\quad\quad\quad\ge
1-\half\gamma_{d-1}5^{\frac{d-1}{2}}M^2 m^{-(d-1)}\[\frac{1}{\beta}\ln(\frac{2(M-1)m}{R+\delta})+2mR\]^{\frac{d-1}{2}},
\label{eqn:ineq2}
\end{align}
where $\gamma_d$ is defined as the ratio of surface areas of $(d-1)$- and $d$-dimensional unit sphere:
\bea
\gamma_d \coloneqq \frac{1}{d} \frac{\omega_{d-1}}{\omega_d}=\frac{1}{d\sqrt{\pi}}\frac{\Gamma\(\frac{d+1}{2}\)}{\Gamma\(\frac{d}{2}\)}.
\eea
Recall $d, M\in \mathbb{N}_+$, $p\in [0,1]$. 
With some real value $C\in\R$, it holds
\bea
M=\sqrt{p}C^{\frac{d-1}{4}}.
\eea
From \eqref{eqn:ineq2}, we have
\bea
5^{\frac{d-1}{2}}
    \(\sqrt{p}C^{\frac{d-1}{4}}\)^2 m^{-(d-1)}\Bigg\{\frac{1}{\beta}\ln{\[\frac{2\(\sqrt{p}C^{\frac{d-1}{4}}-1\)m}{R+\delta}\]}+\frac{1}{\beta}\Bigg\}^{\frac{d-1}{2}}-p \le 0,
\eea
which leads to
\bea
5^{\frac{d-1}{2}}C^{\frac{d-1}{2}} m^{-(d-1)}\Bigg\{\frac{1}{\beta}\ln{\[\frac{2\(\sqrt{p}C^{\frac{d-1}{4}}-1\)m}{R+\delta}\]}+\frac{1}{\beta}\Bigg\}^{\frac{d-1}{2}}\le 1.
\label{eqn:ineq3}
\eea

To apply \cref{lemma:W}, we first rearrange \eqref{eqn:ineq3} as
\bea
\frac{5C}{m^2\beta}\Bigg\{\ln\[\frac{2\(\sqrt{p}C^{\frac{d-1}{4}}-1\)m}{R+\delta}\]+1\Bigg\}-1 \leq 0,
\eea
and then identify
\bea
\label{eqn:identified_ab}
a\coloneqq \frac{4}{d-1}\Bigg\{\ln\[\frac{2m(\sqrt{p}-1)}{R+\delta}\]+1\Bigg\}, 
\quad 
b\coloneqq \frac{4m^2\beta}{5(d-1)}.
\eea

By \cref{lemma:W}, we have the solution
\bea
\label{eqn:c}
C=\frac{b}{W_0({\exp\{a+\ln{b}\}})},
\eea
where $W_0(\cdot)$  is the upper branch of the Lambert $W$ function.
Since the domain of the Lambert $W$ function is $x>-\nicefrac{1}{e}$ and {the fact $\exp{a+\ln{b}}>0$}, the solution exists.

When $C$ satisfies inequality \eqref{eqn:ineq3}, we arrive a lower bound on the exponential storage capacity $M$:
\bea
M
\geq
\sqrt{p}C^{\frac{d-1}{4}}.
\eea
Notably, the above takes similar form as \cite[Theorem~3]{ramsauer2020hopfield}.
To see the blessings of sparsity, we consider the following asymptotic analysis and compare with results from the dense modern Hopfield model.
To compare with results from dense modern Hopfield model, we denote the dense counterparts of  $a,b$ with $\Tilde{\cdot}$ notation, i.e. 
\bea
\Tilde{a}\coloneqq \frac{2}{d-1} \[1+\ln{\(2\beta m^2 p\)}\],\quad \Tilde{b}=b.
\eea

By \cite{corless1996lambert}, for sufficient large $z$, $W_0(z)$ is asymptotic to
\bea
W_0(z)\simeq \ln{z} - \ln{\ln{z}}+\calO(1).
\eea

Therefore,
for sufficient large $\beta $, we have
\bea
W_0\(\exp{a+\ln b}\)\simeq
a+\ln b-\ln\({a+\ln b}\)+\calO(1),
\eea
which is dominated by $a$.

For $a$, we have
\bea
\Tilde{a}
\leq
a,
\eea
and hence
\bea
W_0\(\exp{\Tilde{a}+\ln \Tilde{b}}\)
\leq
W_0\(\exp{a+\ln b}\).
\eea
Therefore, combining above with  \eqref{eqn:c}, we have
\bea
\Tilde{C} = \frac{b}{W_0\(\exp{\Tilde{a}+\ln \Tilde{b}}\)} \leq \frac{b}{W_0\(\exp\{a+\ln b\}\)} = C,
\eea
which states that the lower bound of the sparse capacity is larger than that of \cite{ramsauer2020hopfield}
\bea
M=\sqrt{p}C^{\frac{d-1}{4}}
\geq
\sqrt{p}\Tilde{C}^{\frac{d-1}{4}}=M_{\text{Dense}}.
\eea
\end{proof}

\section{Auxiliary Theoretical Background}
\label{sec:axuiliary_theory}

\begin{remark}[Remark on \cref{def:variational_sparsemax}]
\label{remark:sparsemax}
To see the equivalence of the two optimization problems, we observe
\begin{align}
    \argmax_{\bp\in\Delta^d} \[\Braket{\bp,\bz} -\Psi(\bp)\]
    &=
    \argmax_{\bp\in\Delta^d} \[\Braket{\bp,\bz} -\half\norm{\bp}^2\]\nonumber\\
    \label{eqn:sparse_equ2}
    &=\argmin_{\bp\in\Delta^d} \[-\half\(\norm{\bp}^2 +\norm{\bz}^2-2\Braket{\bp,\bz} \)\]\\
    &=\argmin_{\bp\in\Delta^d}\[\half\norm{\bp-\bz}^2\]\nonumber,
\end{align}
where the last line is obtained by inserting $\norm{\bz}^2$ as a constant in \eqref{eqn:sparse_equ2}.
\end{remark}

\clearpage

\section{Additional Experiments}
\label{sec:syn_additional}
In order to highlight the benefits of the sparse Hopfield model, particularly under conditions of high data sparsity, we broaden our experimental studies with more models. 
These models include the \texttt{SparseHopfield}, \texttt{Hopfield}, the attention mechanism \cite{vaswani2017attention}, and a attention-based MIL baseline, the gated-attention mechanism \cite{ilse2018attention}.
 
\subsection{Visualization of Experimental Validation of Theoretical Results}
We provide visual demonstrations of \cref{sec:exp_theory} in \cref{fig:visual_aid}.
\begin{figure*}[h]
    \centering
    \hfill
    {    \includegraphics[width=1\textwidth]{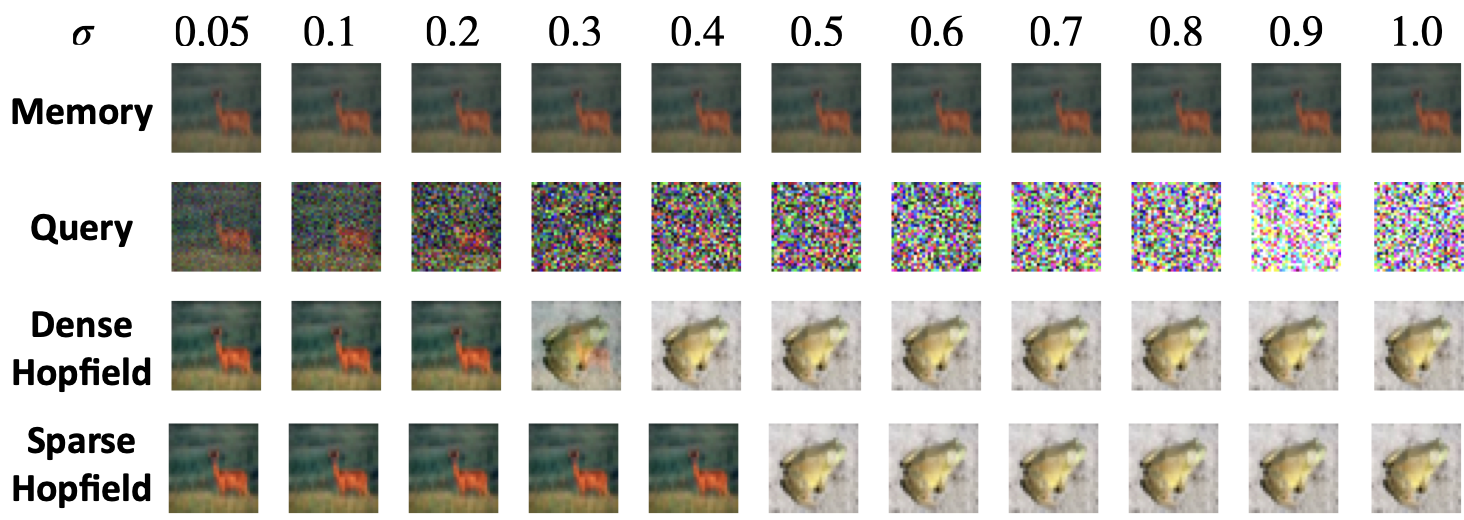}
    }
    \\
    \hfill
    {    \includegraphics[width=1\textwidth]{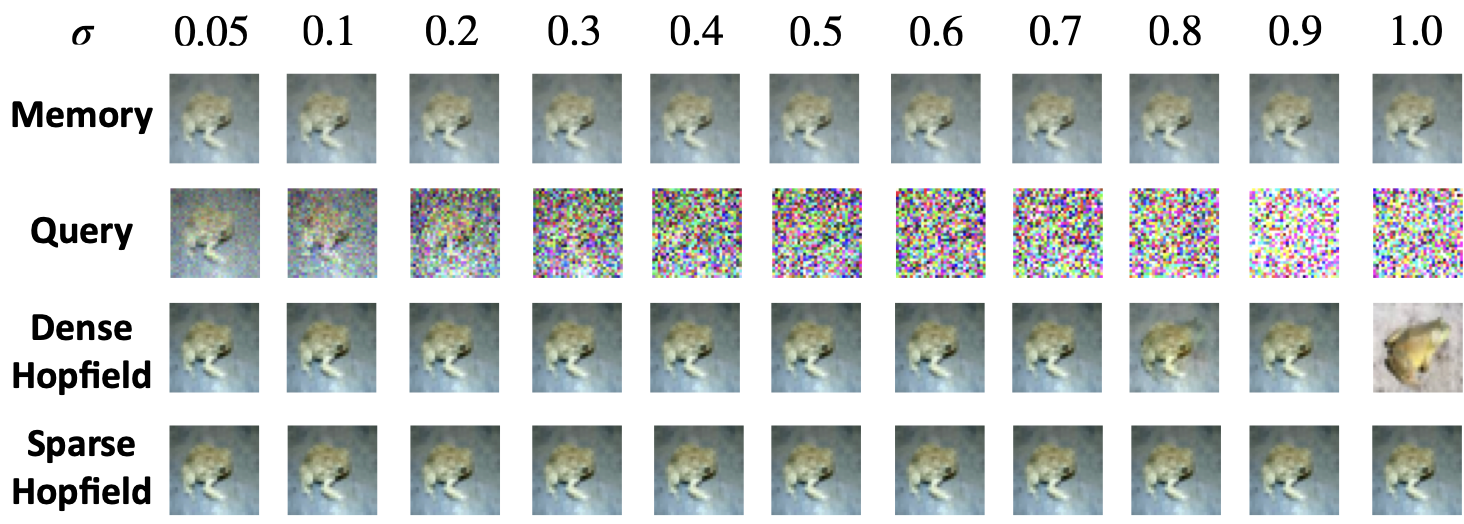}
    }
    \\
    \hfill
    {    \includegraphics[width=1\textwidth]{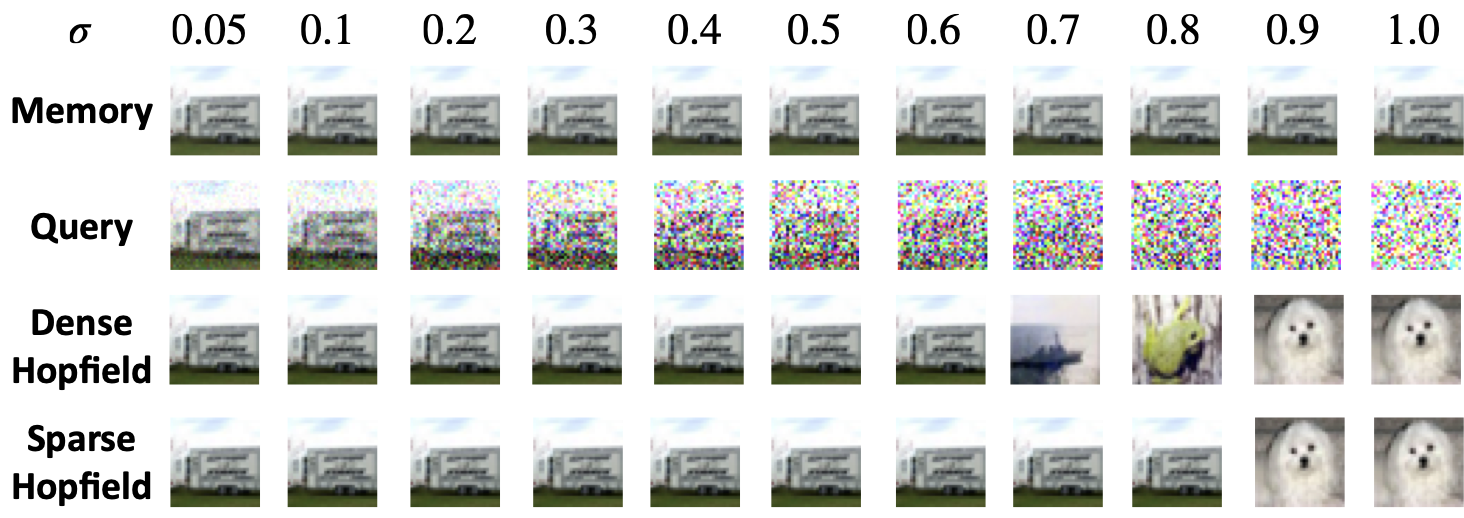}
    }
    \caption{\textbf{Visualizing noise-robustness of sparse and dense Hopfield models (\cref{fig:capacity_robustness} of \cref{sec:exp_theory}).}
    We perform memory retrieval using both Dense and Sparse Hopfield models, with queries subjected to varying levels of noise. We randomly select an image from the CIFAR10 dataset to serve as the memory pattern. This selected image is then contaminated with different levels of random noise ($\mu=0$ and $ 0.05\leq \sigma \leq 1.5$) to generate query patterns. The results demonstrate that the Sparse Hopfield model is more effective in retrieving the original image, showcasing its superior robustness against noise.}
    \label{fig:visual_aid}
\end{figure*}

\subsection{Bit Pattern MIL}
To supplement \cref{sec:syn_exp}, we conduct further numerical investigations on the same MIL tasks (\textbf{bag sparsity} and \textbf{bag size}) with \texttt{SparseHopfield}, \texttt{Hopfield}.
In these experiments, we contrast the performance of \texttt{SparseHopfield} and \texttt{Hopfield} (and also \texttt{SparseHopfieldPooling} and \texttt{HopfieldPooling}) with the attention mechanism \cite{vaswani2017attention} and the gated-attention mechanism  \cite{ilse2018attention}.
For the \textbf{bag size}, we fix the number of positive pattern in a bag to be 1, and vary bag size from 20 to 300.
For the \textbf{bag sparsity}, we fix the bag size as 200, and inject from 2 to 100 positive patterns in a positive bag, results in 1 to 50 percent of positive patterns in each positive bag.
The results are reported in \cref{table:app_exp}.
For numerical experiments on synthetic datasets, we do not use hyperparameter search due to the simplicity of both model structure and data.

\begin{table}[H]
\vspace{-1.0em}
    \centering
    \caption{\textbf{Top (Bag Size):} Accuracy comparison on bit pattern dataset for sparse and dense Hopfield model. 
    We report the average accuracy over 10 runs.
    The results suggest that the sparse Hopfield model demonstrates a better performance when facing a bag size increase.
    \textbf{Bottom (Bag Sparsity):}
    Performance comparison on bit pattern dataset for sparse and dense Hopfield model with varying bag sparsity. 
    We report the average accuracy over 10 runs. 
    The results suggest that the sparse Hopfield model demonstrates a better performance across all sparsity.}
    \vspace*{0.05truein} 

\resizebox{ \textwidth}{!}{    
    \begin{tabular}{cccccccc}
    \toprule
    \toprule
    \makecell{\rotatebox{0}{\small Bag Size}} 
         & \multicolumn{1}{c}{20} 
         & \multicolumn{1}{c}{50}
         & \multicolumn{1}{c}{100} 
         & \multicolumn{1}{c}{150} 
         & \multicolumn{1}{c}{200}
         & \multicolumn{1}{c}{300}

        \\
        \midrule
        \makecell{\rotatebox{0}{\small Sparse Hopfield}} 
        & 98.82 $\pm$ 0.34
        & 99.45 $\pm$ 0.19
        & 97.13 $\pm$ 0.11
        & 95.98 $\pm$ 0.12
        & 94.17 $\pm$ 0.01
        & 90.15 $\pm$ 0.30
        \\
        \makecell{\rotatebox{0}{\small Dense Hopfield}} 
        & 99.65 $\pm$ 0.70
        & 99.51 $\pm$ 0.87
        & 53.90 $\pm$ 0.00
        & 49.51 $\pm$ 0.02
        & 51.92 $\pm$ 0.12
        & 53.83 $\pm$ 0.12
        \\
        \makecell{\rotatebox{0}{\small Sparse Hopfield Pooling}} 
        & \textbf{99.71 $\pm$ 0.06}
        & 100.0 $\pm$ 0.00
        & 100.0 $\pm$ 0.00
        & \textbf{99.76 $\pm$ 0.00}
        & \textbf{99.76 $\pm$ 0.00}
        & \textbf{99.76 $\pm$ 0.00}
        \\
        \makecell{\rotatebox{0}{\small Dense Hopfield Pooling}} 
        & 99.68 $\pm$ 0.15
        & 100.0 $\pm$ 0.00
        & 100.0 $\pm$ 0.00
        & 76.44 $\pm$ 0.23
        & 49.13 $\pm$ 0.01
        & 52.88 $\pm$ 0.01
        \\
        \midrule
        \makecell{\rotatebox{0}{\small Attention}} 
        & 87.01 $\pm$ 0.00
        & 74.51 $\pm$ 0.01
        & 45.19 $\pm$ 0.31
        & 53.75 $\pm$ 0.76
        & 46.63 $\pm$ 0.02
        & 53.36 $\pm$ 0.03
        \\
        \makecell{\rotatebox{0}{\small Gated}} 
        & 87.88 $\pm$ 0.00
        & 63.44 $\pm$ 0.04
        & 75.38 $\pm$ 0.56
        & 73.45 $\pm$ 0.70
        & 71.05 $\pm$ 0.35
        & 49.61 $\pm$ 1.78
        \\
    \end{tabular}
    \vspace*{-0.5truein}3}
\resizebox{ \textwidth}{!}{    
    \begin{tabular}{ccccccccc}
    \toprule
    \toprule
    \makecell{\rotatebox{0}{\small Bag Sparsity}} 
         & \multicolumn{1}{c}{1\%} 
         & \multicolumn{1}{c}{2\%} 
         & \multicolumn{1}{c}{3\%}
         & \multicolumn{1}{c}{5\%} 
         & \multicolumn{1}{c}{10\%} 
         &  \multicolumn{1}{c}{20\%}
         &  \multicolumn{1}{c}{40\%}
         &  \multicolumn{1}{c}{50\%}
        \\
        \midrule
        \makecell{\rotatebox{0}{\small Sparse Hopfield }} 
        & 95.62 $\pm$ 0.01
        & 95.98 $\pm$ 0.30
        & 99.68 $\pm$ 0.01
        & 100.0 $\pm$ 0.00
        & 100.0 $\pm$ 0.00
        & 100.0 $\pm$ 0.00
        & 100.0 $\pm$ 0.00
        & 100.0 $\pm$ 0.00
        \\
        \makecell{\rotatebox{0}{\small Dense Hopfield }} 
        & 51.44 $\pm$ 0.01
        & 57.21 $\pm$ 0.01
        & 75.48 $\pm$ 0.01
        & 99.03 $\pm$ 0.11
        & 99.51 $\pm$ 0.02
        & 100.0 $\pm$ 0.00
        & 100.0 $\pm$ 0.00
        & 100.0 $\pm$ 0.00
        \\
        \makecell{\rotatebox{0}{\small Sparse Hopfield Pooling}} 
        & \textbf{99.76 $\pm$ 0.00}
        & \textbf{99.68 $\pm$ 0.00}
        & 100.0 $\pm$ 0.00
        & 100.0 $\pm$ 0.00
        & 100.0 $\pm$ 0.00
        & 100.0 $\pm$ 0.00
        & 100.0 $\pm$ 0.00
        & 100.0 $\pm$ 0.00        
        \\ 
        \makecell{\rotatebox{0}{\small Dense Hopfield Pooling}} 
        & 49.20 $\pm$ 0.00
        & 85.58 $\pm$ 0.10
        & 100.0 $\pm$ 0.00        
        & 100.0 $\pm$ 0.00
        & 99.68 $\pm$ 0.00
        & 100.0 $\pm$ 0.00
        & 100.0 $\pm$ 0.00
        & 100.0 $\pm$ 0.00    
        \\
        \midrule
        \makecell{\rotatebox{0}{\small Attention}} 
        & 74.51 $\pm$ 0.01
        & 78.81 $\pm$ 0.04
        & 96.63 $\pm$ 0.02
        & 100.0 $\pm$ 0.00
        & 99.51 $\pm$ 0.01
        & 100.0 $\pm$ 0.00
        & 100.0 $\pm$ 0.00
        & 100.0 $\pm$ 0.00
        \\
        \makecell{\rotatebox{0}{\small Gated}} 
        & 78.94 $\pm$ 0.41
        & 95.28 $\pm$ 0.35
        & 98.55 $\pm$ 0.00
        & 99.03 $\pm$ 0.01
        & 100.0 $\pm$ 0.00
        & 100.0 $\pm$ 0.00
        & 100.0 $\pm$ 0.00
        & 100.0 $\pm$ 0.00
        \\
        \bottomrule\bottomrule
    \end{tabular}
    \vspace*{-0.5truein}
}
\label{table:app_exp}
\vspace{-1em}
\end{table}

\subsection{Convergence Analysis}
\label{sec:syn_abliation_convegence}
\begin{figure*}[t]
    \centering
    \hfill
    {    \includegraphics[width=1\textwidth]{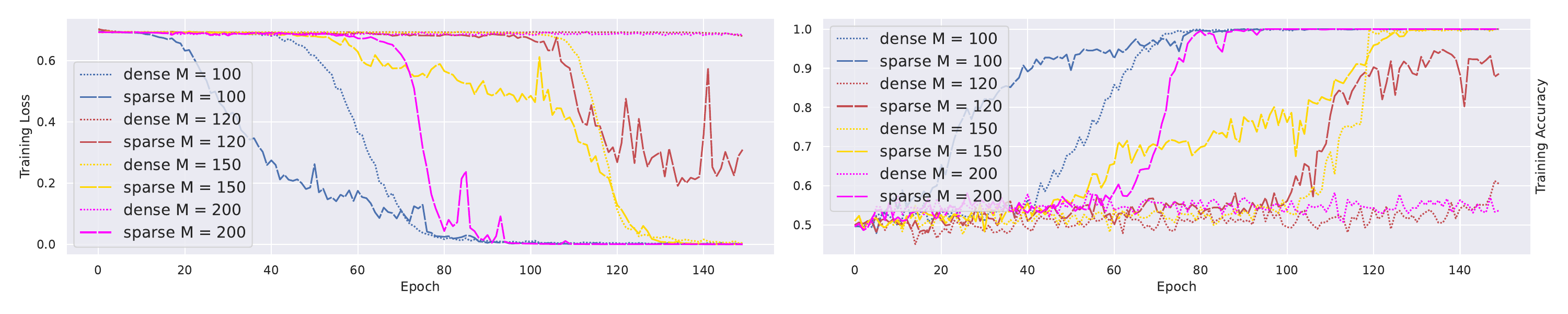}
    }
    \\
    \vspace{-2em}
    \hfill
    {    \includegraphics[width=1\textwidth]{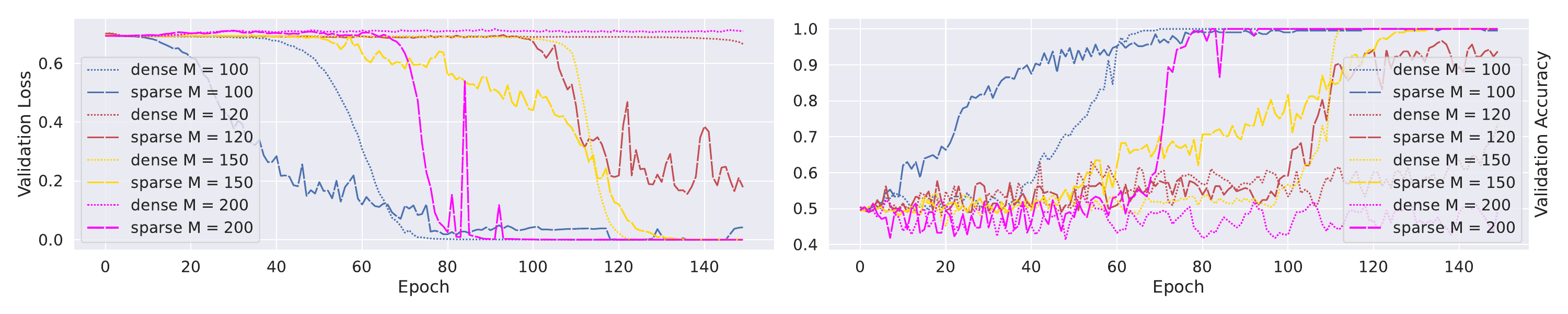}
    }
    \vspace{-2em}
    \caption{\textbf{Top:} The training loss and accuracy curve of \texttt{SparseHopfield} and \texttt{Hopfield} with different bag sizes. 
    \textbf{Bottom:} The validation loss and accuracy curve of \texttt{SparseHopfield} and \texttt{Hopfield} with different bag sizes. 
    The plotted are the mean of 10 runs.
    The results indicate that the sparse Hopfield model converges faster than the dense model and also yields superior validation/test accuracy.
    }
    \label{fig:convg2}
\end{figure*}

To supplement \cref{sec:syn_exp}, we also analyze the convergence behavior of the \texttt{SparseHopfield} and \texttt{Hopfield} numerically.
In \cref{fig:convg2}, we plot the loss and accuracy curve for both models on the bit pattern dataset for the \textbf{bag size} tasks mentioned in \cref{sec:syn_exp}. 
We include various bag sizes in the plot to 
examine how the loss curve responds to an increase in bag size (i.e., the number of memory patterns, $M$).
The results show that, \texttt{SparseHopfield} surpasses the \texttt{Hopfield} in nearly all bag sizes. 
Moreover, for the same bag size, \texttt{SparseHopfield} always reaches the minimum validation loss faster than \texttt{Hopfield}.
This provides empirical support for our theoretical prediction outlined in \cref{eqn:sparse_error_bound}. 
In conjunction with the findings illustrated in \cref{fig:convg}, \cref{fig:convg2} reinforces the benefits of utilizing the sparse Hopfield model. 
In particular, the evidence verifies the claim in \cref{eqn:sparse_error_bound}, demonstrating that the convergence speed of the sparse and dense Hopfield models shows different dependencies on the bag size $M$ in this experiment.

\subsection{Sparsity Generalization}
\label{sec:syn_abliation_gen}
To supplement \cref{sec:syn_exp}, 
we explore another scenario where the bag sparsity shifts between training and test data.
We train dense and sparse Hopfield models on a certain bag sparsity, and evaluate on another. 
The main goal of this setting is to investigate the generalization performance of dense and sparse Hopfield models when the information sparsity shift of training and test data distribution.
We fix the bag size to 200, and then implant different number of positive signals to the training and test dataset range from 0.5 to 50 percent.
We report the results of \texttt{HopfieldPooling} and \texttt{SparseHopfieldPooling} in \cref{table:sparsity_generalization}.\footnote{For the ease of presentation, we exclude the standard deviations in \cref{table:sparsity_generalization} as they are all close to zero and less than 0.31\%}
The result shows that for \texttt{HopfieldPooling}, while train on dense bags help its performance of evaluating on sparse bags,  lacking ability of learning from sparse bags still affects its performance.
Meanwhile \texttt{SparseHopfieldPooling} is more robust against sparsity shift, especially for the case where it was trained on dense bags and evaluate on sparse bags.
However, both sparse and dense Hopfield models inevitably suffer from a performance drop when having a sparsity gap when train bags are much more sparse than test bags.

\begin{table}[h]
    \centering
    \caption{\textbf{Accuracy comparison on bit pattern dataset for sparse and dense Hopfield Model when varying the train/test sparsity gap. }
    We report the average accuracy over 10 runs. 
    The result shows that for \texttt{HopfieldPooling}, while train on dense bags help its performance of evaluating on sparse bags,  lacking ability of learning from sparse bags still affects its performance.
    Meanwhile \texttt{SparseHopfieldPooling} is more robust against sparsity shift, especially for the case where it was trained on dense bags and evaluate on sparse bags.
    However, both sparse and dense Hopfield models inevitably suffer from a performance drop when having a sparsity gap when train bags are much more sparse than test bags.}
    \vspace*{0.05truein}

\resizebox{ \textwidth}{!}{    
    \begin{tabular}{c|ccccccc}
    \toprule
    \toprule
         &
         & \multicolumn{4}{c}{\# of Train Positive Signal per Bag (Dense/Sparse)} 
         \\
    \makecell{\rotatebox{0}{\small \# of Test}} 
         & \multicolumn{1}{c}{1} 
         & \multicolumn{1}{c}{2} 
         & \multicolumn{1}{c}{10} 
         &  \multicolumn{1}{c}{20}
         &  \multicolumn{1}{c}{40}
         &  \multicolumn{1}{c}{80}
         &  \multicolumn{1}{c}{100}
         \\
        \midrule
        \makecell{\rotatebox{0}{\small 1}} 
        & 46.63 / 99.76
        & 48.55 / 94.31
        & 53.84 / 74.52
        & 59.61 / 81.73
        & 66.82 / 81.25
        & 72.07 / 81.73
        & 72.59 / 81.25
        \\
        \makecell{\rotatebox{0}{\small 2}} 
        & 47.59 / 52.40
        & 51.44 / 98.18
        & 58.17 / 95.19
        & 62.01 / 95.67
        & 69.23 / 95.67
        & 72.59 / 95.67
        & 72.11 / 95.67
        \\
        \makecell{\rotatebox{0}{\small 10}} 
        & 99.51 / 100.0
        & 99.51 / 100.0
        & 99.51 / 100.0
        & 99.51 / 100.0
        & 99.51 / 100.0
        & 99.51 / 100.0
        & 99.51 / 100.0
        \\
        \makecell{\rotatebox{0}{\small 20}} 
        & 100.0 / 100.0
        & 100.0 / 100.0
        & 100.0 / 100.0
        & 100.0 / 100.0
        & 100.0 / 100.0
        & 100.0 / 100.0
        & 100.0 / 100.0
        \\
        \makecell{\rotatebox{0}{\small 40}} 
        & 100.0 / 100.0
        & 100.0 / 100.0
        & 100.0 / 100.0
        & 100.0 / 100.0
        & 100.0 / 100.0
        & 100.0 / 100.0
        & 100.0 / 100.0
 
        \\
        \makecell{\rotatebox{0}{\small 80}} 
        & 100.0 / 97.03
        & 100.0 / 100.0
        & 100.0 / 100.0
        & 100.0 / 100.0
        & 100.0 / 100.0
        & 100.0 / 100.0
        & 100.0 / 100.0

        \\
        \makecell{\rotatebox{0}{\small 100}} 
        & 100.0 / 100.0
        & 100.0 / 100.0
        & 100.0 / 100.0
        & 100.0 / 100.0
        & 100.0 / 100.0
        & 100.0 / 100.0
        & 100.0 / 100.0

        \\ \bottomrule\bottomrule
    \end{tabular}
    \label{table:sparsity_generalization}
    \vspace*{-0.5truein}
}
\end{table} 

{\color{Blue}
\subsection{Real-World Experiments}

To examine the practical applicability of the proposed model, we implement it in two additional experiments that utilize transformer-based models for distinct tasks. 
These tasks include multivariate time series prediction \cite{zhang2022crossformer}, and neural machine translation \cite{vaswani2017attention}. 
In these experiments, we substitute the existing attention mechanism with both the $\mathtt{Hopfield}$ and $\mathtt{SparseHopfield}$ layers.

\subsubsection{Multivariate Time Series Prediction}
\label{sec:MTS}
For the multivariate time series prediction task, we implement two variants of the SOTA Crossformer model \cite{zhang2022crossformer}, \textbf{Crossformer-DH} and \textbf{Crossformer-SH}, with $\mathtt{Hopfield}$ and $\mathtt{SparseHopfield}$ layers respectively. 
These models employ an architecture akin to the Swin-Transformer \cite{liu2021swin}, utilizing shifting windows to extract information at multiple resolutions.
The experiment results are showed in \cref{table:mts-result}.
Our results indicate that our proposed $\mathtt{SparseHopfield}$ not only consistently enhances transformer-based deep learning models but also achieves SOTA performance.
In 60+\% of 58 settings, the Sparse Hopfield model, Crossformer-SH, ranks first or second, with 30 topand 7 runner-up performances.

\paragraph{Datasets.} We conduct the experiments on four multivariate time series real-world datasets: ETTh1 (Electricity Transformer Temperature-hourly), ETTm1 (Electricity Transformer Temperature-minutely),  WTH (Weather), ILI (Inﬂuenza-Like Illness), ECL (Electricity Consuming Load), Traffic.

\paragraph{Baselines.} 
We benchmark our method against the results of \cite{zhang2022crossformer} and other baselines (Tranformer \cite{vaswani2017attention}, Informer \cite{zhou2021informer} and Autoformer \cite{chen2021autoformer}) therein. 

\paragraph{Setup.}
We adopt the same setting as in \cite{zhang2022crossformer}: multivariate time series prediction task on various datasets.
Following \cite{zhang2022crossformer}, for each dataset, we evaluate our models with several different prediction horizons.
As for hyperparameters, we simply adopt the optimized hyperparameter configuration used in \cite{zhang2022crossformer} obtained via grid search for both \cite{zhang2022crossformer} and all baselines.
We report the average accuracy of 5 runs, evaluated using Mean Square Error (MSE) and Mean Absolute Error (MAE) metrics.

\begin{table}
\centering
\caption{\textbf{Accuracy comparison for multivariate time series predictions on various datasets, using both the sparse and dense Hopfield models.}
Based on SOTA prediction model Crossformer \cite{zhang2022crossformer}, we implement two Crossformer variants, \textbf{Crossformer-DH} and \textbf{Crossformer-SH}, with $\mathtt{Hopfield}$ and $\mathtt{SparseHopfield}$ layers respectively. 
We report the average accuracy of 5 runs, evaluated using Mean Square Error (MSE) and Mean Absolute Error (MAE) metrics. 
We benchmark our method against the results of \cite{zhang2022crossformer} and other baselines (Transformer \cite{vaswani2017attention}, Informer \cite{zhou2021informer} and Autoformer \cite{chen2021autoformer}) therein. 
We evaluate each dataset with different prediction horizons (showed in the second column).
We have the best results \textbf{bolded} and the second best results \underline{underlined}.
In 60+\% of 58 settings, the Sparse Hopfield model, Crossformer-SH, ranks first or second, with 30 top and 7 runner-up performances.
Our results indicate that our proposed $\mathtt{SparseHopfield}$ not only consistently enhances transformer-based deep learning models but also achieves SOTA or comparable performance.
}
\resizebox{ \textwidth}{!}{    
\begin{tabular}{cc|cccccccccccccccccc}
\toprule
 \multicolumn{2}{c|}{Models} &  \multicolumn{2}{c}{Transformer} & \multicolumn{2}{c}{Informer} & \multicolumn{2}{c}{Autoformer} & \multicolumn{2}{c}{Crossformer} & \multicolumn{2}{c}{\textbf{Crossformer-DH}} & \multicolumn{2}{c}{\textbf{Crossformer-SH}} \\
\midrule
 \multicolumn{2}{c|}{Metric} & MSE & MAE & MSE & MAE & MSE & MAE & MSE & MAE & MSE & MAE  & MSE & MAE \\
\midrule
\multirow{6}{1em}{\rot{ETTh1}} & 24 & 0.620 & 0.577 & 0.577 & 0.549 & 0.439 & 0.440 & 0.305 & 0.367 & \underline{0.299} & \underline{0.365} & \textbf{0.295} & \textbf{0.357} \\ 
& 48 & 0.692 & 0.671 & 0.685 & 0.625 & 0.429 & 0.442 & 0.352 & \underline{0.394} & \underline{0.351} & 0.399 & \textbf{0.346} & \textbf{0.392} \\ 
 & 168 & 0.947 & 0.797 & 0.931 & 0.752 & 0.493 & 0.479 & \textbf{0.410} & \textbf{0.441} &  \underline{0.412} & \underline{0.443} & 0.425 & 0.455 \\ 
& 336 & 1.094 & 0.813 & 1.128 & 0.873 & 0.509 & 0.492 & \textbf{0.440} & \textbf{0.461} & 0.455  &  0.468 & 0.459 & 0.477   \\ 
& 720 & 1.241 & 0.917 & 1.215 & 0.896 & 0.539 & 0.537 & \underline{0.519} & \underline{0.524} & 0.523 & 0.529 & \textbf{0.518} & \textbf{0.522}  \\ 
\midrule
\multirow{6}{1em}{\rot{ETTm1}} & 24 & 0.306 & 0.371 & 0.323 & 0.369 & 0.410 & 0.428 & 0.310 & 0.371 & \underline{0.199} & \underline{0.285} & \textbf{0.198} & \textbf{0.287} \\ 
& 48 & 0.465 & 0.470 & 0.494 & 0.503 & 0.483 & 0.464 & 0.300 & \underline{0.352} & \underline{0.290} & 0.356 &  \textbf{0.276} & \textbf{0.340}  \\ 
 & 96 & 0.681 & 0.612 & 0.678 & 0.614 & 0.502 & 0.476 & \underline{0.320} & \underline{0.373} & 0.344 & 0.398 & \textbf{0.305} & \textbf{0.371} \\ 
& 288 & 1.162 & 0.879 & 1.056 & 0.786 & 0.604 & 0.522 & 0.404 & \underline{0.427} & \underline{0.404} & 0.429 & \textbf{0.373} & \textbf{0.406} \\ 
& 672 & 1.231 & 1.103 & 1.192 & 0.926 & 0.607 & 0.530 & 0.569 & 0.528 & \underline{0.568} & \underline{0.523} & \textbf{0.467} & \textbf{0.474}  \\ 
\midrule
 \multirow{6}{1em}{\rot{WTH}} & 24 & 0.349 & 0.397 & 0.335 & 0.381 & 0.363 & 0.396 & \textbf{0.294} & \textbf{0.343} & \textbf{0.294} & \textbf{0.343} & \textbf{0.294} & \underline{0.344} \\ 
& 48 & 0.386 & 0.433 & 0.395 & 0.459 & 0.456 & 0.462 & \underline{0.370} & 0.411 & \textbf{0.369} & \textbf{0.408} & 0.375 & \underline{0.410} \\ 
& 168 & 0.613 & 0.582 & 0.608 & 0.567 & 0.574 & 0.548 & \underline{0.473} & \underline{0.494} & \textbf{0.472} & \textbf{0.493} & 0.480 & 0.499 \\ 
 & 336 & 0.707 & 0.634 & 0.702 & 0.620 & 0.600 & 0.571 & \textbf{0.495} & \textbf{0.515} & 0.498 & \underline{0.519} & 0.504 & 0.523 \\ 
& 720 & 0.834 & 0.741 & 0.831 & 0.731 & 0.587 & 0.570 & \textbf{0.526} & \textbf{0.542} & \underline{0.528} & 0.546 & 0.536 & \underline{0.544} \\ 
\midrule
 \multirow{5}{1em}{\rot{ILI}} & 24 & 3.954 & 1.323 & 4.588 & 1.462 & \underline{3.101} & 1.238 & \textbf{3.041} & \underline{1.186} & 3.428 & 1.279 & 3.124 & \textbf{1.143}  \\ 
& 36 & 4.167 & 1.360 & 4.845 & 1.496 & \textbf{3.397} & 1.270 & 3.406 & \underline{1.232} & 3.490 & 1.306 & \underline{3.404} & \textbf{1.192} \\ 
& 48 & 4.746 & 1.463 & 4.865 & 1.516 & \textbf{2.947} & \textbf{1.203 }& {3.459} & 1.221 & 3.600 & 1.277 & 3.509 & \underline{1.205} \\ 
& 60 & 5.219 & 1.553 & 5.212 & 1.576 & \textbf{3.019} & \textbf{1.202}& \underline{3.640} & 1.305 & 3.666 & 1.271 & 3.709 & \underline{1.205} 
\\ 
\midrule
\multirow{6}{1em}{\rot{ECL}} & 48 & 0.334 & 0.399 & 0.344 & 0.393 & 0.241 & 0.351 & \underline{0.156} & \underline{0.255} & 0.159 & 0.264 & \textbf{0.154} & \textbf{0.254} \\ 
& 168 & 0.353 & 0.420 & 0.368 & 0.424 & {0.299} & 0.387 & 0.231 & \underline{0.309} & \underline{0.290} & 0.316 &  \textbf{0.225} & \textbf{0.303}  \\ 
 & 336 & 0.381 & 0.439 & 0.381 & 0.431 & 0.375 & 0.428 & \underline{0.323} & \underline{0.369} &\textbf{0.318} & \textbf{0.363} & 0.332 & 0.375 \\ 
& 720 & \underline{0.391} & 0.438 & 0.406 & 0.443 & \textbf{0.377} & 0.434 & 0.404 & \underline{0.423} & {0.397} & \textbf{0.421} & 0.414 & 0.429 \\ 
& 960 & 0.492 & 0.550 & 0.460 & 0.548 & \textbf{0.366 }& \textbf{0.426} & 0.433 & \underline{0.438} & 0.434 & \underline{0.438} & 0.440 & 0.443  \\ 
\midrule
\multirow{6}{1em}{\rot{Traffic}} & 24 & 0.597 & 0.332 & 0.608 & 0.334 & 0.550 & 0.363 & \underline{0.491} &  \underline{0.274} & \textbf{0.488} & \textbf{0.271} & 0.496 & 0.280 \\ 
& 48 & 0.658 & 0.369 & 0.644 & 0.359 & 0.595 &0.376 & 0.519 & 0.295 & \textbf{0.513} & \underline{0.291} &  \underline{0.516} & \textbf{0.290}  \\ 
 & 168 & 0.664 & 0.363 &0.660 & 0.391  & 0.649 & 0.407 & \underline{0.513} & \underline{0.289} & {0.516} & \underline{0.289}  &\textbf{0.512} & \textbf{0.288} \\ 
& 336 & 0.654 & 0.358 & 0.747 & 0.405 & 0.624 & 0.388 & \underline{0.530} & \underline{0.300} & 0.541 & 0.304 &  \textbf{0.529} & \textbf{0.297} \\ 
& 720 & 0.685 & 0.370 & 0.792 & 0.430 & 0.674 & 0.417 & 0.573 & 0.313 & \underline{0.557} & \underline{0.307} & \textbf{0.555} & \textbf{0.304}  \\ 
\bottomrule
\end{tabular}
}
\label{table:mts-result}
\end{table}

\subsubsection{Neural Machine Translation}
\label{sec:nmt}
We showcase the application of the proposed sparse Hopfield model in the context of the classic neural machine translation task, as described in \cite{vaswani2017attention}. 
By substituting the attention mechanism in the transformer with a 1-step $\mathtt{Hopfield}$ and $\mathtt{SparseHopfield}$, we compare the performance (BLEU score) of the transformer and Hopfield models on various language pairs. 
The results of this comparison can be found in \cref{nmt-result}.

\paragraph{Datasets.}
We use the WMT17 \cite{wmt17} machine translation task dataset.
Which consists of sentence pairs of two different languages, where we consider the translation between German and English (EN-DE), Russian and English (RU-EN).
The EN-DE setting has 5.91M pairs of training data, 3000 pairs of validation and 3000 pairs of test data.
The EN-RU setting has 25.78M pairs of training data, 3000 pairs of validation and 3000 pairs of test data.

\paragraph{Baselines.}
For the baselines, we follow the architecture of base transformer in \cite{vaswani2017attention} which has 6 layers of encoder and decoder.
The hidden dimension is 512 and the feed forward dimension is 2048.
More details of configuration can be found in \cref{table:hyperparameters-nmt}.
Note that when switching the attention in base transformer to either $\mathtt{Hopfield}$ or $\mathtt{SparseHopfield}$, no extra parameter was added in our experiment.
Thus, the comparison in our setting is fair.

\paragraph{Setup.}
We follow the setup in \cite{vaswani2017attention} on the WMT-17 dataset.
In this experiment, we consider the task of English to German (EN-DE), German to English (DE-EN), Russian to English (RU-EN) and English to Russian (EN-RU).
We report the BLEU score on the test set.
For WMT17, we follow the base-transformer in \cite{vaswani2017attention}, and train the model with 50000 steps, and report the performance of the last checkpoint\footnote{We follow the implementation in \url{https://github.com/OpenNMT/OpenNMT-py} for the NMT experiment.}. 
Our results show that our proposed $\mathtt{SparseHopfield}$ not only consistently improves upon transformer-based deep learning models but also surpasses the performance of the dense Hopfield model \cite{ramsauer2020hopfield}.
}
\begin{table}[h]\label{nmt-result}
\centering
\caption{\textbf{Results for the machine translation on the WMT17 dataset with language pairs of DE-EN, EN-DE, RU-EN, EN-RU. }
We showcase the application of the proposed sparse Hopfield model in the context of the classic neural machine translation task on WMT17 dataset, as outlined in \cite{vaswani2017attention}. 
By substituting the attention mechanism in the transformer with a 1-step $\mathtt{Hopfield}$ and $\mathtt{SparseHopfield}$, we compare the performance (BLEU score) of the transformer and Hopfield models. 
To ensure a fair comparison, all models (Transformer, Dense Hopfield, Sparse Hopfield) are of the same size.
Our results show that our proposed $\mathtt{SparseHopfield}$ consistently improves upon transformer-based deep learning models
}
\begin{tabular}{ccccc}
\toprule
\textbf{Dataset} & \textbf{DE-EN} & \textbf{EN-DE} & \textbf{RU-EN} & \textbf{EN-RU} \\
\midrule
Transformer & 29.8 & 34.9 & 28.5 & 24.8\\
Dense Hopfield  & 33.6 & \textbf{37.2} & 28.5 & \textbf{24.9} \\
Sparse Hopfield  & \textbf{36.1} & 37.1 & \textbf{28.6} & 24.8 \\
\bottomrule
\end{tabular}
\end{table}
\begin{table}[h]
        \centering
        \caption{Hyperparameter of the NMT experiment.
        }
        \vspace*{0.05truein} 
        \begin{tabular}{l*{1}{c}}
        \toprule
            \bf{parameter} & \multicolumn{1}{c}{\bf{values}}\\ 
            \midrule
            batch size & $4096$\\
            initial lr
            & $2.0$\\
            vocab size (DE-EN)
            & $36000$ \\
            vocab size (EN-RU)
            & $34776$ \\
            num heads
            & $8$\\
            hidden dimension
            & $512$ \\
            word vector dimension
            & $512$ \\
            feed forward dimension
            & $2048$ \\
            encoder layer 
            & $6$ \\
            decoder layer 
            & $6$ \\
            label smoothing
            & $0.1$ \\
            decay method 
            & Noam \\
            optimizer 
            & Adam \\
            warm up steps (DE-EN)
            & 4000 \\
            warm up steps (EN-RU)
            & 8000 \\
            train steps (DE-EN)
            &  50000 \\
            train steps (EN-RU)
            &  80000 \\
            max sequence length
            & 96 \\
            beam size 
            & 5 \\
            tokenizer 
            & sentencepiece \\
            \bottomrule
        \end{tabular}
        \label{table:hyperparameters-nmt}
    \end{table}

\section{Experimental Details}
\label{sec:exp_detials}

All experiments are conducted on the platform with NVIDIA GEFORCE RTX 2080 Ti and INTEL XEON SILVER 4214 @ 2.20GHz.
\subsection{Multiple Instance Learning (MIL)}
\subsubsection{Synthetic Dataset}
\label{appendix:syn}
\paragraph{Architectural Details.}
For the Hopfield model, the architecture is composed of 1 layer of either \texttt{Hopfield} or \texttt{HopfieldPooling} and 1 layer of fully connected output projection.
For the attention model, the architecture is composed of 1 layer of attention layer and 1 layer of fully connected output projection.
The dataset contains 50\% of positive bags and 50\% of negative bags.
\paragraph{Training Details.}
We use an AdamW \cite{loshchilov2017decoupled} optimizer.
For each bag size, we ran the experiment 10 times with different random seed.
For all of our synthetic dataset experiments, we use the exact same configuration, shown in Table 5.
The coefficients of Adam optimizer, betas, are set to $(0.9,0.999)$. 
As the number of training epochs, we use 150, and the evaluate the model on the testset with the last checkpoint.
All the experiments done on the synthetic datasets follow the same architecture and training details.

\paragraph{Baselines.}
For our synthetic dataset, we consider two baselines.
\cite{ilse2018attention} (\textbf{Gated}), where they proposed a gated-attention mechanism by inserting one extra linear layer on the attention weights before the softmax function. And they replaced the activation function of query to Tanh function.
\cite{vaswani2017attention} (\textbf{Attn}), where they proposed a multi-head attention mechanism has been widely used in modern deep learning.

\begin{table}[h]
    \vspace{-0.2truein}
        \centering
        \caption{Statistics of Bit Pattern Synthetic Dataset.
        }
        \vspace*{0.05truein} 
        \begin{tabular}{l*{4}{c}}
        \toprule
              & Unique Patterns & Pattern Length (bits) &  Training  & Test \\ 
            \midrule
             &4 & 4 & 800 & 200 \\
            \bottomrule
        \end{tabular}
        \label{table:statistics-bitpattern}
    \end{table}

\begin{table}[h]
        \centering
        \caption{Hyperparameter of the Bit Pattern Dataset.
        }
        \vspace*{0.05truein} 
        \begin{tabular}{l*{1}{c}}
        \toprule
            \bf{parameter} & \multicolumn{1}{c}{\bf{values}}\\ 
            \midrule
            batch size & $128$\\
            learning rates
            & $10^{-3}$\\
            scaling factors
            & $0.25$ \\
            num heads
            & 8 \\
            head dimension
            & 8 \\
            max update steps
            & 3 \\
            dropout 
            & 0.5 \\
            \bottomrule
        \end{tabular}
        \label{table:hyperparameters-bitpattern}
    \end{table}

\subsubsection{MIL Benchmark Datasets}
\label{appendix:mil_benchmark}
The experiment is conducted on 4 popular MIL datasets. Elephant, Fox and Tiger are datasets for image annotation which are composed of preprocessed and segmented colored images. Each image is characterized by color, texture and shape descriptors. These datasets contain 100 positive images that contain the purposed animals and 100 negative images that are drawn from a pool of images of other animals. Furthermore, we tested our model on the UCSB breast cancer classification task. An instance in UCSB dataset represents a patch of a histopathological image of cancerous or normal tissue. The detailed statistics of datasets are summarized in Table \ref{table:statistics}.

\begin{table}[h]
        \centering
        \caption{Statistics of MIL benchmark datasets
        }
        \vspace*{0.05truein} 
        \begin{tabular}{l*{5}{c}}
        \toprule
            Name & Instances & Features & Bags &$+$bags &$-$bags\\ 
            \midrule
            Elephant
            & 1391 &230 & 200 &100 &100\\
            Fox
            & 1302 &230 & 200 &100 &100\\
            Tiger
            & 1220 &230 & 200 &100 &100\\
            UCSB Breast Cancer
            & 2002 &708 & 58 &26 &32\\
            \bottomrule
        \end{tabular}
        \label{table:statistics}
    \end{table} 

In detail, we used a similar architecture described in \cite{ramsauer2020hopfield} to perform the MIL tasks. Firstly, the instance embeddings are sent to fully connected linear embedding layers with ReLU activation. After that, we used a \texttt{SparseHopfield} which has Sparse Retrieval Dynamics to process the output of fully connected linear layers. Afterward, we flatten the output of \texttt{SparseHopfield} and use a linear network with ReLU activation can perform classification.

To avoid application bias, we follow the experiment setting in \cite{kuccukacsci2018bag, ramsauer2020hopfield} and utilize a stratified ten-fold cross-validation to demonstrate the success of proposed \texttt{SparseHopfield} and \texttt{Hopfield}. For each fold in cross-validation, we use a stratified sampling process to split folds for training into a training set and validation set with a 0.1 split rate. We train the models' parameters and tune hyperparameters via grid searching. Once the hyperparameters are selected and the parameters of models are tuned on the first train folds of the first seed, we apply the selected configuration of the model models on other test folds or folds of other random seeds. All reported ROC-AUC scores are the average results of 5 runs with different random seeds. 

The grid search space is listed in Table \ref{table:hyperparameters}. The embedding layers are the pre-HopfieldPooling linear network and the layer width of them is the number of hidden units. A dropout operation, also known as bag dropout, is applied to the attention matrix since it's easy to overfit on these benchmark datasets. All models are trained with the Adam optimizer for 50 epochs. To combat overfitting, we also use an early-stopper that chooses the best checkpoint on the validation set.
\begin{table}[h]
    \vspace{-0.2truein}
        \centering
        \caption{Hyperparameter grid search space on the respective validation sets of the Elephant, Fox, Tiger and UCSB breast cancer datasets.
        }
        \vspace*{0.05truein} 
        \begin{tabular}{l*{1}{c}}
        \toprule
            \bf{parameter} & \multicolumn{1}{c}{\bf{values}}\\ 
            \midrule
            batch size & \{$4$, $8$, $16$\}\\
            learning rates
            & \{$10^{-3}$, $10^{-5}$\}\\
            learning rate decay
            & \{$0.98$, $0.96$, $0.94$\}\\
            embedding layers
            & \{1, 2\}\\
            layer width
            & \{$32$, $64$, $128$\}\\
            number of heads
            & \{8, 12\}\\
            head dimensions
            & \{$16$, $32$\}\\
            scaling factors
            & \{$0.1$, $1$, $10$\}\\
            bag dropout
            & \{$0.0$, $0.75$\}\\
            \bottomrule
        \end{tabular}
        \label{table:hyperparameters}
    \end{table}

One thing that should be noticed is that \cite{ramsauer2020hopfield} uses the pooling layers \texttt{HopfieldPooling} for MIL tasks instead of associative layers \texttt{SparseHopfield} or \texttt{Hopfield}. We also conduct an ablation experiment in that the model uses the first two modules for MIL tasks following the model structure as well as its training and testing process presented above. As shown in Table \ref{table:mil_benchmark_pooling}, the pooling layers can reach comparative results with associative layers on Fox and Tiger datasets but have performance degradation on Elephant and UCSB datasets. Besides, the \texttt{SparseHopfieldPooling} can also perform better than \texttt{HopfieldPooling} on Tiger and Elephant datasets.

\begin{table}[h]
    \vspace{-0.2truein}
    \centering
    \caption{Results for MIL benchmark datsets in terms of AUC score. The models use pooling layers \texttt{HopfieldPooling} and \texttt{SparseHopfieldPooling} instead.
        }
    \vspace*{0.05truein} 
    \resizebox{ \textwidth}{!}{  
    \begin{tabular}{l*{4}{c}}
    \toprule
    \toprule
            Method & \multicolumn{1}{c}{Tiger}
             & \multicolumn{1}{c}{Fox}
             & \multicolumn{1}{c}{Elephant}
             & \multicolumn{1}{c}{UCSB}\\ 
            \midrule
            w/ \texttt{HopfieldPooling}
            & $0.871\pm0.014$
            & $0.637\pm0.035$
            & $0.876\pm0.015$
            & $0.828\pm0.068$\\
            w/ \texttt{SparseHopfieldPooling}
            & $0.884\pm0.007$
            & $0.610\pm0.033$
            & $0.914\pm0.016$
            & $0.796\pm0.107$\\
            \bottomrule
            \bottomrule
    \end{tabular}
    }
    \label{table:mil_benchmark_pooling}
\end{table}

\end{document}